%% file: main.tex
\pdfoutput=1
\documentclass[11pt]{article}

\usepackage{graphicx}
\usepackage{diagbox}
\usepackage{authblk}
\usepackage[hyperfootnotes=false]{hyperref}
\usepackage{algorithm}
\usepackage{algpseudocode}
\makeatletter
\providecommand{\theHALG@line}{}
\renewcommand{\theHALG@line}{\thealgorithm.\arabic{ALG@line}}
\makeatother

\usepackage{notation-2}
\usepackage{arxiv-2}
\usepackage{amssymb}
\usepackage{mathrsfs}
\usepackage{float}
\usepackage{setspace}
\usepackage{lmodern}

\renewcommand{\tilde}[1]{\widetilde{#1}}

\renewcommand{\hat}[1]{\widehat{#1}}

\newcommand{\pll}{\kern 0.3em/\kern -0.9em /\kern 0.3em}

\newcommand{\Ber}{\text{Bern}}

\makeatletter
\renewcommand\AB@authnote[1]{\textsuperscript{#1}\hspace{5pt}}

\makeatother

\title{\normalfont Iterative Calibration and Learning: Mitigating Teacher Bias in LLM Distillation without Target-Domain Reward Feedback}
\author{
\makebox[\textwidth][c]{%
    \parbox{0.30\textwidth}{\centering
        Haichen Hu\thanks{Email: \texttt{huhc@mit.edu}}\\
        MIT
    }\hfill
    \parbox{0.30\textwidth}{\centering
        Yuheng Zhang\thanks{Email: \texttt{yuhengz2@illinois.edu}}\\
        UIUC
    }\hfill
    \parbox{0.30\textwidth}{\centering
        David Simchi-Levi\thanks{Email: \texttt{dslevi@mit.edu}}\\
        MIT
    }
}}

\hypersetup{
  hidelinks,
  pdftitle={Iterative Calibration and Learning: Mitigating Teacher Bias in LLM Distillation without Target-Domain Reward Feedback},
  pdfauthor={Haichen Hu, Yuheng Zhang, David Simchi-Levi}
}

\begin{document}
\maketitle
\singlespacing
\vspace{-0.5cm}

\input{sections/01_introduction_model}
\input{sections/02_algorithm_theory}
\input{sections/03_separation_discussion}

\bibliographystyle{plainnat}
\bibliography{refs}

\clearpage
\appendix
\allowdisplaybreaks[1]
\input{sections/appendix_related_work}
\input{sections/04_appendix}

\end{document}

%% file: sections/01_introduction_model.tex
\begin{abstract}
Large language model (LLM) distillation can propagate teacher bias under
covariate shift when target reward feedback is unavailable. We propose
Iterative Calibration and Learning (IterCaL), which alternates source-reward
teacher calibration through token-level branching with student learning on
target prompts. We prove polynomial convergence in expected average KL to
the oracle student, the reference-regularized target-reward maximizer within
a potentially misspecified student class; a quantified source-to-target
calibration discrepancy gives an explicit additive true-return term.
Our separation shows that
regularized direct teacher matching retains positive oracle-KL error at its
population optimum for every shared reference penalty, even when the teacher
is stronger than every student. On the same instance, IterCaL converges to
the oracle using source rewards alone.
\end{abstract}

\section{Introduction}
\label{sec:introduction}

Large language models (LLMs) support language understanding, content
generation, mathematical reasoning, and program synthesis
\citep{openai2023gpt4,gemini2023gemini,
grattafiori2024llama3,qwen2024qwen25,deepseek2025r1}.
Their growing role in supplying predictions, decisions, and training
data motivates studying their statistical behavior.

Statistical learning theory examines generalization, inference, and
black-box evaluation; for language models, it also studies pretraining,
in-context learning, and preference-based policy learning
\citep{
bartlett2020benign,angelopoulos2023prediction,
saunshi2021mathematical,bai2023statisticians,
kim2024minimax,ye2024online,wainwright2025wild,hu2025perturbingderivativewildrefitting,
hu2025doublywildrefittingmodelfree,hu2026interleaved}.
These advances raise a central question: when can one model reliably
train another?

Knowledge distillation trains a lower-cost student on a capable
teacher's output probabilities or generated responses, supporting model
compression and transfer of language and reasoning capabilities
\citep{hinton2015distilling,gu2024minillm,deepseek2025r1}.
Such synthetic supervision is valuable when task-specific
demonstrations are scarce.

However, direct imitation can propagate teacher bias: distillation can
amplify errors on difficult image classes despite improving average
accuracy \citep{lukasik2022teacherspet}.
Covariate shift compounds this risk, since strong performance on
supervised source questions need not transfer to target questions.
LLM experiments likewise show accuracy drops when in-context
demonstrations and evaluation examples differ in topic
\citep{roussinov2025controlling}.
The challenge is to retain teacher knowledge while correcting imitation
errors \citep{yamamoto2026residual}.

Reliable target feedback may also be unavailable. Multilingual
evaluations reveal lower reward-model accuracy and inconsistent
preferences across languages \citep{gureja2025mrewardbench}.
Feedback can be costly: mathematical process verifiers
have required extensive human step annotations \citep{lightman2024verify},
while code evaluation needs task-specific tests and execution
infrastructure \citep{chen2021code}.
We therefore use trusted reward feedback only on source questions, with
only questions observed in the target dataset. The learner must correct
target teacher bias without evaluating target answers against true
rewards. Thus, we ask:\par
\vspace{0.3\baselineskip}
{\centering\itshape
Can we distill a near-optimal student from a biased teacher under covariate
shift, using source reward feedback but none on target questions?\par}
\vspace{0.3\baselineskip}
\noindent\textbf{Our contribution.}
We propose Iterative Calibration and Learning (IterCaL), which distills a
biased teacher using only source reward feedback. IterCaL alternates teacher
calibration with student learning on target prompts. The student is the
returned target policy and supplies token alternatives for source comparisons,
using the models available in distillation.

In an autoregressive framework, we prove polynomial convergence to the
regularized oracle student under shared realizability. We further establish an
explicit calibration-transfer bound when the source-calibrated policy only
approximately represents the optimal target policy. The oracle student
maximizes true target reward with KL regularization to a pre-trained reference
within the student class. We also establish a separation theorem: direct
teacher matching has a constant positive KL gap from the optimal student even
when the teacher outperforms every student, whereas IterCaL converges to the
oracle.

\paragraph{Paper structure.}
Section~\ref{sec:itercal-model} formalizes the distillation problem and oracle
benchmark. Section~\ref{sec:algorithm} presents IterCaL and its calibration and
student updates. Section~\ref{sec:itercal-theory} proves convergence to the
oracle student. Section~\ref{sec:sm-temperature} defines regularized direct
teacher matching and establishes the separation. Section~\ref{sec:discussion}
discusses the results; related work appears in
Appendix~\ref{sec:related-work}.

\paragraph{Notation.}
$[n]:=\{1,\ldots,n\}$. Expectation and probability over $X$ are written $\EE_X[\cdot]$ and
$\PP_X(\cdot)$; conditioning on a sigma-algebra $\cF$ is written
$\EE[\cdot\mid\cF]$ and $\PP(\cdot\mid\cF)$.
We use $\sigma(X_1,\ldots,X_k)$ for the generated sigma-algebra and
$\cF\vee\cG$ for the smallest sigma-algebra containing $\cF,\cG$.
We use $\|v\|_2$, $\|A\|_{\mathrm{op}}$, and $\|A\|_{\mathrm F}$ for the
Euclidean, induced operator, and Frobenius norms. For a random vector $X$,
$\|X\|_{L^2}:=(\EE\|X\|_2^2)^{1/2}$; $O_{L^2}$ and $o_{L^2}$ denote rates
in this norm.
For distributions $P,Q$ on a finite set $\mathcal Z$, define
$\KL(P\|Q)=\sum_{z\in\mathcal Z}P(z)\log(P(z)/Q(z))$,
with zero terms when $P(z)=0$ and value $+\infty$ if
$P(z)>0=Q(z)$ for some $z$.
For nonempty closed convex $\mathcal C$,
$\Proj{\mathcal C}{v}:=\arg\min_{u\in\mathcal C}\|u-v\|_2^2$
is Euclidean projection; $\operatorname{Unif}(\mathcal C)$ is the
uniform distribution on $\mathcal C$.
For $b_k>0$, $a_k=O(b_k)$ means $|a_k|/b_k$ is eventually bounded,
and $a_k=o(b_k)$ means this ratio tends to zero as $k\to\infty$.

\section{Model setup}
\label{sec:itercal-model}

In this section, we formulate LLM distillation for post-training under
a fixed source--target design. We model autoregressive generation as a
finite-horizon, token-level Markov decision process
$\mathcal M=(\cS,\cA,P,R,H)$, with prefix-state space $\cS$, finite token
vocabulary $\cA$, deterministic transition kernel $P$, and horizon
$H\ge1$. The reward $R$ is computed from the complete answer.
A prompt $x\in\cX$ contains a question, its instructions, and context.
Starting from $s_1=(x,\varnothing)$, policy $\pi$ samples a legal token
$a_h\sim\pi(\cdot\mid s_h)$ at $s_h=(x,a_{1:h-1})$ and appends it,
giving $s_{h+1}=(x,a_{1:h})$ for $h=1,\ldots,H$.

\begin{definition}[Legal tokens and feasible trajectories]
\label{def:itercal-feasible-trajectories}
Let $\cA$ be a finite token alphabet. For $1\le h\le H$, consider the
prefix state $s_h=(x,a_{1:h-1})$. For fixed-length answers, $\cA$ contains
only answer tokens, and the legal-token set is $\cB(s_h):=\cA$. For answers
allowing early termination, let $\cA$ also contain $\mathtt{EOS}$ and
$\mathtt{null}$, and define
\[
\cB(s_h):=
\begin{cases}
\cA\setminus\{\mathtt{null}\},&
\mathtt{EOS}\notin\{a_1,\ldots,a_{h-1}\},\\
\{\mathtt{null}\},&
\mathtt{EOS}\in\{a_1,\ldots,a_{h-1}\}.
\end{cases}
\]
Under either convention, the state space $\cS$ consists of all prefixes reachable from
$(x,\varnothing)$, $x\in\cX$, by appending legal tokens, up to length
$H$. At terminal states $s_{H+1}$, set $\cB(s_{H+1}):=\varnothing$.
For a reachable state $s_h$ with $h\le H$, define the feasible suffix set
\[
\begin{aligned}
&\cA(s_h):=\bigl\{b_{h:H}\in\cA^{H-h+1}:\\
&\quad b_k\in\cB((x,a_{1:h-1},b_{h:k-1}))
\text{ for }k=h,\ldots,H\bigr\}.
\end{aligned}
\]
Empty blocks in the concatenated prefix are omitted. Set
\[
\cA(s_{H+1}):=\{\varnothing\},
\ 
\cA(x):=\cA((x,\varnothing)),
\]
where $\cA(x)$ contains full feasible answers.
\end{definition}

For fixed-length generation, the model produces exactly $H$ answer tokens,
after which $\mathtt{EOS}$ is appended automatically. This deterministic
terminal token is excluded from the modeled answer $a_{1:H}$. When early
termination is allowed, a generated $\mathtt{EOS}$ is followed by
$\mathtt{null}$ padding to length $H$. If the horizon is reached without an
earlier $\mathtt{EOS}$, it is likewise appended automatically outside
$a_{1:H}$. At each nonterminal state, $\pi(\cdot\mid s_h)$ is a
probability distribution supported on $\cB(s_h)$, inducing the
full-answer law
\begin{equation}
\pi(a_{1:H}\mid x)
=\prod_{h=1}^{H}\pi(a_h\mid x,a_{1:h-1}),
\  a_{1:H}\in\cA(x).
\label{eq:itercal-chain}
\end{equation}
We use $\pi$ for both token and full-answer laws.

Post-training seeks to improve answer quality while retaining
pretrained behavior. We study covariate shift between fixed source
and target prompt datasets:
\[
\cD_{\mathrm{src}}=\{x_i\}_{i=1}^{n},
\ 
\cD_{\mathrm{tar}}=\{\widetilde x_j\}_{j=1}^{m},
\  n,m\ge1.
\]
Both contain only prompts and may differ in their mixtures of
question types. We condition on both datasets throughout;
randomness comes from policy sampling and the training algorithm.

The reward for a complete answer is given by the deterministic function
\[
R:\{(x,a_{1:H}):x\in\cX,\ a_{1:H}\in\cA(x)\}\longrightarrow[0,1].
\]
We can query an exact verifier \emph{only} on \emph{source prompts}:
given $x_i$ and any $a_{1:H}\in\cA(x_i)$, it returns
$R(x_i,a_{1:H})$. The objective concerns target answers, whose
true rewards $R(\widetilde x_j,a_{1:H})$ are never queried during
training. This restriction reflects the cost of reliable target
feedback, which can require expert-designed rubrics, labeled
answers, and validation.

Let $\pi_{\mathrm{pre}}$ be the frozen pretrained reference
policy, with positive probability on every legal token.
For $\lambda>0$, our target post-training objective is
\begin{equation}
\begin{aligned}
&J_{\lambda,m}(\pi):=\frac1m\sum_{j=1}^{m}
\EE_{a_{1:H}\sim\pi(\cdot\mid\widetilde x_j)}
\Bigg[
R(\widetilde x_j,a_{1:H})\\
&\quad-\lambda\sum_{h=1}^{H}
\log\frac{\pi(a_h\mid\widetilde x_j,a_{1:h-1})}
{\pi_{\mathrm{pre}}(a_h\mid\widetilde x_j,a_{1:h-1})}
\Bigg].
\end{aligned}
\label{eq:itercal-return}
\end{equation}
The first term measures answer quality; the second penalizes
deviation from the pretrained policy. By \eqref{eq:itercal-chain},
the expected log-ratio sum equals the full-answer KL:
\[
\begin{aligned}
&\KL\!\left(\pi(\cdot\mid x)\,\middle\|\,
\pi_{\mathrm{pre}}(\cdot\mid x)\right)
\\
&\quad=
\EE_{a_{1:H}\sim\pi(\cdot\mid x)}
\!\left[
\sum_{h=1}^{H}
\log\frac{\pi(a_h\mid x,a_{1:h-1})}
{\pi_{\mathrm{pre}}(a_h\mid x,a_{1:h-1})}
\right].
\end{aligned}
\]
Thus $\lambda$ controls the tradeoff between reward and proximity
to the reference. The objective specifies the desired target
behavior, although its reward term is unavailable during training.
We address this information constraint through teacher
distillation: source reward feedback calibrates an available
teacher model, whose likelihoods then supervise the student
on the target prompts.

To formalize the teacher, its calibration, and the student, we
use linear-softmax policy classes. Fix known feature maps
$\phi:\cS\times\cA\to\RR^D$ and
$\phi_{\mathrm{stu}}:\cS\times\cA\to\RR^d$, where $D,d\ge1$,
with
\[
\|\phi(s,a)\|_2\le1,
\ 
\|\phi_{\mathrm{stu}}(s,a)\|_2\le1.
\]
The features may depend on the entire prefix, preserving the
autoregressive dependence of the policy.
For every nonterminal state $s$ and legal token $a\in\cB(s)$,
define
\begin{equation}
\begin{aligned}
\pi_w(a\mid s)
&=\frac{\exp(w^\top\phi(s,a))}
{\sum_{b\in\cB(s)}\exp(w^\top\phi(s,b))},\\
\pi_{\mathrm{stu},\theta}(a\mid s)
&=\frac{\exp(\theta^\top\phi_{\mathrm{stu}}(s,a))}
{\sum_{b\in\cB(s)}\exp(\theta^\top\phi_{\mathrm{stu}}(s,b))}.
\end{aligned}
\label{eq:itercal-softmax}
\end{equation}
Both policies assign zero probability to illegal tokens.
For $B>0$, the calibration parameter belongs to a nonempty
compact convex set $W$, and the student parameter belongs to
the closed ball $\Theta$:
\[
\begin{aligned}
W&\subseteq\{w\in\RR^D:\|w\|_2\le B\},\\
\Theta&:=\{\theta\in\RR^d:\|\theta\|_2\le B\}.
\end{aligned}
\]
The given teacher is $\pi_{\mathrm{tea}}=\pi_{w_{\mathrm{tea}}}$
with known $w_{\mathrm{tea}}\in W$. Its proposal policy remains frozen,
while calibration adjusts a separate parameter within $W$,
initialized at $w_{\mathrm{tea}}$.
The teacher may be biased relative to the optimal target behavior.
The reference $\pi_{\mathrm{pre}}$ need not belong to either
softmax class, and we impose neither sparsity nor an upper bound
relating $D$ to $H$.

\paragraph{Practical relevance.}
At the branching stage, both candidate tokens come from models already
present in distillation: the frozen teacher and the current student.
Thus, IterCaL requires neither vocabulary enumeration, a hand-designed
uniform-token sampler, an auxiliary proposal model, nor modification of
the frozen teacher's sampling distribution. This model-native mechanism
aligns with empirical LLM distillation workflows:
MiniLLM uses teacher--student mixture sampling \citep{gu2024minillm},
while GKD combines student-generated sequences with a fixed dataset
that may itself be teacher-generated \citep{agarwal2024onpolicy}.

For the performance metric, our benchmark is the best post-trained policy within the student
class. Specifically, we choose
\[
\theta^\dagger_{\lambda,m}
\in\argmax_{\theta\in\Theta}
J_{\lambda,m}(\pi_{\mathrm{stu},\theta}).
\]
$\theta^\dagger_{\lambda,m}$ exists because $\Theta$ is compact, and the objective
is continuous in $\theta$. We call
$\pi_{\mathrm{stu},\theta^\dagger_{\lambda,m}}$ an oracle student:
it optimizes the true target objective using the same class
available to the learned student.
For comparison, at each source or target prompt $x$, we define
the unrestricted optimal policy by
\begin{equation}
\begin{aligned}
&\pi^\star_\lambda(\cdot\mid x)
\in\argmax_{\pi(\cdot\mid x)\in\Delta(\cA(x))}
\EE_{a_{1:H}\sim\pi(\cdot\mid x)}\Bigg[\\
&\quad R(x,a_{1:H})-\lambda\sum_{h=1}^{H}
\log\frac{\pi(a_h\mid x,a_{1:h-1})}
{\pi_{\mathrm{pre}}(a_h\mid x,a_{1:h-1})}
\Bigg].
\end{aligned}
\label{eq:itercal-true-definition}
\end{equation}
Here $\Delta(\cA(x))$ is the simplex of distributions over full
feasible answers. Token conditionals are obtained from prefix
marginals. Lemma~\ref{lem:itercal-gibbs} establishes that this optimum
is unique and assigns positive probability to every feasible
answer. The student class may be unable to represent this
unrestricted optimum.

We connect source reward information to target behavior through the following realizability assumption.

\begin{assumption}[Realizability]
\label{ass:itercal-truth}
There exists a single parameter $w^\star_\lambda\in W$ such that,
at every legal nonterminal prefix state $s$ of every source or
target prompt,
\[
\pi^\star_\lambda(a\mid s)
=\pi_{w^\star_\lambda}(a\mid s)
\ \text{for every }a\in\cB(s).
\]
\end{assumption}

The shared coefficients transfer source calibration to target
prompts through their prefix features. Only the calibration class
must be realizable; the student may remain misspecified.
This shared structure supports our fixed-design analysis without
a density-ratio assumption.

For $w\in W$, define the target distillation cost
\begin{equation}
C_w(\theta):=\frac{\lambda}{m}\sum_{j=1}^{m}
\KL\!\left(
\pi_{\mathrm{stu},\theta}(\cdot\mid\widetilde x_j)
\,\middle\|\,
\pi_w(\cdot\mid\widetilde x_j)
\right).
\label{eq:itercal-cost}
\end{equation}
This cost is estimable from target student rollouts and the
student and calibrated-policy likelihoods. Under realizability,
Lemma~\ref{lem:itercal-gibbs} gives
\[
\argmax_{\theta\in\Theta}
J_{\lambda,m}(\pi_{\mathrm{stu},\theta})
=
\argmin_{\theta\in\Theta}C_{w^\star_\lambda}(\theta).
\]
Thus, correctly calibrated distillation targets the regularized
oracle student.

\begin{assumption}
\label{ass:itercal-unique}
The function $C_{w^\star_\lambda}$ has a unique minimizer over
$\Theta$,
$\argmin_{\theta\in\Theta}C_{w^\star_\lambda}(\theta)
=\{\theta^\dagger_{\lambda,m}\}$.

\end{assumption}

Neither interiority nor a positive student Hessian is required.
For iterates $(w_t,\theta_t)$, write
$C_t(\theta):=C_{w_t}(\theta)$ and define
\begin{equation}
\begin{aligned}
\varepsilon_t
&:=C_t(\theta_t)-\min_{\theta\in\Theta}C_t(\theta),\\
\Delta_{\lambda,m}(\theta)
&:=C_{w^\star_\lambda}(\theta)
-C_{w^\star_\lambda}(\theta^\dagger_{\lambda,m}),\\
\cK_{\lambda,m}(\theta)
&:=\frac1m\sum_{j=1}^{m}\\
&\quad\times
\KL\!\left(
\pi_{\mathrm{stu},\theta}(\cdot\mid\widetilde x_j)
\,\middle\|\,
\pi_{\mathrm{stu},\theta^\dagger_{\lambda,m}}
(\cdot\mid\widetilde x_j)
\right).
\end{aligned}
\label{eq:itercal-errors}
\end{equation}
These measure the current calibrated objective gap, excess cost
under true calibration, and average KL from student to oracle,
respectively. We aim for $\EE[\cK_{\lambda,m}(\theta_T)]\to0$,
while the minimum distillation cost within the student class
may remain positive.

Finally, source comparisons must contain enough information
to identify the shared calibration parameter.
For a teacher-generated prefix
$s_{i,h}=(x_i,a_{i,1:h-1})$, compare its next token $a_{i,h}$
with a legal alternative $b$.
The difference
$\phi(s_{i,h},a_{i,h})-\phi(s_{i,h},b)$ determines the
calibration direction observed in that comparison. For each $h=1,\ldots,H$, define the source information matrix
$G_h$ by
\begin{equation}
\begin{aligned}
&G_h
:=\frac1n\sum_{i=1}^{n}
\EE_{a_{i,1:h}\sim\pi_{\mathrm{tea}}(\cdot\mid x_i)}
\Bigg[
\frac1{|\cB(s_{i,h})|}
\sum_{b\in\cB(s_{i,h})}\\
&\qquad\bigl(\phi(s_{i,h},a_{i,h})-\phi(s_{i,h},b)\bigr)\\
&\qquad\times\bigl(\phi(s_{i,h},a_{i,h})-\phi(s_{i,h},b)\bigr)^\top
\Bigg],
\end{aligned}
\label{eq:itercal-gram}
\end{equation}
The expectation is over the teacher's output
$a_{i,1:h}$, and the inner average is uniform over legal tokens. We impose the following identification
condition on their average across token positions.
\begin{assumption}
\label{ass:itercal-identification} Define the joint source information matrix as
$G_{\mathrm{joint}}:=\frac1H\sum_{h=1}^{H}G_h.$ We assume that it is positive definite, i.e., 
\[
\mu_{\mathrm{joint}}
:=\lambda_{\min}(G_{\mathrm{joint}})>0.
\]
\end{assumption}

This assumption requires comparisons across all token positions
to identify every calibration direction; individual matrices
$G_h$ may be singular.
At a padded state, the only legal token is $\mathtt{null}$,
so its feature difference and information contribution are zero.

%% file: sections/02_algorithm_theory.tex
\section{Algorithm}
\label{sec:algorithm}

In this section, we propose Iterative Calibration and Learning (IterCaL),
summarized in Algorithm~\ref{alg:itercal}. Each round calibrates the teacher
using source rewards and then updates the student on target prompts.

\begin{algorithm*}[!t]
\caption{Iterative Calibration and Learning (IterCaL) via Token-Level Branching}
\label{alg:itercal}
\small
\begin{algorithmic}[1]
\Require Source and target datasets $\{x_i\}_{i=1}^{n}$,
$\{\widetilde x_j\}_{j=1}^{m}$; source reward verifier $R$;
teacher $\pi_{\mathrm{tea}}=\pi_{w_{\mathrm{tea}}}$ and
$\pi_{\mathrm{pre}}$; features $\phi$, $\phi_{\text{stu}}$ and sets $W,\Theta$;
$w_{\mathrm{tea}}\in W$, $\theta_0\in\Theta$;
$\lambda,\gamma>0$, integer $T\ge1$;
schedules \eqref{eq:br-budgets}.
\Statex All draws use fresh randomness conditional on the preceding variables.

\State $w_0\gets w_{\mathrm{tea}}$,
$\pi_{\mathrm{cal},0}\gets\pi_{w_0}$.
\label{line:br-init}

\For{$t=0,\ldots,T-1$}
    \Statex \textit{Token-Level branching on source}
    \State Draw independently
    $i_t\sim\operatorname{Unif}\{1,\ldots,n\}$ and
    $h_t\sim\operatorname{Unif}\{1,\ldots,H\}$.
    \label{line:br-design}

    \State Draw $a_{t,1:H}\sim
    \pi_{\mathrm{tea}}(\cdot\mid x_{i_t})$ autoregressively.
    \label{line:br-teacher}

    \State $s_t\gets(x_{i_t},a_{t,1:h_t-1})$,
    $c_{t,1}\gets a_{t,h_t}$.
    \Comment{Teacher prefix and token.}
    \label{line:br-prefix}
    \label{line:br-candidates}

    \State Draw $c_{t,0}\sim
    \pi_{\mathrm{stu},\theta_t}(\cdot\mid s_t)$.
    \Comment{Student alternative.}

    \State $z_t\gets
    \phi(s_t,c_{t,1})-\phi(s_t,c_{t,0})$.
    \label{line:br-feature}

    \Statex \textit{Teacher calibration.}

    \State $\displaystyle
    Y_t\sim\operatorname{Bernoulli}\!\left(
    \frac{\pi_{\mathrm{pre}}(c_{t,1}\mid s_t)}
    {\pi_{\mathrm{pre}}(c_{t,1}\mid s_t)
    +\pi_{\mathrm{pre}}(c_{t,0}\mid s_t)}
    \right)$.
    \label{line:br-choice}

    \State
    $a^{\mathrm{pre}}_{t,1:h_t-1}\gets a_{t,1:h_t-1}$,
    $a^{\mathrm{pre}}_{t,h_t}\gets c_{t,Y_t}$.
    \Comment{Construct the selected branch.}
    \label{line:br-force}

    \State Draw $a^{\mathrm{pre}}_{t,h_t+1:H}\sim
    \pi_{\mathrm{pre}}
    (\cdot\mid x_{i_t},a^{\mathrm{pre}}_{t,1:h_t})$
    autoregressively.
    \label{line:br-reference}

    \State $R_t\gets R(x_{i_t},a^{\mathrm{pre}}_{t,1:H})$.
    \Comment{Exactly one source reward query.}
    \label{line:br-reward}

    \State Draw $U_t\sim\operatorname{Unif}[0,1]$.

    \State $I_t\gets
    \mathbf{1}\{U_t\le\exp((R_t-1)/\lambda)\}$.
    \label{line:br-accept}

    \State $g_t\gets
    I_tz_t[\sigma(z_t^\top w_t)-Y_t]$.
    \Comment{One trial; rejection gives $g_t=0$.}
    \label{line:br-gradient}

    \State $w_{t+1}\gets\Proj{W}{w_t-\eta_tg_t}$.
    \label{line:br-project}

    \State $\pi_{\mathrm{cal},t+1}\gets\pi_{w_{t+1}}$.

    \Statex \textit{Student gradient update.}
    \State Define $Z_{t+1}(\theta,x,a_{1:H}):=\lambda\log\frac{\pi_{\text{stu},\theta}(a_{1:H}|x)}{\pi_{\text{cal},t+1}(a_{1:H}|x)}$, $S_{\mathrm{stu},\theta}(x,a_{1:H}):=\nabla_\theta
\log\pi_{\mathrm{stu},\theta}(a_{1:H}\mid x)$, $\forall \theta\in\Theta$.
    \State Draw independent prompt-answer pairs
    for $\ell=1,\ldots,b_{t+1}$:
    \label{line:br-kl}
   $$\displaystyle
    j_{t+1,\ell}^{\mathrm g}
    \sim\operatorname{Unif}\{1,\ldots,m\},
    \  
    a^{\mathrm g}_{t+1,\ell,1:H}
    \mid j_{t+1,\ell}^{\mathrm g}
    \sim\pi_{\mathrm{stu},\theta_t}
    (\cdot\mid\widetilde x_{j_{t+1,\ell}^{\mathrm g}}).
    $$
    \vspace{-1.5em}
    \State $\displaystyle
    \widehat g_{t+1}^{\mathrm{stu}}\gets
    \frac1{b_{t+1}}\sum_{\ell=1}^{b_{t+1}}
    S_{\mathrm{stu},\theta_t}
    (\widetilde x_{j_{t+1,\ell}^{\mathrm g}},
    a^{\mathrm g}_{t+1,\ell,1:H})
    Z_{t+1}
    (\theta_t;\widetilde x_{j_{t+1,\ell}^{\mathrm g}},
    a^{\mathrm g}_{t+1,\ell,1:H})$.
    \label{line:br-student-gradient}

    \State $\vartheta_{t+1,1}\gets
    \Proj{\Theta}{
    \theta_t-\alpha_{t+1}\widehat g_{t+1}^{\mathrm{stu}}}$;
    independently draw
    $\vartheta_{t+1,2}\sim\operatorname{Unif}(\Theta)$.
    \label{line:br-exploration}

    \Statex \textit{Candidate student evaluation and selection.}

    \For{$k=1,2$}

        \State Draw fresh independent prompt-answer pairs
        for $\ell=1,\ldots,q_{t+1}$:
         $$
        j_{t+1,k,\ell}
        \sim\operatorname{Unif}\{1,\ldots,m\},
        \  
        a^{\mathrm{val}}_{t+1,k,\ell,1:H}
        \mid j_{t+1,k,\ell}
        \sim\pi_{\mathrm{stu},\vartheta_{t+1,k}}
        (\cdot\mid\widetilde x_{j_{t+1,k,\ell}}).
        $$
        \vspace{-1.5em}
        \State $\displaystyle
        \widehat C_{t+1,k}\gets
        \frac1{q_{t+1}}\sum_{\ell=1}^{q_{t+1}}
        Z_{t+1}
        (\vartheta_{t+1,k};
        \widetilde x_{j_{t+1,k,\ell}},
        a^{\mathrm{val}}_{t+1,k,\ell,1:H})$.
        \label{line:br-evaluation}

    \EndFor

    \State $\widehat k_{t+1}\gets
    \min\argmin_{k\in\{1,2\}}\widehat C_{t+1,k}$,
    $\theta_{t+1}\gets
    \vartheta_{t+1,\widehat k_{t+1}}$.
    \label{line:br-student}

\EndFor
\State \Return $\pi_{\mathrm{stu},\theta_T}$.
\label{line:br-return}
\end{algorithmic}
\end{algorithm*}

Algorithm~\ref{alg:itercal} maintains and updates $w_t$ for the
calibrated teacher $\pi_{\mathrm{cal},t}=\pi_{w_t}$ and $\theta_t$ for
the student $\pi_{\mathrm{stu},\theta_t}$.
The student is the target policy optimized and returned. Its current
iterate also supplies the alternative token in the next source comparison,
so every branching step uses only the frozen teacher and the student already
being trained, without an auxiliary proposal model. The calibration analysis
in Lemma~\ref{lem:itercal-calibration} allows this proposal distribution to adapt
across rounds.
The original teacher $\pi_{\mathrm{tea}}$ remains frozen and supplies the
source prefix and teacher token; $\pi_{\mathrm{pre}}$ determines branch
sampling and completion.

To calibrate the teacher, Lines~\ref{line:br-design}-\ref{line:br-feature}
pair its next token $c_{t,1}$ at a randomly selected position in a
source answer with a student alternative $c_{t,0}$ at the same prefix
$s_t$. Both tokens are sampled from models already used by IterCaL, and their
feature difference $z_t$ determines the log-odds:
\[
\log\frac{\pi_w(c_{t,1}\mid s_t)}{\pi_w(c_{t,0}\mid s_t)}
=z_t^\top w.
\]
Thus, a comparison constrains calibration along $z_t$; sampling across
all $H$ positions covers the joint source directions in
\eqref{eq:itercal-gram}. Lemma~\ref{lem:itercal-calibration} proves the
corresponding adaptive-design lower bound.

Lines~\ref{line:br-choice}-\ref{line:br-accept} choose a branch in
proportion to the reference-policy probabilities, complete it with
$\pi_{\mathrm{pre}}$, and query one source reward $R_t$.
The branch is accepted with probability $\exp((R_t-1)/\lambda)$.
Let $\mathcal F_t$ denote the history before round $t$.
Under Assumption~\ref{ass:itercal-truth}, Lemma~\ref{lem:itercal-label} gives
\[
\begin{aligned}
&\Pr\!\left(
Y_t=1\,\middle|\,
I_t=1,\mathcal F_t,s_t,c_{t,1},c_{t,0}
\right)\\
&\quad=
\frac{\pi^\star_\lambda(c_{t,1}\mid s_t)}
{\pi^\star_\lambda(c_{t,1}\mid s_t)
+\pi^\star_\lambda(c_{t,0}\mid s_t)}
=\sigma(z_t^\top w^\star_\lambda),
\end{aligned}
\]
where $\sigma(u)=(1+e^{-u})^{-1}$.
Accepted comparisons therefore provide logistic supervision for
$w^\star_\lambda$. Lines~\ref{line:br-gradient}-\ref{line:br-project}
apply the acceptance-weighted logistic gradient and project the update
onto $W$; rejection gives a zero update.

The resulting teacher defines the target objective
$C_{t+1}=C_{w_{t+1}}$ from \eqref{eq:itercal-cost}.
Source calibration transfers to this objective because the same
$w^\star_\lambda$ represents the regularized optimal policy on both
datasets (Assumption~\ref{ass:itercal-truth}). Source comparisons identify
this parameter under Assumption~\ref{ass:itercal-identification}, and the
calibrated-policy likelihood interface evaluates it on target answers.
At the true calibration, minimizing $C_{w^\star_\lambda}$ recovers the
regularized oracle student (Lemma~\ref{lem:itercal-gibbs}).

To update the student, Lines~\ref{line:br-kl}-\ref{line:br-student-gradient}
average products of the answer score $S_{\mathrm{stu},\theta_t}$ and
log-ratio cost $Z_{t+1}$ over $b_{t+1}$ conditionally independent current-student
rollouts on uniformly sampled target prompts. Holding the calibrated
teacher fixed throughout the batch makes this an unbiased gradient
estimate for $C_{t+1}$, with conditional variance reduced by a factor of
$b_{t+1}$ (Lemma~\ref{lem:itercal-unbiased}).
Line~\ref{line:br-exploration} forms one projected-gradient candidate
and an independent uniform candidate from $\Theta$.
This combination is motivated by the population analysis: gradient
steps provide progress near the oracle, while uniform exploration
provides progress outside that region (Lemma~\ref{lem:itercal-hybrid}).

Lines~\ref{line:br-evaluation}--\ref{line:br-student} use conditionally
independent $q_{t+1}$-rollout batches; their $Z_{t+1}$ averages are unbiased
for each candidate's student-to-calibrated-teacher KL cost $C_{t+1}$, and
the smaller estimate is selected. This requires policy likelihoods,
student scores for the gradient, and target rollouts, but no target rewards.
IterCaL interleaves source calibration and target learning: each
calibrated policy supervises the next student update, and the updated
student supplies alternatives for subsequent source comparisons.
This feedback changes the calibration information captured by
$G_{\mathrm{src}}(\theta_t)$ in
\eqref{eq:itercal-student-information}.
In Appendix~\ref{app:itercal-ctl-comparison}, we prove that, under
covariance coverage improvement and the stated regularity conditions,
IterCaL achieves strictly smaller asymptotic expected
oracle-policy KL than calibrate-then-learn with the same source-query
and target-rollout budgets. We also provide a practically motivated
autoregressive example demonstrating this advantage.

For the calibration step size, define
\begin{equation}
\gamma
:=e^{-1/\lambda}e^{-2B}\sigma'(2B)\mu_{\mathrm{joint}}>0.
\label{eq:itercal-gamma}
\end{equation}
This constant quantifies the source-calibration information
used in the convergence analysis. The algorithm uses $\gamma$ to set its calibration
step size. Define
\[
L_{\mathrm{st}}:=\lambda H(1+8BH),\ 
\alpha_{\mathrm{stu}}:=\frac1{2L_{\mathrm{st}}}.
\]
Lemma~\ref{lem:itercal-smoothness} proves that $L_{\mathrm{st}}$
is a uniform smoothness bound for the student objectives.
For $t=0,\ldots,T-1$, we use
\begin{equation}
\begin{aligned}
\eta_t&=\frac1{\gamma(t+2)},&
\alpha_{t+1}&=\alpha_{\mathrm{stu}},\\
b_{t+1}&=t+2,&
q_{t+1}&=t+2,
\end{aligned}
\label{eq:br-budgets}
\end{equation}
These schedules set the calibration and student step sizes, gradient
batch size, and validation budget per candidate. Both sampling errors
vanish while the student step size remains fixed.
Each target update uses $b_{t+1}+2q_{t+1}=3(t+2)$ rollouts. Over $T$ rounds,
IterCaL makes exactly $T$ source reward queries and uses
$N_T=3T(T+3)/2$ target rollouts.

\section{Theoretical guarantee}
\label{sec:itercal-theory}

We establish finite-iteration convergence of the regularized return gap
and oracle-policy KL, accounting for calibration, stochastic gradients,
and candidate validation. The datasets are fixed, and expectations cover
all adaptive randomness.

\begin{theorem}
\label{thm:itercal-main}
Under the model and assumptions in Section~\ref{sec:itercal-model},
there exist fixed-model constants $K_{\mathrm{opt}},C_{\lambda,m}>0$ and
$p_{\lambda,m}\ge1$, independent of $T$, such that the output
of Algorithm~\ref{alg:itercal} satisfies, for every integer $T\ge1$,
\begin{align}
\EE[\Delta_{\lambda,m}(\theta_T)]
&\le K_{\mathrm{opt}}(T+1)^{-1/4},
\label{eq:itercal-main-gap}\\
\EE[\cK_{\lambda,m}(\theta_T)]
&\le C_{\lambda,m}(T+1)^{-1/(4p_{\lambda,m})}.
\label{eq:itercal-main-kl}
\end{align}
Here the errors are defined in \eqref{eq:itercal-errors}, and
$\Delta_{\lambda,m}$ is the oracle's true regularized target return
minus the student's return.
\end{theorem}

Theorem~\ref{thm:itercal-main} shows that target rewards are unnecessary for
oracle-student recovery. Without a target verifier, source feedback identifies
calibration through shared source--target structure; target training uses only
model-native samples and likelihoods. The student converges in true regularized
target return and target-policy KL to the in-class oracle,
rather than imitating the frozen teacher.

\paragraph{Approximate calibration transfer.}
Theorem~\ref{thm:itercal-main} uses one calibration parameter across domains,
as is natural with domain-invariant correctness rules such as mathematics or
executable code. For open-ended text, \citet{stiennon2020summarize} train their
policy and reward model on Reddit TL;DR and transfer to CNN/DailyMail without
news-specific fine-tuning. We therefore let the calibrated-teacher class
contain the regularized optimal policy on source prefixes while allowing target
misspecification. Following the best-in-class formulation in
\citet{xie2026mathematical}, we define each domain's best calibration.

Let $\cS_{\mathrm{src}}$ and $\cS_{\mathrm{tar}}$ be the legal nonterminal
prefixes of the fixed source and target prompts. For
$\mathsf D\in\{\mathrm{src},\mathrm{tar}\}$, define
\begin{equation}
\begin{aligned}
d_{\mathsf D}(w)
&:=\max_{\substack{s\in\cS_{\mathsf D}\\a\in\cB(s)}}
\left|\log\frac{\pi^\star_\lambda(a\mid s)}{\pi_w(a\mid s)}\right|,\\
w^\star_{\lambda,\mathsf D}
&\in\argmin_{w\in W}d_{\mathsf D}(w).
\end{aligned}
\label{eq:itercal-domain-best}
\end{equation}
Here $d_{\mathsf D}(w)$ is calibration $w$'s worst token-level
log-probability error on domain $\mathsf D$. Exact source realizability means
$d_{\mathrm{src}}(w^\star_{\lambda,\mathrm{src}})=0$, while
$d_{\mathrm{tar}}(w^\star_{\lambda,\mathrm{tar}})$ may be positive and measures
the calibrated-teacher class's smallest possible target error. IterCaL learns
the source-best calibration, whose target error is
\begin{equation}
\begin{aligned}
\delta_{\mathrm{tar}}
&:=d_{\mathrm{tar}}(w^\star_{\lambda,\mathrm{src}})\\
&\le d_{\mathrm{tar}}(w^\star_{\lambda,\mathrm{tar}})
+\rho_{\mathrm{tar}},\\
\rho_{\mathrm{tar}}
&:=\max_{\substack{s\in\cS_{\mathrm{tar}}\\a\in\cB(s)}}
\left|\log\frac{\pi_{w^\star_{\lambda,\mathrm{tar}}}(a\mid s)}
{\pi_{w^\star_{\lambda,\mathrm{src}}}(a\mid s)}\right|.
\end{aligned}
\label{eq:itercal-target-discrepancy}
\end{equation}
Thus $\delta_{\mathrm{tar}}$ measures IterCaL's target mismatch. Its bound
separates the teacher class's best target approximation error from the
source-to-target calibration shift $\rho_{\mathrm{tar}}$; the corresponding
answer-level log-probability mismatch is at most $H\delta_{\mathrm{tar}}$.

\begin{corollary}
\label{cor:itercal-approx-target}
Under this approximate target-realizability condition, assume source
identification (Assumption~\ref{ass:itercal-identification}), exact source
realizability $d_{\mathrm{src}}(w^\star_{\lambda,\mathrm{src}})=0$, and a
unique minimizer of $C_{w^\star_{\lambda,\mathrm{src}}}$ over $\Theta$. With the
schedules in \eqref{eq:br-budgets}, IterCaL guarantees, for some model-dependent
$K^{\mathrm{src}}_{\mathrm{opt}}>0$ and every $T\ge1$,
\begin{align}
&\EE\!\left[J_{\lambda,m}
(\pi_{\mathrm{stu},\theta^\dagger_{\lambda,m}})
-J_{\lambda,m}(\pi_{\mathrm{stu},\theta_T})\right]\notag\\[-0.25em]
&\qquad\le K^{\mathrm{src}}_{\mathrm{opt}}(T+1)^{-1/4}
+2\lambda H\delta_{\mathrm{tar}}.
\label{eq:itercal-approx-return}
\end{align}
If, in addition, every
$\theta\in\argmax_{\vartheta\in\Theta}J_{\lambda,m}
(\pi_{\mathrm{stu},\vartheta})$ induces the same target answer laws as
$\theta^\dagger_{\lambda,m}$, IterCaL also guarantees, for model-dependent
constants $a^\star_{\lambda,m}>0$ and $p^\star_{\lambda,m}\ge1$,
\begin{equation}
\EE[\cK_{\lambda,m}(\theta_T)]
\le\left(
\frac{K^{\mathrm{src}}_{\mathrm{opt}}(T+1)^{-1/4}
+2\lambda H\delta_{\mathrm{tar}}}{a^\star_{\lambda,m}}
\right)^{1/p^\star_{\lambda,m}}.
\label{eq:itercal-approx-kl}
\end{equation}
\end{corollary}

Thus source-only calibration approaches the true target oracle up to the
explicit transfer term $2\lambda H\delta_{\mathrm{tar}}$, without target
reward queries.

%% file: sections/03_separation_discussion.tex
\section{A Separation from Regularized Direct Teacher Matching}
\label{sec:sm-temperature}

We construct a binary-verdict LLM-as-judge instance in which a teacher with
higher true regularized and unregularized reward than every student drives
regularized direct matching away from the oracle, while IterCaL recovers it.
The teacher matches the regularized optimal policy on six of eight prompts and is only weakly
biased on two, so the separation does not rely on an uninformative teacher
\citep{zhu2025judgelm,kim2024prometheus}.

First, define regularized direct teacher matching, or soft matching (SM).
Given target prompts, a frozen teacher $\pi_{\mathrm{tea}}$, a pretrained
reference $\pi_{\mathrm{pre}}$, and student class
$\{\pi_{\mathrm{stu},\theta}:\theta\in\Theta\}$, it minimizes the following
cost with an independently tuned reference coefficient $\beta\ge0$:
\begin{equation}
\begin{aligned}
C_{\mathrm{SM},\beta}(\theta):=\frac1m\sum_{j=1}^m
\Bigl[&\KL\!\left(\pi_{\mathrm{stu},\theta}(\cdot\mid\widetilde x_j)
\,\middle\|\,\pi_{\mathrm{tea}}(\cdot\mid\widetilde x_j)\right)\\
+\beta&\KL\!\left(\pi_{\mathrm{stu},\theta}(\cdot\mid\widetilde x_j)
\,\middle\|\,\pi_{\mathrm{pre}}(\cdot\mid\widetilde x_j)\right)\Bigr].
\end{aligned}
\label{eq:sm-temp-cost}
\end{equation}
For a feasible answer $a_{1:H}$, define
\begin{equation}
\begin{aligned}
Z_{\mathrm{SM},\beta}(\theta;x,a_{1:H})
&:=\log\frac{\pi_{\mathrm{stu},\theta}(a_{1:H}\mid x)}
{\pi_{\mathrm{tea}}(a_{1:H}\mid x)}\\[-0.2em]
&\quad+\beta\log\frac{\pi_{\mathrm{stu},\theta}(a_{1:H}\mid x)}
{\pi_{\mathrm{pre}}(a_{1:H}\mid x)},\\
S_{\mathrm{stu},\theta}(x,a_{1:H})
&:=\nabla_\theta\log\pi_{\mathrm{stu},\theta}(a_{1:H}\mid x).
\end{aligned}
\label{eq:sm-temp-sample}
\end{equation}
Algorithm~\ref{alg:sm-temperature} gives this direct-matching baseline.

\begin{algorithm}[H]
\caption{Direct Teacher Matching}
\label{alg:sm-temperature}
\small
\begin{algorithmic}[1]
\Require Fixed target prompts $\{\widetilde x_j\}_{j=1}^{m}$;
frozen $\pi_{\mathrm{tea}},\pi_{\mathrm{pre}}$; student class
$\Theta$; $\beta\ge0$; stepsizes $\eta_t^{\mathrm{SM}}$; $T\ge1$.
\State $\theta_0^{\mathrm{SM}}\gets0$.
\For{$t=0,\ldots,T-1$}
  \State Draw $j_t^{\mathrm{SM}}\sim\operatorname{Unif}\{1,\ldots,m\}$.
  \State Draw $a_{t,1:H}^{\mathrm{SM}}\sim
  \pi_{\mathrm{stu},\theta_t^{\mathrm{SM}}}
  (\cdot\mid\widetilde x_{j_t^{\mathrm{SM}}})$ autoregressively.
  \State $\displaystyle\widehat g_t^{\mathrm{SM}}\gets
  \begin{aligned}[t]
  &
  S_{\mathrm{stu},\theta_t^{\mathrm{SM}}}
  (\widetilde x_{j_t^{\mathrm{SM}}},a_{t,1:H}^{\mathrm{SM}})\\[-0.2em]
  &\quad\times
  Z_{\mathrm{SM},\beta}(\theta_t^{\mathrm{SM}};
  \widetilde x_{j_t^{\mathrm{SM}}},a_{t,1:H}^{\mathrm{SM}}).
  \end{aligned}$
  \Statex \Comment{Estimate the cost gradient.}
  \State $\displaystyle\theta_{t+1}^{\mathrm{SM}}\gets
  \Proj{\Theta}{\theta_t^{\mathrm{SM}}
  -\eta_t^{\mathrm{SM}}\widehat g_t^{\mathrm{SM}}}$.
\EndFor
\State \Return $\pi_{\mathrm{stu},\theta_T^{\mathrm{SM}}}$.
\end{algorithmic}
\end{algorithm}

The score identity makes $\widehat g_t^{\mathrm{SM}}$ unbiased for
$\nabla C_{\mathrm{SM},\beta}(\theta_t^{\mathrm{SM}})$. This baseline uses
target prompts and frozen policy likelihoods but no reward queries.

\paragraph{Separation construction.}
We compare the oracle, the SM population optimum, and IterCaL on one
fixed-design instance. Take $H=\lambda=1$, $d=2$, $D=n=4$, $m=8$, and
$B=3$, with
$W=\{w\in\RR^4:\|w\|_2\le B\}$ and
$\Theta=\{\theta\in\RR^2:\|\theta\|_2\le B\}$.
There are target prompts
$\{\widetilde x_{i,k}:i\in\{1,2\},k\in\{1,2,3,4\}\}$ and source prompts
$x_1,\ldots,x_4$. Take the two one-token verdicts to be
$\cA=\{0,1\}$. Generation terminates automatically after the verdict,
and $\pi_{\mathrm{pre}}(a\mid x)=1/2$ for $a\in\{0,1\}$.

Write $u_1,\ldots,u_4$ and $e_1,e_2$ for the standard bases of
$\RR^4$ and $\RR^2$. Let
$q=(u_1+u_2)/\sqrt2$,
$r_+=u_3/\sqrt2+u_4/2$, and
$r_-=-u_3/\sqrt2+u_4/4$.
Before specifying rewards, assign the public covariates
\begin{equation}
\begin{aligned}
(v(\widetilde x_{1,k}))_{k=1}^4&=(q,q,-q,r_+),\\
(v(\widetilde x_{2,k}))_{k=1}^4&=(q,q,-q,r_-),\\
(v(x_j))_{j=1}^4&=(u_1/\sqrt2,q,u_3/\sqrt2,
 u_1/\sqrt2+u_4/2).
\end{aligned}
\label{eq:sm-temp-covariates}
\end{equation}
Define feature maps that contain no label-dependent indicator:
\begin{equation}
\begin{aligned}
\phi(x,a)&=v(x)\mathbf1\{a=1\},\\
\phi_{\mathrm{stu}}(\widetilde x_{i,k},a)&=e_i\mathbf1\{a=1\},\\
(\phi_{\mathrm{stu}}(x_j,1))_{j=1}^4
&=(e_1,e_2,(e_1+e_2)/\sqrt2,(e_1-e_2)/\sqrt2),\\
\phi_{\mathrm{stu}}(x_j,0)&=0.
\end{aligned}
\label{eq:sm-temp-features}
\end{equation}
All feature norms are at most one. Set
\begin{equation}
w^\star_1=\sqrt2(u_1+u_3),\qquad
w_{\mathrm{tea}}=\sqrt2u_2+\tfrac12u_4.
\label{eq:sm-temp-teacher}
\end{equation}
Let $\pi_{\mathrm{tea}}=\pi_{w_{\mathrm{tea}}}$. The truth parameter
$w^\star_1$ is unknown to the learner and identified from source rewards; set
\begin{equation}
y(x)=\mathbf1\{(w^\star_1)^\top v(x)>0\},\quad
R(x,a)=\mathbf1\{a=y(x)\}.
\label{eq:sm-temp-reward}
\end{equation}
The target labels are $(1,1,0,1)$ and $(1,1,0,0)$, while all source labels
are $1$. Student group $i$ predicts verdict $1$ with probability
$\sigma(\theta_i)$; the teacher logits are $(1,1,-1,1/4)$ and
$(1,1,-1,1/8)$. Hence the teacher matches the regularized optimal policy on six prompts and weakly
favors $1$ on the two oppositely labeled prompts, so its likelihoods do not
reveal their labels.

The four source covariates span $\RR^4$, and
\begin{equation}
G_{\mathrm{joint}}=\frac18\sum_{j=1}^4v(x_j)v(x_j)^\top\succ0.
\label{eq:sm-temp-joint-information}
\end{equation}
Hence source rewards identify every calibration direction, including the
$u_3$ direction that separates the two low-confidence target prompts.
IterCaL still draws its actual alternatives from the current student;
Lemma~\ref{lem:itercal-calibration} lower-bounds their information matrix by
$e^{-2B}G_{\mathrm{joint}}$.
All gaps below are strict and persist when the two $u_4$ components and the
teacher's $u_4$ coefficient are perturbed slightly, so the separation is not
an exact-symmetry artifact. Appendix~\ref{app:sm-details} gives the proof.

Write $J(\pi):=J_{1,8}(\pi)$, and let $\overline R(\pi)$ denote
the mean unregularized reward over the eight target prompts.
For student policies, abbreviate these by $J(\theta)$ and
$\overline R(\theta)$, and set $\cK(\theta):=\cK_{1,8}(\theta)$.

\begin{proposition}
\label{prop:sm-temp-oracle}
The construction satisfies Assumptions~\ref{ass:itercal-truth}--\ref{ass:itercal-identification}.
The student class excludes the unrestricted optimum, and its unique
oracle is
\begin{equation}
\theta^\dagger_{1,8}=(1/2,0)^\top\in\operatorname{int}(\Theta).
\label{eq:sm-temp-oracle-law}
\end{equation}
For every $\theta\in\Theta$,
\begin{equation}
\begin{aligned}
\overline R(\theta)&=\frac38+\frac14\sigma(\theta_1),\\
J(\theta^\dagger_{1,8})-J(\theta)
&=\Delta_{1,8}(\theta)=\cK(\theta).
\end{aligned}
\label{eq:sm-temp-return}
\end{equation}
The teacher strictly exceeds every student policy in both
$\overline R$ and $J$.
\end{proposition}

In this construction, the oracle-student logits are $(1/2,0)^\top$, whereas
the groupwise mean teacher logits are $(5/16,9/32)^\top$. A shared reference
penalty only scales this non-axis direction and therefore cannot recover the
oracle. For the finite-iteration comparison, specialize
Algorithm~\ref{alg:sm-temperature} to $\beta=1$, write
$C_{\mathrm{SM}}:=C_{\mathrm{SM},1}$ and
$Z_{\mathrm{SM}}:=Z_{\mathrm{SM},1}$, and define
$\kappa_{\mathrm{SM}}:=\sigma'(B)$,
$G_{\mathrm{SM}}:=3B$, $\mu_{\mathrm{SM}}:=\kappa_{\mathrm{SM}}$, and
$\eta_t^{\mathrm{SM}}:=1/[\mu_{\mathrm{SM}}(t+2)]$.

\begin{theorem}
\label{thm:sm-temperature}
For every shared $\beta\ge0$, the unique population minimizer is
\[
\begin{gathered}
\theta_{\mathrm{SM}}^\star(\beta)
=\frac{(5/16,9/32)^\top}{1+\beta},\quad
\cK\!\left(\theta_{\mathrm{SM}}^\star(\beta)\right)
\ge\frac{\sigma'(1/2)}{36}>0,\\
\overline R(\theta^\dagger_{1,8})
-\overline R(\theta_{\mathrm{SM}}^\star(\beta))
\ge\frac{\sigma(1/2)-\sigma(5/16)}4>0.
\end{gathered}
\]
For $\beta=1$ and every $T\ge1$, Algorithm~\ref{alg:sm-temperature} satisfies
\begin{equation}
\EE\!\left[\cK(\theta_T^{\mathrm{SM}})\right]
\ge\frac{\kappa_{\mathrm{SM}}}{4}
\left[\frac{\sqrt{565}}{64}
-\frac{G_{\mathrm{SM}}}{\mu_{\mathrm{SM}}\sqrt{T+1}}\right]_+^2,
\label{eq:sm-temp-finite}
\end{equation}
where $[u]_+:=\max\{u,0\}$.
With both student parameters initialized at zero, IterCaL's limiting expected
unregularized and regularized rewards are strictly higher than those of
this $\beta=1$ baseline.
\end{theorem}

Indeed, $\theta_{\mathrm{SM}}^\star(1)=(5/32,9/64)^\top$, so its
limiting unregularized-reward gap is
$[\sigma(1/2)-\sigma(5/32)]/4>0$, and its limiting
regularized-reward gap equals
$\cK(\theta_{\mathrm{SM}}^\star(1))>0$.
The identity in Proposition~\ref{prop:sm-temp-oracle} permits
$a_{1,8}=p_{1,8}=1$ in Lemma~\ref{lem:itercal-lojas}; hence
Lemma~\ref{lem:itercal-optimization} gives
$\EE[\cK(\theta_T)]\le K_{\mathrm{opt}}(T+1)^{-1/4}$ for IterCaL
on this instance.

Theorem~\ref{thm:sm-temperature} gives the final separation: every shared
$\beta$ leaves direct matching with positive oracle-KL, regularized-reward,
and unregularized-reward gaps, whereas Theorem~\ref{thm:itercal-main} gives
IterCaL's convergence to the oracle. Source feedback identifies the reward
direction unresolved by the two low-confidence teacher predictions; direct
matching only shrinks their biased group averages.

\section{Discussion}
\label{sec:discussion}

IterCaL uses source rewards to calibrate the policy guiding target-side
student learning. For fixed designs, it converges polynomially to the
regularized oracle using $T$ source queries and $O(T^2)$ target rollouts;
under approximate source-to-target calibration transfer, it attains the explicit transfer bound
in Corollary~\ref{cor:itercal-approx-target}. The separation shows that every
shared $\beta$ leaves direct matching with positive reward gaps, while IterCaL
reaches the oracle.

Corollary~\ref{cor:itercal-approx-target} also quantifies calibration transfer.
Source rewards identify an exact source calibration, while the target-best
calibration measures how closely the teacher class can represent the optimal
target policy. This target approximation error and the source-to-target
calibration shift upper-bound $\delta_{\mathrm{tar}}$, the error from applying
the source-optimal calibration on target prefixes. Autoregressive
factorization accumulates at most
$H\delta_{\mathrm{tar}}$ per answer, yielding the explicit
$2\lambda H\delta_{\mathrm{tar}}$ term without target reward queries.
When true-optimal students share their target answer laws, the same
perturbation controls oracle-policy KL through a fixed-model \L{}ojasiewicz
power.

Future work should evaluate IterCaL on language models, study less frequent
calibration, and sharpen the model-dependent \L{}ojasiewicz exponent and its
horizon and dimension dependence.

%% file: sections/appendix_related_work.tex
\section{Related Work}
\label{sec:related-work}

We place IterCaL in three related literatures: statistical distillation and
imitation learning, reinforcement learning and LLM post-training, and transfer
learning under covariate shift.

\paragraph{Statistical theory of distillation and imitation learning.}
Statistical analyses of distillation study the benefits of
teacher-generated supervision and the propagation of teacher error.
\citet{menon2021statistical} explain the bias–variance tradeoff
of soft labels, while \citet{ildiz2025highdimensional} characterize
high-dimensional distillation risk under model and covariate shift.
\citet{xie2026mathematical} identify settings where a risk-minimizing
teacher preserves the restricted student's optimum and improves
the statistical efficiency of averaged SGD.
To correct imperfect supervision, \citet{dao2021semiparametric}
develop cross-fitting and loss corrections, and
\citet{iliopoulos2022weighted} analyze student-dependent
reweighting of noisy pseudo-labels.
Relatedly, \citet{xia2024prediction} construct pseudo-responses
using training-only helper covariates and obtain prediction bounds
combining an oracle rate with surrogate error.
For sequential prediction, \citet{ross2011reduction} establish
a no-regret foundation for learning from expert feedback at
learner-visited states, while \citet{czarnecki2019distilling}
analyze policy-distillation updates and convergence in tabular
settings.
\citet{foster2024behavior} show that, under realizability,
online expert access need not improve worst-case statistical
complexity over offline behavior cloning with logarithmic loss.
With noisy expert feedback, \citet{sriraman2026behavior}
establish an offline on-policy separation for learning a
realizable clean expert.
Beyond realizability, \citet{zhang2026online} study how student
misspecification and alignment between expert scores and rewards
affect the benefits of online imitation, and give finite-sample
guarantees using base-policy sampling.
Closest to our motivation, \citet{yamamoto2026residual} study
bias propagation under source-target covariate shift.
Their method refits teachers to student residuals and achieves
a provable separation from direct soft matching.
Our work alternates source-reward-based teacher calibration
with target-side student learning.
We establish convergence to the regularized oracle within the
student class and separation from regularized direct teacher matching,
without target reward feedback or requiring the student to
represent the unrestricted optimal policy.

\paragraph{Reinforcement learning and LLM post-training.}
Reinforcement learning plays an important role in modern deep learning,
with a substantial theoretical literature on exploration, policy
optimization, and learning with function approximation
\citep{jiang2017contextual,jin2020provably,agarwal2021policy,
xie2021bellman,xie2023coverage,zhang2023offline,
qian2024offline,hu2026double}.
For preference-based learning, \citet{zhu2023principled} connect
reward estimation to policy performance and establish guarantees
for pessimistic learning.
\citet{xiong2024iterative} develop offline, online, and hybrid
algorithms for KL-regularized preference learning with finite-sample
guarantees.
\citet{xie2025exploratory} introduce exploration bonuses for
provably sample-efficient preference optimization, while
\citet{zhao2025sharp} characterize how KL regularization and
reference-policy coverage affect sample complexity.
For LLM post-training, \citet{chen2026coverage} study how
pre-training provides response coverage for downstream improvement,
and \citet{foster2025foundation} distinguish the statistical
and computational roles of base-model coverage.
\citet{huang2025sharpening} analyze self-improvement through
sharpening, where training amortizes the selection of high-likelihood
responses.
Related inference-time analyses characterize Best-of-$N$
through win-rate guarantees \citep{sriraman2026bestofn}
and use pessimistic scoring to mitigate reward hacking
\citep{yu2026caution}.
For outcome-supervised learning, \citet{jia2025verify}
connect outcome feedback to process-level learning,
\citet{yuan2025trajectory} develop trajectory Bellman residual
minimization for KL-regularized policy learning, and
\citet{chen2025outcome} establish sample-complexity guarantees
for outcome-based online RL.
\citet{kim2026coverage} analyze coverage improvement and
convergence in on-policy preference learning and reward
distillation, explicitly accounting for reward-model error
in the latter.
These works study how feedback and coverage support policy
improvement and response selection.
Our work addresses biased-teacher LLM distillation under
source-target covariate shift.
We alternate source-reward-based teacher calibration with target-side
student learning and prove convergence in average target KL
to the regularized oracle within the student class,
together with separation from regularized direct teacher matching,
without target reward feedback.

\paragraph{Transfer learning under covariate shift.}
Transfer-learning theory studies how source supervision supports
prediction on a different target distribution. Under covariate shift,
linear and kernel regression
analyses quantify this transfer through source-target covariance
geometry and distributional overlap
\citep{lei2021nearoptimal,ma2023optimally}.
In well-specified parametric models, \citet{ge2024mle} establish
minimax guarantees for source-only maximum likelihood estimation,
with transfer difficulty governed by source and target Fisher
information rather than a bounded density ratio.
Pseudo-labeling methods further use unlabeled target covariates
to select source-trained estimators in kernel ridge regression
and kernel generalized linear models
\citep{wang2026pseudolabeling,weill2026pseudolabeling}.
Beyond covariate shift, \citet{xia2026classification} study
oversampling under label shift, separating balanced-data risk
from the cost of estimating the minority-class distribution.
These results characterize prediction under distribution shift,
but do not study LLM distillation without target feedback. Under shared policy realizability
and joint source identification, our analysis controls the student's
error on fixed target questions using reward feedback confined
to source questions.

%% file: sections/04_appendix.tex
\section{Mathematical Tools}
This appendix records the two \L{}ojasiewicz inequalities used to convert
objective progress into oracle-policy KL convergence in
Theorem~\ref{thm:itercal-main}.

\begin{lemma}[\L{}ojasiewicz inequality \cite{bierstone-milman-1988}]
\label{lem:itercal-compact-lojas}
Let $K\subseteq\RR^q$ be nonempty, and let $f,g:K\to\RR$.
Suppose that their graphs are compact subanalytic subsets of
$\RR^{q+1}$ and that
\[
\{x\in K:f(x)=0\}\subseteq\{x\in K:g(x)=0\}.
\]
Then there are constants $c>0$ and $\rho>0$ for which
\[
|f(x)|\ge c|g(x)|^\rho\ (x\in K).
\]
In particular, if $f,g\ge0$ on $K$, we choose any
$M\ge\max\{1,\sup_{x\in K}g(x)\}$. Then, the same constants
$c,\rho$ defined above yield
\[
f(x)\ge a[g(x)]^p,\ 
\ p:=\max\{1,\rho\}\ge1,
\ a:=cM^{\rho-p}>0.
\]
\end{lemma}

To quantify the progress of a projected student step near the oracle,
we use the following subgradient form of the \L{}ojasiewicz inequality.
The domain in this result can include the boundary of the student
parameter ball.

\begin{lemma}[{\citealp[Theorem~3.1]{bolte2007lojasiewicz}}]
\label{lem:itercal-subgradient-lojas}
Let $f:\RR^d\to\RR\cup\{+\infty\}$ have a subanalytic graph
and closed domain, and suppose that $f$ is continuous on its domain.
If $\bar\theta$ is a critical point, meaning
$0\in\partial f(\bar\theta)$, then there are a neighborhood
$U$ of $\bar\theta$, $c>0$, and $\rho\in[0,1)$ such that
\[
\operatorname{dist}(0,\partial f(\theta))
\ge c\,|f(\theta)-f(\bar\theta)|^\rho
\ 
\text{for }\theta\in U\cap\operatorname{dom}f
\text{ with } f(\theta)\ne f(\bar\theta).
\]
Here $\partial f$ is the limiting subdifferential and
$\operatorname{dist}(0,\partial f(\theta))
:=\inf_{v\in\partial f(\theta)}\|v\|_2$.
\end{lemma}

\section{Proofs in Section \ref{sec:itercal-theory}}\label{app:proofs_thm}
Using these tools, this appendix expands the proof sketch in
Section~\ref{sec:itercal-theory}. We first specify the query interfaces,
then state the auxiliary lemmas and combine them to prove
Theorem~\ref{thm:itercal-main}; Appendix~\ref{app:lemma-proofs} contains
the lemma proofs.

\paragraph{Implementation under the access model.}
Algorithm~\ref{alg:itercal} constructs each source comparison from a tokenized
teacher answer and one current-student token at the same prefix.
Feature queries at these two tokens suffice for the logistic update.
Reference branch selection needs their reference probability ratio,
including when neither token was sampled from the reference; completion
uses the exact reference conditional after the selected token.
The source verifier evaluates the resulting complete answer.
On target prompts, the current student and both proposed students must
be callable at their specified parameters. Their sampled answers require
student and calibrated-policy likelihoods; the gradient batch also
requires the student score.
These are exact query interfaces: observed-token features alone do not
supply normalized likelihoods or scores. In particular, the two different
policies' normalizers do not generally cancel in the target log ratio.
No vocabulary enumeration is required of the learner, although a model
implementation may process the vocabulary internally.
The rollout and reward-query counts exclude the cost of implementing
these interfaces and obtaining the supplied $w_{\mathrm{tea}}$ and $\gamma$.

We begin by identifying the correct distillation objective: the following
lemma shows that maximizing the regularized student return is equivalent
to minimizing the KL cost against $\pi_{w^\star_\lambda}$.

\begin{lemma}
\label{lem:itercal-gibbs}
For each source or target prompt $x$, define
\[
Z_\lambda(x):=\sum_{b_{1:H}\in\cA(x)}
\pi_{\mathrm{pre}}(b_{1:H}\mid x)e^{R(x,b_{1:H})/\lambda}.
\]
The unique unrestricted optimal answer policy is
\begin{equation}
\pi^\star_\lambda(a_{1:H}\mid x)
=\frac{\pi_{\mathrm{pre}}(a_{1:H}\mid x)
e^{R(x,a_{1:H})/\lambda}}{Z_\lambda(x)}.
\label{eq:itercal-gibbs}
\end{equation}
Moreover, for every student parameter,
\begin{equation}
J_{\lambda,m}(\pi_{\mathrm{stu},\theta})
=\frac{\lambda}{m}\sum_{j=1}^{m}
\log Z_\lambda(\widetilde x_j)-C_{w^\star_\lambda}(\theta).
\label{eq:itercal-gibbs-student}
\end{equation}
\end{lemma}

The Gibbs representation in Lemma~\ref{lem:itercal-gibbs} lets us identify
the distribution of accepted branch indices. The following lemma shows
that these indices form logistic observations with parameter
$w^\star_\lambda$, providing the statistical basis for the source update.

We now define a filtration $\cbr{\cF_t}_{t=0}^T$. We treat the source and target datasets, the reward verifier,
the policy features, the frozen policies, and the algorithmic
schedules as fixed. Let $\cF_t$ denote the information available
immediately before the source sampling in round $t$.
Formally, set
$\cF_0=\sigma(w_0,\theta_0)$. For every $t\ge 0$, we define $\cF_t$ recursively 
\[
\begin{aligned}
\cF_{t+1}=\cF_t\vee\sigma\Bigl(
& i_t,h_t,a_{t,1:H},c_{t,0},Y_t,
  a^{\mathrm{pre}}_{t,1:H},U_t,\bigl(
    j_{t+1,\ell}^{\mathrm g},
    a^{\mathrm g}_{t+1,\ell,1:H}
  \bigr)_{\ell=1}^{b_{t+1}},
  \vartheta_{t+1,2},\bigl(
    j_{t+1,k,\ell},
    a^{\mathrm{val}}_{t+1,k,\ell,1:H}
  \bigr)_{\substack{k\in\{1,2\}\\
                    \ell=1,\ldots,q_{t+1}}}
\Bigr).
\end{aligned}
\]
The displayed variables comprise the source samples, the student-gradient
batch, and the exploration candidate and candidate-validation samples.

All other quantities computed in round $t$, including
$s_t,c_{t,1},z_t,R_t,I_t,g_t$, the gradient candidate
$\vartheta_{t+1,1}$, and the selected student, are measurable
functions of $\cF_t$ and these samples. In particular,
$w_t$ and $\theta_t$ are $\cF_t$-measurable, whereas
$w_{t+1}$ and $\theta_{t+1}$ are
$\cF_{t+1}$-measurable.
For deterministic initialization, $\cF_0$ is the trivial
sigma-algebra.
Throughout the following lemmas, conditioning on
$\cF_t,s_t,c_{t,1},c_{t,0}$ means conditioning on $\cF_t\vee\sigma(s_t,c_{t,1},c_{t,0})$.
Then, we have the following lemma.
\begin{lemma}
\label{lem:itercal-label}
Under Assumption~\ref{ass:itercal-truth}, for every round
$t\ge 0$, Algorithm~\ref{alg:itercal} satisfies, almost surely,
\begin{equation}
\Pr\!\left(
I_t=1
\,\middle|\,
\cF_t,s_t,c_{t,1},c_{t,0}
\right)
\ge e^{-1/\lambda}.
\label{eq:itercal-acceptance}
\end{equation}
Moreover, conditional on acceptance, the branch index satisfies
\begin{align}
&\Pr\!\left(
Y_t=1
\,\middle|\,
I_t=1,\cF_t,s_t,c_{t,1},c_{t,0}
\right)=
\frac{\pi^\star_\lambda(c_{t,1}\mid s_t)}
{\pi^\star_\lambda(c_{t,1}\mid s_t)
+\pi^\star_\lambda(c_{t,0}\mid s_t)}
=\sigma(z_t^\top w^\star_\lambda),
\label{eq:itercal-accepted-label}
\end{align}
where
\[
z_t=\phi(s_t,c_{t,1})-\phi(s_t,c_{t,0}),
\ 
\sigma(u)=\frac{1}{1+e^{-u}}.
\]
These statements also hold when the two proposal tokens coincide,
including prefixes after $\mathtt{EOS}$. The random variable $Y_t$
denotes the branch index, not the token value.
\end{lemma}

Combining the accepted-label identity in Lemma~\ref{lem:itercal-label}
with joint source identification, we obtain convergence of the calibration
updates under adaptive student proposals. The bounds below will
control both the calibration error and the changes in the target
objective across rounds.

For $\theta\in\Theta$, define the information matrix generated by the
student proposal as
\begin{equation}
\begin{aligned}
G_{\mathrm{src}}(\theta)
:={}&\frac1{nH}\sum_{i=1}^{n}\sum_{h=1}^{H}
\EE_{a_{i,1:h}\sim\pi_{\mathrm{tea}}(\cdot\mid x_i)}
\Bigg[\sum_{b\in\cB(s_{i,h})}
\pi_{\mathrm{stu},\theta}(b\mid s_{i,h})\\
&\qquad\times z_{i,h}(b)z_{i,h}(b)^\top\Bigg],
\qquad
z_{i,h}(b):=\phi(s_{i,h},a_{i,h})-\phi(s_{i,h},b).
\end{aligned}
\label{eq:itercal-student-information}
\end{equation}
Thus $G_{\mathrm{src}}(\theta)$ records the calibration directions
sampled when the current student supplies the alternative token.

\begin{lemma}
\label{lem:itercal-calibration}
With $G_{\mathrm{joint}}$ and $\gamma$ defined in
\eqref{eq:itercal-gram}-\eqref{eq:itercal-gamma}, every round $t\ge0$
satisfies
\begin{align}
\EE\!\left[z_tz_t^\top\mid\cF_t\right]
&=G_{\mathrm{src}}(\theta_t)
\succeq e^{-2B}G_{\mathrm{joint}},\ 
\|z_t\|_2\le2,
\label{eq:itercal-adaptive-design}\\
\EE\!\left[
(w_t-w^\star_\lambda)^\top g_t
\,\middle|\,\cF_t\right]
&\ge e^{-1/\lambda}\sigma'(2B)
(w_t-w^\star_\lambda)^\top G_{\mathrm{src}}(\theta_t)
(w_t-w^\star_\lambda)\nonumber\\
&\ge\gamma\|w_t-w^\star_\lambda\|_2^2,
\qquad \|g_t\|_2\le2.
\label{eq:itercal-source-drift}
\end{align}
Consequently, for every integer $t\ge1$, the projected updates in
Algorithm~\ref{alg:itercal} satisfy
\begin{align}
\EE\!\left[\|w_t-w^\star_\lambda\|_2^2\right]
&\le\frac4{\gamma^2(t+2)},
\label{eq:itercal-source-mse}\\
\|w_t-w_{t-1}\|_2
&\le\frac2{\gamma(t+1)}.
\label{eq:itercal-source-increment}
\end{align}
\end{lemma}

Having controlled the source updates, we turn to the quantities estimated
from target rollouts. We first bound the log-ratio cost of a single answer,
which will be used to control the population cost and the error in its
Monte Carlo evaluation.

\begin{lemma}
\label{lem:itercal-z}
For $t=0,\ldots,T$, $w_t\in W$, $\vartheta\in\Theta$, and a
feasible answer $a_{1:H}\in\cA(x)$ at a target prompt $x$, define
\[
Z_t(\vartheta;x,a_{1:H})
:=\lambda\log
\frac{\pi_{\mathrm{stu},\vartheta}(a_{1:H}\mid x)}
{\pi_{w_t}(a_{1:H}\mid x)}
=\lambda\sum_{h=1}^{H}
\log\frac{\pi_{\mathrm{stu},\vartheta}(a_h\mid x,a_{1:h-1})}
{\pi_{w_t}(a_h\mid x,a_{1:h-1})}.
\]
Then, uniformly over these choices,
\begin{equation}
|Z_t(\vartheta;x,a_{1:H})|\le4\lambda BH.
\label{eq:itercal-z-bound}
\end{equation}
\end{lemma}

The preceding trajectory bound gives a uniform bound on its expected cost.
We also quantify sensitivity to the student and calibration parameters,
so that parameter changes can be translated into cost changes in the
optimization analysis.

\begin{lemma}
\label{lem:itercal-lipschitz}
For $w\in W$ and $\theta\in\Theta$, define $C_w(\theta):=\frac{\lambda}{m}\sum_{j=1}^{m}
\operatorname{KL}\!\left(
\pi_{\mathrm{stu},\theta}(\cdot\mid\widetilde x_j)
\,\middle\|\,
\pi_w(\cdot\mid\widetilde x_j)\right)$.

The cost used in round $t$ is $C_t(\theta):=C_{w_t}(\theta)
=\frac1m\sum_{j=1}^{m}
\EE_{a_{1:H}\sim\pi_{\mathrm{stu},\theta}(\cdot\mid\widetilde x_j)}
\!\left[Z_t(\theta;\widetilde x_j,a_{1:H})\right]$,
where $Z_t$ is defined in Lemma~\ref{lem:itercal-z}.
For every $w,v\in W$ and $\theta,\vartheta\in\Theta$,
\begin{align}
0\le C_w(\theta)&\le4\lambda BH,
\label{eq:itercal-cost-bound}\\
|C_w(\theta)-C_w(\vartheta)|
&\le4\lambda BH^{3/2}\|\theta-\vartheta\|_2,
\label{eq:itercal-stu-lip}\\
\sup_{\theta\in\Theta}|C_w(\theta)-C_v(\theta)|
&\le2\lambda H\|w-v\|_2.
\label{eq:itercal-w-lip}
\end{align}
In particular, these bounds apply to $C_t$ by setting $w=w_t$.
\end{lemma}

For the newly calibrated objective $C_{t+1}$, we next verify that the
algorithm's fresh target batch provides an unbiased gradient estimate.
We also bound its variance, which controls the error in the
projected-gradient proposal.

\begin{lemma}
\label{lem:itercal-unbiased}
Let $\cF_t$ be the history before the source draws in round $t$,
so that $\theta_t$ is $\cF_t$-measurable, and set
$\cG_{t+1}:=\cF_t\vee\sigma(w_{t+1})$. Define the score
\[
S_{\mathrm{stu},\theta}(x,a_{1:H})
:=\nabla_\theta\log\pi_{\mathrm{stu},\theta}(a_{1:H}\mid x).
\]
Conditional on $\cG_{t+1}$, independently for
$\ell=1,\ldots,b_{t+1}$, draw
\[
j_{t+1,\ell}^{\mathrm g}\sim\operatorname{Unif}\{1,\ldots,m\},
\ 
a^{\mathrm g}_{t+1,\ell,1:H}
\sim\pi_{\mathrm{stu},\theta_t}
(\cdot\mid\widetilde x_{j_{t+1,\ell}^{\mathrm g}}),
\]
and define
\begin{align*}
\widehat g_{t+1,\ell}^{\mathrm{stu}}
&:=S_{\mathrm{stu},\theta_t}
(\widetilde x_{j_{t+1,\ell}^{\mathrm g}},a^{\mathrm g}_{t+1,\ell,1:H})
Z_{t+1}(\theta_t;\widetilde x_{j_{t+1,\ell}^{\mathrm g}},
a^{\mathrm g}_{t+1,\ell,1:H}),\ \widehat g_{t+1}^{\mathrm{stu}}
:=\frac1{b_{t+1}}\sum_{\ell=1}^{b_{t+1}}
\widehat g_{t+1,\ell}^{\mathrm{stu}}.
\end{align*}
Here $Z_{t+1}$ and $C_{t+1}=C_{w_{t+1}}$ are defined in
Lemmas~\ref{lem:itercal-z} and \ref{lem:itercal-lipschitz}, respectively.
For every $t=0,\ldots,T-1$, almost surely, we have
\begin{align}
\EE\!\left[\widehat g_{t+1}^{\mathrm{stu}}
\,\middle|\,\cG_{t+1}\right]
&=\EE\!\left[\widehat g_{t+1}^{\mathrm{stu}}
\,\middle|\,w_{t+1},\theta_t\right]
=\nabla_\theta C_{t+1}(\theta_t),
\label{eq:itercal-student-unbiased}
\end{align}
\begin{align}
\EE\!\left[\left\|\widehat g_{t+1}^{\mathrm{stu}}
-\nabla_\theta C_{t+1}(\theta_t)\right\|_2^2
\,\middle|\,\cG_{t+1}\right]
&\le\frac{16\lambda^2B^2H^3}{b_{t+1}}.
\label{eq:itercal-student-variance}
\end{align}
\end{lemma}

The preceding bounds control the size of the student gradient. We next
bound its variation, both as the student parameter changes and as the
calibrated teacher changes. These bounds determine a fixed student
step size and control the error caused by using the current calibration.

\begin{lemma}
\label{lem:itercal-smoothness}
Recall the cost $C_w$ in \eqref{eq:itercal-cost} and the
oracle-policy average KL divergence $\cK_{\lambda,m}$ in \eqref{eq:itercal-errors}.
For any answer laws $Q_j$ that are positive on
$\cA(\widetilde x_j)$, the function
$\theta\mapsto m^{-1}\sum_{j=1}^{m}
\KL(\pi_{\mathrm{stu},\theta}(\cdot\mid\widetilde x_j)\|Q_j)$
is real analytic on $\RR^d$. In particular, this holds for
$C_w$ and $\cK_{\lambda,m}$. Define
\[
L_{\mathrm{st}}:=\lambda H(1+8BH),
\ 
\alpha_{\mathrm{stu}}:=\frac{1}{2L_{\mathrm{st}}}.
\]
For all $w,v\in W$ and $\theta,\vartheta\in\Theta$, we have
\begin{align}
\|\nabla_\theta^2 C_w(\theta)\|_{\mathrm{op}}
&\le L_{\mathrm{st}},
\label{eq:itercal-hessian-bound}\\
\|\nabla_\theta C_w(\theta)-\nabla_\theta C_w(\vartheta)\|_2
&\le L_{\mathrm{st}}\|\theta-\vartheta\|_2,
\label{eq:itercal-gradient-smoothness}\\
\|\nabla_\theta C_w(\theta)-\nabla_\theta C_v(\theta)\|_2
&\le2\lambda H^{3/2}\|w-v\|_2.
\label{eq:itercal-gradient-calibration}
\end{align}
\end{lemma}

After constructing the candidates, the algorithm chooses among them
using sampled costs. The uniform log-ratio bound in
Lemma~\ref{lem:itercal-z} allows the following lemma to control how far the
selected candidate's true cost can exceed the better of the two proposals.

\begin{lemma}
\label{lem:itercal-selection}
For each target update $t\ge1$, define $\cH_t
:=
\cF_{t-1}\vee
\sigma\!\left(
w_t,\vartheta_{t,1},\vartheta_{t,2}
\right)$. Define the nonnegative, proof-only comparison error $\xi_t:=2\max_{k\in\{1,2\}}
|\widehat C_{t,k}-C_t(\vartheta_{t,k})|$. Then, we have
\begin{equation}
C_t(\theta_t)\le\min_{k\in\{1,2\}}C_t(\vartheta_{t,k})+\xi_t,
\ 
\EE[\xi_t\mid\cH_t]\le\frac{16\lambda BH}{\sqrt{t+1}}.
\label{eq:itercal-selection}
\end{equation}
\end{lemma}

The next lemma combines two sources of improvement for the fixed true
cost. A projected gradient step provides progress near the oracle;
uniform exploration provides progress when the current cost is bounded
away from the minimum. The exact true gradient is used only to define
the comparison step in this lemma.

\begin{lemma}
\label{lem:itercal-hybrid}
Let $F(\theta):=C_{w^\star_\lambda}(\theta)$ and
$\Delta_{\lambda,m}(\theta)
:=F(\theta)-F(\theta^\dagger_{\lambda,m})$.
Define
\[
y(\theta):=\Proj{\Theta}{\theta-\alpha_{\mathrm{stu}}\nabla_\theta F(\theta)},
\ 
\vartheta^{\mathrm{unif}}\sim\operatorname{Unif}(\Theta),
\ 
M_{\mathrm{opt}}:=8\lambda B^2H^{3/2}.
\]
Under Assumption~\ref{ass:itercal-unique}, there is a fixed-model
constant $0<\kappa_{\mathrm{opt}}\le1/(4M_{\mathrm{opt}})$ such
that, for every $\theta\in\Theta$,
\begin{equation}
\EE_{\vartheta^{\mathrm{unif}}}\!\left[
F(\theta)-\min\bigl\{
F(y(\theta)),F(\vartheta^{\mathrm{unif}})
\bigr\}\right]
\ge\kappa_{\mathrm{opt}}
[\Delta_{\lambda,m}(\theta)]^2.
\label{eq:itercal-hybrid-progress}
\end{equation}
The constant does not depend on the iteration number.
\end{lemma}

We now compare the actual student update with the population
update in Lemma~\ref{lem:itercal-hybrid}. The preceding gradient and
selection bounds control the difference between these updates.
This gives a recursion for the excess true cost, with separate
contributions from gradient estimation, calibration, and validation.

\begin{lemma}
\label{lem:itercal-optimization}
Recall $\Delta_{\lambda,m}(\theta)$ and $\varepsilon_t$ from
\eqref{eq:itercal-errors}. Let $\kappa_{\mathrm{opt}}>0$ be the
constant in Lemma~\ref{lem:itercal-hybrid}, chosen so that
$\kappa_{\mathrm{opt}}\le1/(4M_{\mathrm{opt}})$, where
$M_{\mathrm{opt}}:=8\lambda B^2H^{3/2}$. Define
\[
\begin{aligned}
E_{\mathrm{opt}}
&:=16\alpha_{\mathrm{stu}}\lambda^2B^2H^3
+\frac{16\alpha_{\mathrm{stu}}\lambda^2BH^3+8\lambda H}{\gamma}
+16\lambda BH,\\
K_{\mathrm{opt}}
&:=\max\left\{
M_{\mathrm{opt}},
\sqrt{\frac{2E_{\mathrm{opt}}}{\kappa_{\mathrm{opt}}}},
\frac1{2\kappa_{\mathrm{opt}}}
\right\}.
\end{aligned}
\]
Then Algorithm~\ref{alg:itercal} satisfies, for every integer $T\ge1$,
\begin{equation}
\EE[\Delta_{\lambda,m}(\theta_T)]
\le K_{\mathrm{opt}}(T+1)^{-1/4},
\ 
\EE[\varepsilon_T]
\le K_{\mathrm{opt}}(T+1)^{-1/4}
+\frac{8\lambda H}{\gamma\sqrt{T+2}}.
\label{eq:itercal-optimization-bound}
\end{equation}
\end{lemma}

Lemma~\ref{lem:itercal-optimization} gives the return-gap bound in
Theorem~\ref{thm:itercal-main}. For its oracle-policy KL bound, the following
lemma uses analyticity, compactness, and Assumption~\ref{ass:itercal-unique}
to connect the two errors.

\begin{lemma}
\label{lem:itercal-lojas}
Under the model setup and Assumption~\ref{ass:itercal-unique}, fix
$\theta^\dagger_{\lambda,m}\in\argmin_{\theta\in\Theta}
C_{w^\star_\lambda}(\theta)$.
Recall the excess true cost and the average oracle-policy KL from
\eqref{eq:itercal-errors}:
\[
\begin{aligned}
\Delta_{\lambda,m}(\theta)
&:=C_{w^\star_\lambda}(\theta)
-C_{w^\star_\lambda}(\theta^\dagger_{\lambda,m}),\\
\cK_{\lambda,m}(\theta)
&:=\frac1m\sum_{j=1}^{m}
\KL\!\left(
\pi_{\mathrm{stu},\theta}(\cdot\mid\widetilde x_j)
\,\middle\|\,
\pi_{\mathrm{stu},\theta^\dagger_{\lambda,m}}
(\cdot\mid\widetilde x_j)\right).
\end{aligned}
\]
There exist constants $a_{\lambda,m}>0$ and $p_{\lambda,m}\ge1$
such that
\begin{equation}
\Delta_{\lambda,m}(\theta)
\ge a_{\lambda,m}[\cK_{\lambda,m}(\theta)]^{p_{\lambda,m}}
\ \text{for every }\theta\in\Theta.
\label{eq:itercal-lojas}
\end{equation}
These constants depend on the fixed model and target design,
but not on the iteration number.
\end{lemma}

Combining Lemmas~\ref{lem:itercal-optimization} and \ref{lem:itercal-lojas}
through Jensen's inequality now yields the convergence rate in
Theorem~\ref{thm:itercal-main}, as shown in the proof below.

\begin{proof}[Proof of Theorem~\ref{thm:itercal-main}]
Lemma~\ref{lem:itercal-gibbs} identifies the excess true cost with the
regularized target return gap:
\[
\Delta_{\lambda,m}(\theta)
=J_{\lambda,m}(\pi_{\mathrm{stu},\theta^\dagger_{\lambda,m}})
-J_{\lambda,m}(\pi_{\mathrm{stu},\theta}).
\]
Thus Lemma~\ref{lem:itercal-optimization} proves
\eqref{eq:itercal-main-gap} with its stated $K_{\mathrm{opt}}$.
For the policy error in \eqref{eq:itercal-main-kl},
the definition in \eqref{eq:itercal-errors} gives
\[
\EE\!\left[\frac1m\sum_{j=1}^{m}
\KL\!\left(
\pi_{\mathrm{stu},\theta_T}(\cdot\mid\widetilde x_j)
\,\middle\|\,
\pi_{\mathrm{stu},\theta^\dagger_{\lambda,m}}
(\cdot\mid\widetilde x_j)\right)\right]
=\EE[\cK_{\lambda,m}(\theta_T)].
\]
We first apply Lemma~\ref{lem:itercal-lojas} to compare this policy
error with the excess true cost. Its inequality holds for every
$\theta\in\Theta$, and hence on every sample path at $\theta_T$:
\[
a_{\lambda,m}[\cK_{\lambda,m}(\theta_T)]^{p_{\lambda,m}}
\le\Delta_{\lambda,m}(\theta_T).
\]
Dividing by $a_{\lambda,m}>0$ and taking the increasing power
$1/p_{\lambda,m}$ yields
\[
\cK_{\lambda,m}(\theta_T)
\le a_{\lambda,m}^{-1/p_{\lambda,m}}
[\Delta_{\lambda,m}(\theta_T)]^{1/p_{\lambda,m}}.
\]
The excess cost is bounded by Lemma~\ref{lem:itercal-lipschitz},
so the expectations below are finite. Since $p_{\lambda,m}\ge1$,
the function $u\mapsto u^{1/p_{\lambda,m}}$ is concave:
for $u>0$,
\[
\frac{d^2}{du^2}u^{1/p_{\lambda,m}}
=\frac1{p_{\lambda,m}}
\left(\frac1{p_{\lambda,m}}-1\right)
u^{1/p_{\lambda,m}-2}\le0,
\]
and continuity extends concavity to zero. Taking expectations
and applying Jensen's inequality therefore gives
\begin{align*}
\EE[\cK_{\lambda,m}(\theta_T)]
&\le a_{\lambda,m}^{-1/p_{\lambda,m}}
\EE\!\left[[\Delta_{\lambda,m}(\theta_T)]^{1/p_{\lambda,m}}\right]\le a_{\lambda,m}^{-1/p_{\lambda,m}}
\left(\EE[\Delta_{\lambda,m}(\theta_T)]\right)^{1/p_{\lambda,m}}=\left(\frac{\EE[\Delta_{\lambda,m}(\theta_T)]}
{a_{\lambda,m}}\right)^{1/p_{\lambda,m}}.
\end{align*}
Next, we use Lemma~\ref{lem:itercal-optimization} to bound the numerator by
$K_{\mathrm{opt}}(T+1)^{-1/4}$. This bound already combines
the student gradient progress, calibration error, and finite-rollout
errors. Substituting it into the increasing power gives
\begin{align*}
\EE[\cK_{\lambda,m}(\theta_T)]
&\le
\left(\frac{K_{\mathrm{opt}}(T+1)^{-1/4}}
{a_{\lambda,m}}\right)^{1/p_{\lambda,m}}=\left(\frac{K_{\mathrm{opt}}}{a_{\lambda,m}}\right)^{1/p_{\lambda,m}}
(T+1)^{-1/(4p_{\lambda,m})}.
\end{align*}
Taking $C_{\lambda,m}:=(K_{\mathrm{opt}}/a_{\lambda,m})^{1/p_{\lambda,m}}$
proves \eqref{eq:itercal-main-kl}. All constants are independent of $T$,
and $p_{\lambda,m}\ge1$ is finite, so the bound converges to zero.
\end{proof}

\begin{proof}[Proof of Corollary~\ref{cor:itercal-approx-target}]
Let $\theta^{\mathrm{src}}_{\lambda,m}$ be the unique minimizer of
$C_{w^\star_{\lambda,\mathrm{src}}}$ and define
\[
\widetilde\Delta_{\lambda,m}(\theta)
:=C_{w^\star_{\lambda,\mathrm{src}}}(\theta)
-C_{w^\star_{\lambda,\mathrm{src}}}
(\theta^{\mathrm{src}}_{\lambda,m}).
\]
Since $d_{\mathrm{src}}(w^\star_{\lambda,\mathrm{src}})=0$,
$\pi^\star_\lambda=\pi_{w^\star_{\lambda,\mathrm{src}}}$ on source prefixes,
so the accepted branch label has conditional mean
$\sigma(z_t^\top w^\star_{\lambda,\mathrm{src}})$. The source
identification proof therefore gives
\[
\EE\|w_t-w^\star_{\lambda,\mathrm{src}}\|_2^2
\le \frac4{\gamma^2(t+2)}.
\]
The proof of Lemma~\ref{lem:itercal-optimization} otherwise uses this
bound, smoothness of the fixed student objective, and uniqueness of its
minimizer. Its proof therefore applies with $w^\star_\lambda$ replaced by
$w^\star_{\lambda,\mathrm{src}}$ and
$\theta^\dagger_{\lambda,m}$ by $\theta^{\mathrm{src}}_{\lambda,m}$,
yielding
\begin{equation}
\EE[\widetilde\Delta_{\lambda,m}(\theta_T)]
\le K^{\mathrm{src}}_{\mathrm{opt}}(T+1)^{-1/4}.
\label{eq:itercal-approx-surrogate}
\end{equation}

Define the true target cost
\[
C^\star_{\lambda,m}(\theta):=\frac{\lambda}{m}\sum_{j=1}^m
\KL\!\left(\pi_{\mathrm{stu},\theta}(\cdot\mid\widetilde x_j)
\,\middle\|\,\pi^\star_\lambda(\cdot\mid\widetilde x_j)\right).
\]
For every feasible target answer, the policy chain rule and
\eqref{eq:itercal-target-discrepancy} imply
\[
\left|\log\frac{\pi^\star_\lambda(a_{1:H}\mid\widetilde x_j)}
{\pi_{w^\star_{\lambda,\mathrm{src}}}
(a_{1:H}\mid\widetilde x_j)}\right|
\le H\delta_{\mathrm{tar}}.
\]
Consequently, uniformly over $\theta\in\Theta$,
\begin{equation}
\left|C^\star_{\lambda,m}(\theta)
-C_{w^\star_{\lambda,\mathrm{src}}}(\theta)\right|
\le\lambda H\delta_{\mathrm{tar}}.
\label{eq:itercal-uniform-target-perturbation}
\end{equation}
Indeed, their difference is the student expectation of the negative of the
displayed log-likelihood ratio, averaged over target prompts and multiplied
by $\lambda$.

The Gibbs identity gives, without exact target calibration,
\[
J_{\lambda,m}(\pi_{\mathrm{stu},\theta})
=\frac{\lambda}{m}\sum_{j=1}^m\log Z_\lambda(\widetilde x_j)
-C^\star_{\lambda,m}(\theta).
\]
Since $\theta^\dagger_{\lambda,m}$ minimizes $C^\star_{\lambda,m}$ and
$\theta^{\mathrm{src}}_{\lambda,m}$ minimizes
$C_{w^\star_{\lambda,\mathrm{src}}}$,
\begin{align*}
&J_{\lambda,m}(\pi_{\mathrm{stu},\theta^\dagger_{\lambda,m}})
-J_{\lambda,m}(\pi_{\mathrm{stu},\theta})\\
&=C^\star_{\lambda,m}(\theta)
-C^\star_{\lambda,m}(\theta^\dagger_{\lambda,m})\\
&\le \widetilde\Delta_{\lambda,m}(\theta)
+2\lambda H\delta_{\mathrm{tar}},
\end{align*}
where the last step uses \eqref{eq:itercal-uniform-target-perturbation}
twice and
$C_{w^\star_{\lambda,\mathrm{src}}}
(\theta^{\mathrm{src}}_{\lambda,m})
\le C_{w^\star_{\lambda,\mathrm{src}}}
(\theta^\dagger_{\lambda,m})$.
Taking expectations and using \eqref{eq:itercal-approx-surrogate} proves
\eqref{eq:itercal-approx-return}.

To prove \eqref{eq:itercal-approx-kl}, let
$f$ be the true cost gap in the preceding display
and $g(\theta):=\cK_{\lambda,m}(\theta)$. Both are nonnegative and analytic
on compact $\Theta$. By the graph argument in the proof of
Lemma~\ref{lem:itercal-lojas}, their restricted graphs are compact
subanalytic. The student-class oracle-law uniqueness condition implies
$\{f=0\}\subseteq\{g=0\}$, so
Lemma~\ref{lem:itercal-compact-lojas} gives fixed
$a^\star_{\lambda,m}>0$ and $p^\star_{\lambda,m}\ge1$ such that
\[
f(\theta)\ge a^\star_{\lambda,m}
[g(\theta)]^{p^\star_{\lambda,m}}.
\]
Jensen's inequality and \eqref{eq:itercal-approx-return}, exactly as in
the proof of Theorem~\ref{thm:itercal-main}, give
\[
\EE[\cK_{\lambda,m}(\theta_T)]
\le\left(\frac{K^{\mathrm{src}}_{\mathrm{opt}}(T+1)^{-1/4}
+2\lambda H\delta_{\mathrm{tar}}}{a^\star_{\lambda,m}}
\right)^{1/p^\star_{\lambda,m}}.
\]
\end{proof}

\paragraph{Rollout and accuracy budgets.}
With the linear gradient and validation batches in
\eqref{eq:br-budgets}, the total number of target answer rollouts is
\[
N_T=\sum_{t=0}^{T-1}\bigl(b_{t+1}+2q_{t+1}\bigr)
=\sum_{t=0}^{T-1}3(t+2)=\frac{3T(T+3)}2.
\]
There are exactly $T$ source reward queries. Hence
Theorem~\ref{thm:itercal-main} gives return-gap and oracle-policy KL bounds
of $O(N_T^{-1/8})$ and $O(N_T^{-1/(8p_{\lambda,m})})$, respectively.
For $0<\epsilon<1$, it suffices to take
\[
T=\left\lceil\max\left\{1,
\left(\frac{C_{\lambda,m}}{\epsilon}\right)^{4p_{\lambda,m}}
\right\}\right\rceil
\]
to guarantee $\EE[\cK_{\lambda,m}(\theta_T)]\le\epsilon$.
This uses $O(\epsilon^{-4p_{\lambda,m}})$ source reward queries and
$O(\epsilon^{-8p_{\lambda,m}})$ target rollouts.
For the expected true regularized return gap to be at most $\epsilon$,
the corresponding sufficient budgets are $O(\epsilon^{-4})$ and
$O(\epsilon^{-8})$.
All orders hold for a fixed model as $\epsilon\downarrow0$ and are
sufficient bounds, not complexity lower bounds.

These rollout counts do not include preprocessing to obtain a valid
calibration information bound or imply comparable runtime guarantees.
The constants retain the dependence on $\gamma$ in
\eqref{eq:itercal-gamma}, the local geometric constants in
Lemma~\ref{lem:itercal-hybrid}, and the exponent $p_{\lambda,m}$ in
Lemma~\ref{lem:itercal-lojas}. In particular, the choice of
$\kappa_{\mathrm{opt}}$ in its proof contains
$(\delta_{\mathrm{loc}}/(2M_{\mathrm{opt}}))^d$.
Thus the result does not claim uniform polynomial complexity in
dimension, horizon, or $1/\lambda$.

\section{Proofs of the Supporting Lemmas}
\label{app:lemma-proofs}
We prove the auxiliary lemmas stated in
Appendix~\ref{app:proofs_thm} for Theorem~\ref{thm:itercal-main}.
\begin{proof}[Proof of Lemma \ref{lem:itercal-gibbs}]
Fix $x$ and denote the right-hand side of
\eqref{eq:itercal-gibbs} by $\bar\pi_\lambda(a_{1:H}\mid x)$.

Since $R\in[0,1]$ and the pre-trained policy $\pi_{pre}$ is a probability law, we have
\[
1\le e^{R(x,a_{1:H})/\lambda}\le e^{1/\lambda}
\ \Longrightarrow\ 
1\le Z_\lambda(x)\le e^{1/\lambda}.
\]
Thus, $\bar\pi_\lambda$ is positive and sums to one. 
By the definition of an autoregressive policy, we have that
\[
\bar\pi_\lambda(a_h\mid x,a_{1:h-1})
=\frac{\bar\pi_\lambda(a_{1:h}\mid x)}
{\bar\pi_\lambda(a_{1:h-1}\mid x)}.
\]
Thus, $\bar{\pi}_\lambda$ is an admissible autoregressive policy.

For any competing policy $\pi$, the autoregressive factorization
in \eqref{eq:itercal-chain} gives
\begin{align}\label{eq:policy_KL}
\sum_{h=1}^{H}\log
\frac{\pi(a_h\mid x,a_{1:h-1})}
{\pi_{\mathrm{pre}}(a_h\mid x,a_{1:h-1})}
=\log\frac{\pi(a_{1:H}\mid x)}
{\pi_{\mathrm{pre}}(a_{1:H}\mid x)}.
\end{align}
Recall that $\bar{\pi}_\lambda(a_{1:H}|x)=\frac{\pi_{\mathrm{pre}}(a_{1:H}\mid x)
e^{R(x,a_{1:H})/\lambda}}{Z_\lambda(x)}$.
Taking logarithms and rearranging gives
\begin{align}\label{eq:Reward_equals_barpi}
R(x,a_{1:H})
=\lambda\log\frac{\bar\pi_\lambda(a_{1:H}\mid x)}
{\pi_{\mathrm{pre}}(a_{1:H}\mid x)}
+\lambda\log Z_\lambda(x).
\end{align}
Therefore, subtracting $\lambda$ times \eqref{eq:policy_KL} from
\eqref{eq:Reward_equals_barpi} and taking expectations gives
\begin{align*}
\EE_{a_{1:H}\sim\pi(\cdot\mid x)}
\left[R(x,a_{1:H})-\lambda\sum_{h=1}^{H}
\log\frac{\pi(a_h\mid x,a_{1:h-1})}
{\pi_{\mathrm{pre}}(a_h\mid x,a_{1:h-1})}\right]&=\lambda\log Z_\lambda(x)
-\lambda\EE_{a_{1:H}\sim\pi(\cdot\mid x)}
\left[\log\frac{\pi(a_{1:H}\mid x)}
{\bar\pi_\lambda(a_{1:H}\mid x)}\right]\\
&=\lambda\log Z_\lambda(x)
-\lambda\KL\!\left(\pi(\cdot\mid x)
\,\middle\|\,\bar\pi_\lambda(\cdot\mid x)\right).
\end{align*}
The second equality follows from the definition of KL.

The KL divergence is nonnegative and $KL(p\|p)=0$,
Thus, the preceding objective is uniquely maximized by
$\pi=\bar\pi_\lambda$, which proves \eqref{eq:itercal-gibbs}.
Average the identity over the target prompts and substitute
$\pi^\star_\lambda=\pi_{w^\star_\lambda}$ from
Assumption~\ref{ass:itercal-truth}. This proves
\eqref{eq:itercal-gibbs-student} and the equivalence of the original
regularized benchmark and the minimum of $C_{w^\star_\lambda}$.
The unknown target rewards appear in the analysis, not in an
algorithmic query.
\end{proof}

Let $\cF_t$ be the history immediately before the source step of
round $t$, containing all source and target random draws from rounds
$0,\ldots,t-1$ and the initial parameters. At $t=0$, there are no
preceding rounds. Thus $w_t$ and $\theta_t$ are $\cF_t$-measurable.
The source draws in round $t$ are fresh conditional on this history.
In particular, the student may depend on every preceding calibration
update, exploration draw, and target evaluation.
\begin{proof}[Proof of Lemma \ref{lem:itercal-label}]
We show that the acceptance rule reweights the two branch indices
by exactly the factors appearing in the regularized optimal policy.

Fix the history $\cF_t$, the selected prefix $s_t=(x_{i_t},a_{t,1:h_t-1})$, and the two legal proposal tokens $c_{t,1}$ and $c_{t,0}$.
Here $\cF_t$ is the history before round $t$.
After additionally conditioning on $Y_t=k$, the first $h_t$
tokens are fixed:
\[
a^{\mathrm{pre}}_{t,1:h_t}
=
(a_{t,1:h_t-1},c_{t,k}).
\]
The algorithm then generates only the remaining suffix at random:
\[
a^{\mathrm{pre}}_{t,h_t+1:H}
\sim
\pi_{\mathrm{pre}}(\cdot\mid s_t,c_{t,k}).
\]
Thus, although the reward function $R$ is deterministic,
\[
R_t
=
R\!\left(
x_{i_t},
(a_{t,1:h_t-1},c_{t,k},
 a^{\mathrm{pre}}_{t,h_t+1:H})
\right)
\]
is random because the reference-generated suffix is random.

For $k\in\{0,1\}$, define the expected exponential reward
\[
M_{t,k}
:=
\EE_{
b_{h_t+1:H}\sim
\pi_{\mathrm{pre}}(\cdot\mid s_t,c_{t,k})
}
\!\left[
\exp\!\left(
\frac{
R\!\left(
x_{i_t},
(a_{t,1:h_t-1},c_{t,k},b_{h_t+1:H})
\right)
}{\lambda}
\right)
\right].
\]
The expectation is only over the reference-generated suffix;
the prompt, prefix, and candidate token are held fixed.
Equivalently,
\[
M_{t,k}
=
\EE\!\left[
e^{R_t/\lambda}
\,\middle|\,
\cF_t,s_t,c_{t,1},c_{t,0},Y_t=k
\right].
\]

To make this expectation explicit, let
$b_{h_t+1:H}=(b_{h_t+1},\ldots,b_H)$ denote a possible
suffix, whose reference probability is
\[
\pi_{\mathrm{pre}}(b_{h_t+1:H}\mid s_t,c_{t,k})
=
\prod_{h=h_t+1}^{H}
\pi_{\mathrm{pre}}\!\left(
b_h
\,\middle|\,
x_{i_t},
(a_{t,1:h_t-1},c_{t,k},b_{h_t+1:h-1})
\right).
\]
By the definition of expectation for a finite distribution, we can compute $M_{t,k}$ as
\begin{align*}
M_{t,k}
&=
\sum_{\substack{b_{h_t+1:H}:\\(a_{t,1:h_t-1},c_{t,k},b_{h_t+1:H})
\in\cA(x_{i_t})}}
\pi_{\mathrm{pre}}(b_{h_t+1:H}\mid s_t,c_{t,k})\times\exp\!\left(
\frac{
R\!\left(
x_{i_t},
(a_{t,1:h_t-1},c_{t,k},b_{h_t+1:H})
\right)
}{\lambda}
\right).
\end{align*}
The feasibility condition ensures that the concatenated sequence
is a legal full answer. If $h_t=H$, there is one empty suffix,
with probability one.

Since $R$ takes values in $[0,1]$, and the suffix probabilities
sum to one, we have that
\begin{equation}
1\le M_{t,k}\le e^{1/\lambda},
\ k\in\{0,1\}.
\label{eq:itercal-suffix-moment}
\end{equation}
We next calculate how acceptance changes the branch distribution. Recall that
\[
I_t
=
\mathbf{1}\!\left\{
U_t\le e^{(R_t-1)/\lambda}
\right\},
\ 
U_t\sim\operatorname{Unif}[0,1],
\]
where $U_t$ is independent uniform distribution random variables. Conditioned on $\cF_t,\allowbreak s_t,\allowbreak c_{t,1},\allowbreak c_{t,0},\allowbreak Y_t,\allowbreak R_t$, the
acceptance threshold is fixed. Moreover, notice that $0\le R_t\le1$ implies $e^{(R_t-1)/\lambda}\in[e^{-1/\lambda},1]$.

By the independence and the uniform density of $U_t$, we obtain
\begin{align*}
\Pr\!\left(
I_t=1
\,\middle|\,
\cF_t,s_t,c_{t,1},c_{t,0},Y_t,R_t
\right)=&
\Pr\!\left(
U_t\le e^{(R_t-1)/\lambda}
\,\middle|\,
\cF_t,s_t,c_{t,1},c_{t,0},Y_t,R_t
\right)\\
=&
\int_0^1
\mathbf{1}\!\left\{u\le e^{(R_t-1)/\lambda}\right\}\,du
=
e^{(R_t-1)/\lambda}.
\end{align*}
This calculation averages over $U_t$ only; it does not yet average
over the reference suffix.

Now conditioned only on
$\cF_t,s_t,c_{t,1},c_{t,0},Y_t=k$, the selected branch is fixed, but $R_t$ still depends on the random
reference suffix. Recall that $I_t$ is an indicator and we apply the law of iterated expectation to get
\begin{align*}
\Pr\!\left(
I_t=1
\,\middle|\,
\cF_t,s_t,c_{t,1},c_{t,0},Y_t=k
\right)
&=
\EE\!\left[
I_t
\,\middle|\,
\cF_t,s_t,c_{t,1},c_{t,0},Y_t=k
\right]
\\
&=
\EE\!\left[
\EE\!\left[
I_t
\,\middle|\,
\cF_t,s_t,c_{t,1},c_{t,0},Y_t,R_t
\right]
\,\middle|\,
\cF_t,s_t,c_{t,1},c_{t,0},Y_t=k
\right]
\\
&=
\EE\!\left[
e^{(R_t-1)/\lambda}
\,\middle|\,
\cF_t,s_t,c_{t,1},c_{t,0},Y_t=k
\right]
.
\end{align*}
The third uses the acceptance probability
computed above. 

Finally, $e^{(R_t-1)/\lambda}=e^{-1/\lambda}e^{R_t/\lambda}$,
and the factor $e^{-1/\lambda}$ is deterministic. Hence, we have
\begin{align*}
&\EE\!\left[
e^{(R_t-1)/\lambda}
\,\middle|\,
\cF_t,s_t,c_{t,1},c_{t,0},Y_t=k
\right]=
e^{-1/\lambda}
\EE\!\left[
e^{R_t/\lambda}
\,\middle|\,
\cF_t,s_t,c_{t,1},c_{t,0},Y_t=k
\right]=
e^{-1/\lambda}M_{t,k}.
\end{align*}
The last equality is the definition of $M_{t,k}$. Thus the
acceptance probability for branch $k$ is the average of its
suffix-specific acceptance probabilities, weighted by the
reference suffix law.

The branch-selection rule in the algorithm is
\[
\Pr\!\left(
Y_t=k
\,\middle|\,
\cF_t,s_t,c_{t,1},c_{t,0}
\right)
=
\frac{\pi_{\mathrm{pre}}(c_{t,k}\mid s_t)}
{\pi_{\mathrm{pre}}(c_{t,1}\mid s_t)
+\pi_{\mathrm{pre}}(c_{t,0}\mid s_t)}.
\]
By the property of conditional probability, we have
\begin{align}
\Pr\!\left(
Y_t=k,I_t=1
\,\middle|\,
\cF_t,s_t,c_{t,1},c_{t,0}
\right)
\nonumber&=\Pr(I_t=1|\cF_t,s_t,c_{t,1},c_{t,0},Y_t=k)\cdot\Pr(Y_t=k|\cF_t,s_t,c_{t,1},c_{t,0})\\
&=
e^{-1/\lambda}
\frac{\pi_{\mathrm{pre}}(c_{t,k}\mid s_t)M_{t,k}}
{\pi_{\mathrm{pre}}(c_{t,1}\mid s_t)
+\pi_{\mathrm{pre}}(c_{t,0}\mid s_t)}.
\label{eq:itercal-joint-label-acceptance}
\end{align}
All denominators are positive because the reference policy is positive.

Summing \eqref{eq:itercal-joint-label-acceptance} over the two
branch indices $Y_t=0,1$ and using $M_{t,k}\ge1$, we obtain
\begin{align}\label{eq:Pr(I=1)_ge_esp}
\Pr\!\left(
I_t=1
\,\middle|\,
\cF_t,s_t,c_{t,1},c_{t,0}
\right)
=&
e^{-1/\lambda}
\frac{
\pi_{\mathrm{pre}}(c_{t,1}\mid s_t)M_{t,1}
+\pi_{\mathrm{pre}}(c_{t,0}\mid s_t)M_{t,0}}
{\pi_{\mathrm{pre}}(c_{t,1}\mid s_t)
+\pi_{\mathrm{pre}}(c_{t,0}\mid s_t)}
\nonumber\\
\ge&
e^{-1/\lambda}
\frac{
\pi_{\mathrm{pre}}(c_{t,1}\mid s_t)
+\pi_{\mathrm{pre}}(c_{t,0}\mid s_t)}
{\pi_{\mathrm{pre}}(c_{t,1}\mid s_t)
+\pi_{\mathrm{pre}}(c_{t,0}\mid s_t)}
=
e^{-1/\lambda}.
\end{align}
This proves \eqref{eq:itercal-acceptance}.
By Bayes' rule, we have that
\begin{align}
\Pr\!\left(
Y_t=1
\,\middle|\,
I_t=1,\cF_t,s_t,c_{t,1},c_{t,0}
\right)&=\frac{\Pr(Y_t=1,I_t=1|\cF_t,s_t,c_{t,1},c_{t,0})}{\Pr(I_t=1|\cF_t,s_t,c_{t,1},c_{t,0})}\nonumber\\
&=
\frac{
\pi_{\mathrm{pre}}(c_{t,1}\mid s_t)M_{t,1}}
{
\pi_{\mathrm{pre}}(c_{t,1}\mid s_t)M_{t,1}
+\pi_{\mathrm{pre}}(c_{t,0}\mid s_t)M_{t,0}}.
\label{eq:itercal-branch-bayes}
\end{align}

It remains to present this distribution in terms of the regularized truth.
By Lemma~\ref{lem:itercal-gibbs}, the optimal policy is
\[
\pi^\star_\lambda(a_{1:H}\mid x)
=
\frac{
\pi_{\mathrm{pre}}(a_{1:H}\mid x)
e^{R(x,a_{1:H})/\lambda}}
{Z_\lambda(x)}.
\]
Its probability of the prefix
$(a_{t,1:h_t-1},c_{t,k})$ is obtained by summing over
all feasible suffixes. The corresponding full-answer events are
disjoint, and their union is precisely the prefix event. Hence,
by marginalization and Lemma \ref{lem:itercal-gibbs}, we have
\begin{align}\label{eq:gibbs}
&\pi^\star_\lambda
(a_{t,1:h_t-1},c_{t,k}\mid x_{i_t})\nonumber\\
=&
\sum_{\substack{b_{h_t+1:H}:\\
(a_{t,1:h_t-1},c_{t,k},b_{h_t+1:H})
\in\cA(x_{i_t})}}
\pi^\star_\lambda\!\left(
(a_{t,1:h_t-1},c_{t,k},b_{h_t+1:H})
\mid x_{i_t}\right)\nonumber
\\
=&
\sum_{\substack{b_{h_t+1:H}:\\
(a_{t,1:h_t-1},c_{t,k},b_{h_t+1:H})
\in\cA(x_{i_t})}}
\frac{
\pi_{\mathrm{pre}}\!\left(
(a_{t,1:h_t-1},c_{t,k},b_{h_t+1:H})
\mid x_{i_t}\right)}
{Z_\lambda(x_{i_t})}\times\exp\!\left(
\frac{
R\!\left(
x_{i_t},
(a_{t,1:h_t-1},c_{t,k},b_{h_t+1:H})
\right)}
{\lambda}
\right).
\end{align}

We now expand this sum.
For every feasible suffix $b_{h_t+1:H}$, the autoregressive
chain rule separates the full-answer probability into the
prefix, the selected token, and the suffix:
\begin{align*}
&\pi_{\mathrm{pre}}\!\left(
(a_{t,1:h_t-1},c_{t,k},b_{h_t+1:H})
\mid x_{i_t}\right)\\
=&
\left[
\prod_{h=1}^{h_t-1}
\pi_{\mathrm{pre}}\!\left(
a_{t,h}\mid x_{i_t},a_{t,1:h-1}
\right)
\right]
\pi_{\mathrm{pre}}\!\left(
c_{t,k}\mid x_{i_t},a_{t,1:h_t-1}
\right)\times\left[
\prod_{h=h_t+1}^{H}
\pi_{\mathrm{pre}}\!\left(
b_h\mid x_{i_t},
(a_{t,1:h_t-1},c_{t,k},b_{h_t+1:h-1})
\right)
\right].
\end{align*}
The first product is the probability of generating the fixed
prefix. The middle factor is the conditional probability of
the selected token. The last product is the conditional
probability of generating the suffix after that token.
More explicitly, using $s_t=(x_{i_t},a_{t,1:h_t-1})$, we have
\begin{align*}
&\prod_{h=1}^{h_t-1}
\pi_{\mathrm{pre}}(a_{t,h}\mid x_{i_t},a_{t,1:h-1})
=\pi_{\mathrm{pre}}(a_{t,1:h_t-1}\mid x_{i_t}),\ \pi_{\mathrm{pre}}(c_{t,k}\mid x_{i_t},a_{t,1:h_t-1})
=\pi_{\mathrm{pre}}(c_{t,k}\mid s_t),
\end{align*}
\begin{align*}
\prod_{h=h_t+1}^{H}
\pi_{\mathrm{pre}}\!\left(
b_h\mid x_{i_t},
(a_{t,1:h_t-1},c_{t,k},b_{h_t+1:h-1})
\right)=
\pi_{\mathrm{pre}}(b_{h_t+1:H}\mid s_t,c_{t,k}).
\end{align*}
Thus, we have the factorization
\begin{align*}
&\pi_{\mathrm{pre}}\!\left(
(a_{t,1:h_t-1},c_{t,k},b_{h_t+1:H})
\mid x_{i_t}\right)=
\pi_{\mathrm{pre}}(a_{t,1:h_t-1}\mid x_{i_t})
\,\pi_{\mathrm{pre}}(c_{t,k}\mid s_t)
\,\pi_{\mathrm{pre}}(b_{h_t+1:H}\mid s_t,c_{t,k}).
\end{align*}
This is a factorization into conditional probabilities, not an
independence assumption on the three parts of the answer.
At $h_t=1$ or $h_t=H$, the corresponding empty product is one.

Substituting this factorization into \eqref{eq:gibbs} gives
\begin{align*}
&\pi^\star_\lambda
(a_{t,1:h_t-1},c_{t,k}\mid x_{i_t})
\\*
=&
\sum_{\substack{b_{h_t+1:H}:\\
(a_{t,1:h_t-1},c_{t,k},b_{h_t+1:H})
\in\cA(x_{i_t})}}
\frac{
\pi_{\mathrm{pre}}(a_{t,1:h_t-1}\mid x_{i_t})
\,\pi_{\mathrm{pre}}(c_{t,k}\mid s_t)}
{Z_\lambda(x_{i_t})}\times
\pi_{\mathrm{pre}}(b_{h_t+1:H}\mid s_t,c_{t,k})\\
&\times
\exp\!\left(
\frac{
R\!\left(
x_{i_t},
(a_{t,1:h_t-1},c_{t,k},b_{h_t+1:H})
\right)}{\lambda}
\right).
\end{align*}
Here the summation variable is only $b_{h_t+1:H}$. The quantities $\pi_{\mathrm{pre}}(a_{t,1:h_t-1}\mid x_{i_t}),
\ 
\pi_{\mathrm{pre}}(c_{t,k}\mid s_t),
\ 
Z_\lambda(x_{i_t})$ do not depend on this suffix. Thus, we factor them out of the sum,
and retain exactly the same feasible suffixes to obtain 
\begin{align*}
&\pi^\star_\lambda
(a_{t,1:h_t-1},c_{t,k}\mid x_{i_t})
\\*
=&
\frac{
\pi_{\mathrm{pre}}(a_{t,1:h_t-1}\mid x_{i_t})
\,\pi_{\mathrm{pre}}(c_{t,k}\mid s_t)}
{Z_\lambda(x_{i_t})}\\
&\times
\Biggl[
\sum_{\substack{b_{h_t+1:H}:\\
(a_{t,1:h_t-1},c_{t,k},b_{h_t+1:H})
\in\cA(x_{i_t})}}
\pi_{\mathrm{pre}}(b_{h_t+1:H}\mid s_t,c_{t,k})
\times
\exp\!\left(
\frac{
R\!\left(
x_{i_t},
(a_{t,1:h_t-1},c_{t,k},b_{h_t+1:H})
\right)}{\lambda}
\right)
\Biggr].
\end{align*}
The expression in square brackets is precisely the defining
suffix sum for $M_{t,k}$. Therefore, we get
\begin{align*}
&\pi^\star_\lambda
(a_{t,1:h_t-1},c_{t,k}\mid x_{i_t})=
\frac{
\pi_{\mathrm{pre}}(a_{t,1:h_t-1}\mid x_{i_t})}
{Z_\lambda(x_{i_t})}
\pi_{\mathrm{pre}}(c_{t,k}\mid s_t)M_{t,k}.
\end{align*}
Recall the conditional token probability $\pi^\star_\lambda(c_{t,k}\mid s_t)
=
\frac{
\pi^\star_\lambda
(a_{t,1:h_t-1},c_{t,k}\mid x_{i_t})}
{\pi^\star_\lambda(a_{t,1:h_t-1}\mid x_{i_t})}$.
The denominator and the preceding prefix factor are common
to both branch indices and are positive. They therefore cancel and we get
\begin{align*}
&\frac{
\pi^\star_\lambda(c_{t,1}\mid s_t)}
{\pi^\star_\lambda(c_{t,1}\mid s_t)
+\pi^\star_\lambda(c_{t,0}\mid s_t)}=
\frac{
\pi_{\mathrm{pre}}(c_{t,1}\mid s_t)M_{t,1}}
{
\pi_{\mathrm{pre}}(c_{t,1}\mid s_t)M_{t,1}
+\pi_{\mathrm{pre}}(c_{t,0}\mid s_t)M_{t,0}}.
\end{align*}
Together with \eqref{eq:itercal-branch-bayes}, this proves the
first equality in \eqref{eq:itercal-accepted-label}.

Finally, Assumption~\ref{ass:itercal-truth} gives
$\pi^\star_\lambda=\pi_{w^\star_\lambda}$.
The common softmax normalizer cancels,
so we have
\begin{align*}
&\frac{
\pi^\star_\lambda(c_{t,1}\mid s_t)}
{\pi^\star_\lambda(c_{t,1}\mid s_t)
+\pi^\star_\lambda(c_{t,0}\mid s_t)}=
\frac{
e^{(w^\star_\lambda)^\top\phi(s_t,c_{t,1})}}
{
e^{(w^\star_\lambda)^\top\phi(s_t,c_{t,1})}
+e^{(w^\star_\lambda)^\top\phi(s_t,c_{t,0})}}=
\frac{1}{
1+e^{(
-(w^\star_\lambda)^\top
[\phi(s_t,c_{t,1})-\phi(s_t,c_{t,0})]
)}}
=\sigma(z_t^\top w^\star_\lambda).
\end{align*}
This establishes the second equality in \eqref{eq:itercal-accepted-label}.

If $c_{t,1}=c_{t,0}$, the two branch indices have identical
reference-completion laws, hence $M_{t,1}=M_{t,0}$.
Their prior and accepted probabilities are both $1/2$,
consistent with $z_t=0$ and $\sigma(0)=1/2$.
After $\mathtt{EOS}$, both candidates are $\mathtt{null}$,
so the same argument applies.
\end{proof}
\begin{proof}[Proof of Lemma \ref{lem:itercal-calibration}]
We first recall the random definition vector $z_t$ from
Algorithm~\ref{alg:itercal}. The history $\cF_t$ contains all
randomness before round $t$, so $\theta_t$ is fixed conditional
on $\cF_t$. After choosing the source index $i_t$ and branching
time $h_t$, the algorithm draws
$a_{t,1:H}\sim\pi_{\mathrm{tea}}(\cdot\mid x_{i_t})$ and sets $s_t=(x_{i_t},a_{t,1:h_t-1})$, $c_{t,1}=a_{t,h_t}$ is the teacher token, and $c_{t,0}\sim\pi_{\mathrm{stu},\theta_t}(\cdot\mid s_t)$ is a fresh student token at the same state.
The vector in the lemma is
\[
z_t:=\phi(s_t,c_{t,1})-\phi(s_t,c_{t,0})\in\RR^D.
\]
First, we fix the history $\cF_t$ and a legal nonterminal state $s$,
so that $\cB(s)$ is nonempty. Recall that $B>0$ is the scalar
bound on the parameter norm, whereas $\cB(s)$ is the set of
legal next tokens at $s$. Both $b$ and $c$ below denote individual
tokens, not trajectories.

The current student policy is then explicitly
\[
\pi_{\mathrm{stu},\theta_t}(b\mid s)
=
\begin{cases}
\displaystyle
\frac{\exp\!\left(\theta_t^\top\phi_{\mathrm{stu}}(s,b)\right)}
{\displaystyle\sum_{c\in\cB(s)}
 \exp\!\left(\theta_t^\top\phi_{\mathrm{stu}}(s,c)\right)},
& b\in\cB(s),\\[12pt]
0,& b\in\cA\setminus\cB(s).
\end{cases}
\]
Here $\theta_t\in\RR^d$ is the current student parameter, and
$\phi_{\mathrm{stu}}(s,c)\in\RR^d$ is the student feature vector
of the state-token pair $(s,c)$.
Since $\theta_t$ is $\cF_t$-measurable, it is fixed after we
condition on the history. The model bounds are
\[
\|\theta_t\|_2\le B,
\ 
\|\phi_{\mathrm{stu}}(s,c)\|_2\le1,
\ \forall c\in\cB(s).
\]
Consequently, by the Cauchy-Schwarz inequality, for each legal token $c$, we have
\[
\left|\theta_t^\top\phi_{\mathrm{stu}}(s,c)\right|
\le
\|\theta_t\|_2\,\|\phi_{\mathrm{stu}}(s,c)\|_2
\le B.
\]
Equivalently, $-B\le\theta_t^\top\phi_{\mathrm{stu}}(s,c)\le B$.
Because $u\mapsto e^u$ is increasing, we have
\[
e^{-B}
\le
\exp\!\left(\theta_t^\top\phi_{\mathrm{stu}}(s,c)\right)
\le e^B,
\ c\in\cB(s).
\]

Now fix any $b\in\cB(s)$. Applying these bounds to the full
softmax expression yields
\begin{equation}
\begin{aligned}
\pi_{\mathrm{stu},\theta_t}(b\mid s)
&=
\frac{
\exp\!\left(\theta_t^\top\phi_{\mathrm{stu}}(s,b)\right)}
{\displaystyle\sum_{c\in\cB(s)}
\exp\!\left(\theta_t^\top\phi_{\mathrm{stu}}(s,c)\right)}
\ge
\frac{e^{-B}}
{\displaystyle\sum_{c\in\cB(s)}
\exp\!\left(\theta_t^\top\phi_{\mathrm{stu}}(s,c)\right)}\\
&\ge
\frac{e^{-B}}
{\displaystyle\sum_{c\in\cB(s)}e^B}=
\frac{e^{-B}}{|\cB(s)|e^B}
=
\frac{e^{-2B}}{|\cB(s)|}.
\end{aligned}
\label{eq:itercal-student-token-coverage}
\end{equation}
The first equality is the definition of the student policy.
The first inequality replaces its numerator by the lower bound
$e^{-B}$, leaving the positive denominator unchanged.
The second inequality increases each denominator term to its
upper bound $e^B$; increasing a positive denominator decreases
the fraction.

The lower bound applies only to legal tokens; illegal tokens
have probability zero. At an absorbing prefix before the
terminal horizon, $\cB(s)=\{\mathtt{null}\}$, and the softmax
formula gives
\[
\pi_{\mathrm{stu},\theta_t}(\mathtt{null}\mid s)
=
\frac{
\exp\!\left(\theta_t^\top
\phi_{\mathrm{stu}}(s,\mathtt{null})\right)}
{
\exp\!\left(\theta_t^\top
\phi_{\mathrm{stu}}(s,\mathtt{null})\right)}
=1
\ge e^{-2B}.
\]
Thus \eqref{eq:itercal-student-token-coverage} also holds at these
prefixes.

To prove the matrix inequality, fix a deterministic vector
$v\in\RR^D$ and notice
$v^\top z_tz_t^\top v
=(v^\top z_t)^2$.
Conditional expectation is linear,
and $v$ is fixed, so we have
\[
\begin{aligned}
v^\top\EE\!\left[z_tz_t^\top\mid\cF_t\right]v=\EE\!\left[v^\top z_tz_t^\top v\mid\cF_t\right]=\EE\!\left[(v^\top z_t)^2\mid\cF_t\right].
\end{aligned}
\]
We now expand this expectation in the order of sampling. The
source index $i_t$ and branching time $h_t$ are fresh, independent, uniform
draws. Hence, we have $\Pr(i_t=i,h_t=h\mid\cF_t)=\frac1{nH}$.

The conditional law of total expectation gives us
\begin{align*}
\EE\!\left[(v^\top z_t)^2\mid\cF_t\right]&=\sum_{i=1}^{n}\sum_{h=1}^{H}
\Pr(i_t=i,h_t=h\mid\cF_t)\,
\EE\!\left[(v^\top z_t)^2
\mid\cF_t,i_t=i,h_t=h\right]\\
&=\frac1{nH}\sum_{i=1}^{n}\sum_{h=1}^{H}
\EE\!\left[(v^\top z_t)^2
\mid\cF_t,i_t=i,h_t=h\right].
\end{align*}

For fixed $i,h$, let $a_{i,1:h}$ denote a possible realization
of the first $h$ teacher tokens, and define
$s_{i,h}:=(x_i,a_{i,1:h-1})$. Conditioning on $i_t=i,\ h_t=h,\ a_{t,1:h}=a_{i,1:h}$,
the definitions of $s_t$ and $c_{t,1}$ yield
$s_t=s_{i,h}$ and $c_{t,1}=a_{i,h}$. The only remaining
randomness in $z_t$ is the fresh student token. For each
$b\in\cB(s_{i,h})$, its conditional probability is
\[
\Pr\!\left(c_{t,0}=b
\mid\cF_t,i_t=i,h_t=h,a_{t,1:h}=a_{i,1:h}\right)
=\pi_{\mathrm{stu},\theta_t}(b\mid s_{i,h}).
\]
If $c_{t,0}=b$, substitution into the recalled definition gives
\[
z_t=\phi(s_{i,h},a_{i,h})-\phi(s_{i,h},b).
\]
Thus, we have
\begin{align*}
&\EE\!\left[(v^\top z_t)^2
\mid\cF_t,i_t=i,h_t=h,a_{t,1:h}=a_{i,1:h}\right]\\
&=\sum_{b\in\cB(s_{i,h})}
\pi_{\mathrm{stu},\theta_t}(b\mid s_{i,h})
\left(v^\top[\phi(s_{i,h},a_{i,h})-\phi(s_{i,h},b)]\right)^2.
\end{align*}
The frozen teacher's prefix law is  $\pi_{\mathrm{tea}}(a_{i,1:h}\mid x_i)
=\prod_{k=1}^{h}
\pi_{\mathrm{tea}}(a_{i,k}\mid x_i,a_{i,1:k-1})$.
The remaining teacher suffix $a_{t,h+1:H}$ is absent from $z_t$;
summing its conditional probabilities gives one. 

For fixed $i_t=i$ and $h_t=h$, the preceding calculation
conditions on the teacher tokens $a_{t,1:h}=a_{i,1:h}$
and averages over the student token $c_{t,0}$.
We now average over all possible teacher token sequences
$a_{i,1:h}$, each weighted by its probability
$\pi_{\mathrm{tea}}(a_{i,1:h}\mid x_i)$.
By the law of total expectation, this removes the conditioning
on $a_{t,1:h}$ and gives
$\EE[(v^\top z_t)^2\mid\cF_t,i_t=i,h_t=h]$.
\begin{align*}
&v^\top\EE\!\left[z_tz_t^\top\mid\cF_t\right]v\\
&=\frac1{nH}\sum_{i=1}^{n}\sum_{h=1}^{H}
\EE_{a_{i,1:h}\sim\pi_{\mathrm{tea}}(\cdot\mid x_i)}
\Bigg[
\sum_{b\in\cB(s_{i,h})}
\pi_{\mathrm{stu},\theta_t}(b\mid s_{i,h})
\left(v^\top[\phi(s_{i,h},a_{i,h})-\phi(s_{i,h},b)]\right)^2
\,\Bigm|\,\cF_t
\Bigg]\\
&\ge\frac{e^{-2B}}{nH}\sum_{i=1}^{n}\sum_{h=1}^{H}
\EE_{a_{i,1:h}\sim\pi_{\mathrm{tea}}(\cdot\mid x_i)}
\Bigg[
\frac1{|\cB(s_{i,h})|}
\sum_{b\in\cB(s_{i,h})}
\left(v^\top[\phi(s_{i,h},a_{i,h})-\phi(s_{i,h},b)]\right)^2
\,\Bigm|\,\cF_t
\Bigg]\\
&=e^{-2B}v^\top\left(\frac1H\sum_{h=1}^{H}G_h\right)v
=e^{-2B}v^\top G_{\mathrm{joint}}v.
\end{align*}
The first equality is $v^\top G_{\mathrm{src}}(\theta_t)v$ by
\eqref{eq:itercal-student-information}. Since it holds for every $v$,
\[
\EE[z_tz_t^\top\mid\cF_t]=G_{\mathrm{src}}(\theta_t).
\]
The inequality holds by \eqref{eq:itercal-student-token-coverage}, and the second equality is by definition
\eqref{eq:itercal-gram}. The final equality uses
$G_{\mathrm{joint}}=H^{-1}\sum_{h=1}^{H}G_h$.

Here $a_{i,1:h}$ is an integration variable, not an additional
rollout: each round still draws only one source index and one
teacher answer. Since the calculation holds for every $v$,
it proves the matrix inequality in \eqref{eq:itercal-adaptive-design}.

The feature norm bound also gives
\[
\|z_t\|_2
=\|\phi(s_t,c_{t,1})-\phi(s_t,c_{t,0})\|_2
\le\|\phi(s_t,c_{t,1})\|_2+\|\phi(s_t,c_{t,0})\|_2\le2.
\]

Thus, we have proved \eqref{eq:itercal-adaptive-design}, and we will prove the drift bound \eqref{eq:itercal-source-drift} next. 

Starting from the definition of
the actual stochastic update in Algorithm~\ref{alg:itercal}, we have
\[
g_t:=I_tz_t\bigl[\sigma(z_t^\top w_t)-Y_t\bigr].
\]
Recall that $I_t\in\{0,1\}$ is the acceptance indicator,
$Y_t\in\{0,1\}$ is the selected branch index, and
$\sigma(u)=(1+e^{-u})^{-1}$.

Conditioned on
$\cF_t,s_t,c_{t,1},c_{t,0}$, the parameter $w_t$ is
$\cF_t$-measurable, $w^\star_\lambda$ is fixed, and
$z_t=\phi(s_t,c_{t,1})-\phi(s_t,c_{t,0})$ is also determined. Thus, $w_t$, $z_t$, and
$\sigma(z_t^\top w_t)$ are measurable under this conditioning and $I_t$ and $Y_t$ are still randomized.

First, by definition of $g_t$, we have
\[
\begin{aligned}
(w_t-w^\star_\lambda)^\top g_t&=(w_t-w^\star_\lambda)^\top
\left\{I_tz_t[\sigma(z_t^\top w_t)-Y_t]\right\}\\
&=I_t\bigl[(w_t-w^\star_\lambda)^\top z_t\bigr]
[\sigma(z_t^\top w_t)-Y_t]=z_t^\top(w_t-w^\star_\lambda)\,
I_t[\sigma(z_t^\top w_t)-Y_t].
\end{aligned}
\]
Taking conditional expectations
and pulling out the measurable scalars, we have
\begin{align*}
&\EE\!\left[(w_t-w^\star_\lambda)^\top g_t
\mid\cF_t,s_t,c_{t,1},c_{t,0}\right]=z_t^\top(w_t-w^\star_\lambda)\,
\EE\!\left[
I_t[\sigma(z_t^\top w_t)-Y_t]
\mid\cF_t,s_t,c_{t,1},c_{t,0}\right].
\end{align*}

We evaluate the remaining expectation by splitting the two
possible values of $I_t$. When $I_t=0$ the product is zero;
when $I_t=1$ it is $\sigma(z_t^\top w_t)-Y_t$.
By the conditional law of total expectation,
\begin{align*}
&\EE\!\left[
I_t[\sigma(z_t^\top w_t)-Y_t]
\mid\cF_t,s_t,c_{t,1},c_{t,0}\right]\\
=&\Pr\!\left(I_t=0\mid\cF_t,s_t,c_{t,1},c_{t,0}\right)\cdot0+\Pr\!\left(I_t=1\mid\cF_t,s_t,c_{t,1},c_{t,0}\right)\times\EE\!\left[
\sigma(z_t^\top w_t)-Y_t
\mid I_t=1,\cF_t,s_t,c_{t,1},c_{t,0}\right]\\
=&\Pr\!\left(I_t=1\mid\cF_t,s_t,c_{t,1},c_{t,0}\right)\times
\left[
\sigma(z_t^\top w_t)-
\EE\!\left[Y_t
\mid I_t=1,\cF_t,s_t,c_{t,1},c_{t,0}\right]
\right].
\end{align*}

Because $Y_t$ is binary, its conditional expectation is the
conditional probability that it equals one. By Lemma~\ref{lem:itercal-label}, we have that
\begin{align*}
&\EE\!\left[Y_t
\mid I_t=1,\cF_t,s_t,c_{t,1},c_{t,0}\right]\\
=&0\cdot\Pr\!\left(Y_t=0
\mid I_t=1,\cF_t,s_t,c_{t,1},c_{t,0}\right)+1\cdot\Pr\!\left(Y_t=1
\mid I_t=1,\cF_t,s_t,c_{t,1},c_{t,0}\right)\\
=&\Pr\!\left(Y_t=1
\mid I_t=1,\cF_t,s_t,c_{t,1},c_{t,0}\right)=\sigma(z_t^\top w^\star_\lambda).
\end{align*}

Plugging this back, we obtain
\begin{align}\label{eq:(w-w)g_1}
&\EE\!\left[(w_t-w^\star_\lambda)^\top g_t
\mid\cF_t,s_t,c_{t,1},c_{t,0}\right]=z_t^\top(w_t-w^\star_\lambda)\,
\Pr\!\left(I_t=1\mid\cF_t,s_t,c_{t,1},c_{t,0}\right)\times
\left[\sigma(z_t^\top w_t)-\sigma(z_t^\top w^\star_\lambda)\right].
\end{align}
Finally, we pull out the fixed vector $z_t$ instead
of the scalar inner product, which gives the corresponding vector
conditional mean:
\begin{align}
\EE\!\left[g_t\mid\cF_t,s_t,c_{t,1},c_{t,0}\right]&=z_t\,
\EE\!\left[
I_t[\sigma(z_t^\top w_t)-Y_t]
\mid\cF_t,s_t,c_{t,1},c_{t,0}\right]\nonumber\\
&=\Pr\!\left(I_t=1\mid\cF_t,s_t,c_{t,1},c_{t,0}\right)z_t\times
\left[\sigma(z_t^\top w_t)-\sigma(z_t^\top w^\star_\lambda)\right].
\label{eq:itercal-source-conditional-mean}
\end{align}

By direct differentiation, we have $\sigma'(u)=\frac{e^{-u}}{(1+e^{-u})^2}
=\frac1{2+e^u+e^{-u}}>0$.
The denominator is even and nondecreasing in $|u|$, because its
derivative on $[0,\infty)$ is $e^u-e^{-u}\ge0$. Hence
$\sigma'(u)\ge\sigma'(2B)$ whenever $|u|\le2B$.
For $r\in[0,1]$, convexity of $W$ places
$w^\star_\lambda+r(w_t-w^\star_\lambda)$ in $W$, and therefore we have
\[
\left|z_t^\top
[w^\star_\lambda+r(w_t-w^\star_\lambda)]\right|
\le\|z_t\|_2\,
\|w^\star_\lambda+r(w_t-w^\star_\lambda)\|_2
\le2B.
\]
Now, we apply
the fundamental theorem of calculus along this segment to get
\begin{align*}
\sigma(z_t^\top w_t)-\sigma(z_t^\top w^\star_\lambda)&=\int_0^1\frac{d}{dr}
\sigma\!\left(z_t^\top[w^\star_\lambda+r(w_t-w^\star_\lambda)]\right)\,dr\\
&=\int_0^1
\sigma'\!\left(z_t^\top[w^\star_\lambda+r(w_t-w^\star_\lambda)]\right)
z_t^\top(w_t-w^\star_\lambda)\,dr\\
&=z_t^\top(w_t-w^\star_\lambda)
\int_0^1
\sigma'\!\left(z_t^\top[w^\star_\lambda+r(w_t-w^\star_\lambda)]\right)\,dr.
\end{align*}
Multiplying the equality by $z_t^\top(w_t-w^\star_\lambda)$ gives
\begin{align}\label{eq:(w-w)g_2}
z_t^\top(w_t-w^\star_\lambda)
\left[\sigma(z_t^\top w_t)-\sigma(z_t^\top w^\star_\lambda)\right]\nonumber&=\left[z_t^\top(w_t-w^\star_\lambda)\right]^2
\int_0^1
\sigma'\!\left(z_t^\top[w^\star_\lambda+r(w_t-w^\star_\lambda)]\right)\,dr\nonumber\\
&\ge\left[z_t^\top(w_t-w^\star_\lambda)\right]^2
\int_0^1\sigma'(2B)\,dr\nonumber\\
&=\sigma'(2B)\left[z_t^\top(w_t-w^\star_\lambda)\right]^2
\ge0.
\end{align}
Recall that $I_t=\mathbf{1}\{U_t\le e^{(R_t-1)/\lambda}\}$ and that Lemma~\ref{lem:itercal-label},
\eqref{eq:itercal-acceptance} already establishes
\[
\Pr\!\left(I_t=1\mid\cF_t,s_t,c_{t,1},c_{t,0}\right)
\ge e^{-1/\lambda}.
\]
Combining all these together, we obtain
\begin{align*}
\EE\!\left[(w_t-w^\star_\lambda)^\top g_t
\mid\cF_t,s_t,c_{t,1},c_{t,0}\right]
&=\Pr\!\left(I_t=1\mid\cF_t,s_t,c_{t,1},c_{t,0}\right)
z_t^\top(w_t-w^\star_\lambda)
\left[\sigma(z_t^\top w_t)-\sigma(z_t^\top w^\star_\lambda)\right]\\
&\ge\Pr\!\left(I_t=1\mid\cF_t,s_t,c_{t,1},c_{t,0}\right)
\sigma'(2B)\left[z_t^\top(w_t-w^\star_\lambda)\right]^2\\
&\ge e^{-1/\lambda}\sigma'(2B)
\left[z_t^\top(w_t-w^\star_\lambda)\right]^2.
\end{align*}
The equality was derived by \eqref{eq:(w-w)g_1}. The first inequality uses \eqref{eq:(w-w)g_2} and the second is by \eqref{eq:Pr(I=1)_ge_esp}.

Now, we apply the tower property once more to obtain
\begin{align*}
\EE\!\left[(w_t-w^\star_\lambda)^\top g_t\mid\cF_t\right]&=\EE\!\left[
\EE\!\left[(w_t-w^\star_\lambda)^\top g_t
\mid\cF_t,s_t,c_{t,1},c_{t,0}\right]
\,\middle|\,\cF_t
\right]\\
&\ge\EE\!\left[
e^{-1/\lambda}\sigma'(2B)
\left[z_t^\top(w_t-w^\star_\lambda)\right]^2
\,\middle|\,\cF_t
\right]\\
&=e^{-1/\lambda}\sigma'(2B)
\EE\!\left[
\left[z_t^\top(w_t-w^\star_\lambda)\right]^2
\,\middle|\,\cF_t
\right].
\end{align*}
Notice that $\left[z_t^\top(w_t-w^\star_\lambda)\right]^2
=(w_t-w^\star_\lambda)^\top z_tz_t^\top(w_t-w^\star_\lambda)$. Because $w_t-w^\star_\lambda$ is $\cF_t$-measurable, by the linearity
of conditional expectation, we have
\[
\begin{aligned}
&\EE\!\left[
\left[z_t^\top(w_t-w^\star_\lambda)\right]^2
\,\middle|\,\cF_t\right]=(w_t-w^\star_\lambda)^\top
\EE\!\left[z_tz_t^\top\mid\cF_t\right]
(w_t-w^\star_\lambda).
\end{aligned}
\]
Using the matrix lower bound \eqref{eq:itercal-adaptive-design},
followed by
$G_{\mathrm{joint}}\succeq\mu_{\mathrm{joint}}I_D$,  we have
\begin{align*}
\EE\!\left[
(w_t-w^\star_\lambda)^\top g_t
\,\middle|\,\cF_t\right]
&\ge e^{-1/\lambda}\sigma'(2B)
(w_t-w^\star_\lambda)^\top
\EE\!\left[z_tz_t^\top\mid\cF_t\right]
(w_t-w^\star_\lambda)\\
&=e^{-1/\lambda}\sigma'(2B)
(w_t-w^\star_\lambda)^\top
G_{\mathrm{src}}(\theta_t)(w_t-w^\star_\lambda)\\
&\ge e^{-1/\lambda}e^{-2B}\sigma'(2B)
(w_t-w^\star_\lambda)^\top
G_{\mathrm{joint}}(w_t-w^\star_\lambda)\\
&\ge e^{-1/\lambda}e^{-2B}\sigma'(2B)\mu_{\mathrm{joint}}
\|w_t-w^\star_\lambda\|_2^2\\
&=\gamma\|w_t-w^\star_\lambda\|_2^2.
\end{align*}
Here $I_D$ is the $D$-dimensional identity matrix.
The final equality uses the definition of $\gamma$ in \eqref{eq:itercal-gamma}.

By the Cauchy-Schwarz inequality, we have $\|g_t\|_2
=I_t\|z_t\|_2\,
|\sigma(z_t^\top w_t)-Y_t|
\le2$. We prove
\eqref{eq:itercal-source-drift}.

For any $u\in\RR^D$, let $p=\operatorname{Proj}_W(u)$. Since $W$ is compact and convex, $p$ exists and is unique. For any $r\in W$,
the segment $p+v(r-p)$, $0\le v\le1$, lies in $W$. The right
derivative at zero of its squared distance to $u$ is nonnegative:
\[
0\le\left.\frac{d}{dv}
\|p+v(r-p)-u\|_2^2\right|_{v=0^+}
=2\langle p-u,r-p\rangle.
\]
Thus $\langle u-p,r-p\rangle\le0$. Taking
$r=w^\star_\lambda$ and expanding a square gives
\begin{align*}
\|u-w^\star_\lambda\|_2^2
&=\|u-p\|_2^2+\|p-w^\star_\lambda\|_2^2
+2\langle u-p,p-w^\star_\lambda\rangle\ge\|p-w^\star_\lambda\|_2^2,
\end{align*}
because the cross term and the first squared norm are nonnegative.
We use this inequality with
$u=w_t-g_t/[\gamma(t+2)]$ and $p=w_{t+1}$ to get
\begin{align*}
\|w_{t+1}-w^\star_\lambda\|_2^2
&\le\left\|w_t-w^\star_\lambda
-\frac{g_t}{\gamma(t+2)}\right\|_2^2=\|w_t-w^\star_\lambda\|_2^2
-\frac{2(w_t-w^\star_\lambda)^\top g_t}{\gamma(t+2)}
+\frac{\|g_t\|_2^2}{\gamma^2(t+2)^2}.
\end{align*}
Conditioning on $\cF_t$, we apply \eqref{eq:itercal-adaptive-design} and \eqref{eq:itercal-source-drift} to get 
\begin{align}
&\EE\!\left[\|w_{t+1}-w^\star_\lambda\|_2^2
\,\middle|\,\cF_t\right]\le\left(1-\frac2{t+2}\right)
\|w_t-w^\star_\lambda\|_2^2
+\frac4{\gamma^2(t+2)^2}=\frac{t}{t+2}\|w_t-w^\star_\lambda\|_2^2
+\frac4{\gamma^2(t+2)^2}.
\label{eq:itercal-source-recursion}
\end{align}

We now prove \eqref{eq:itercal-source-mse} by induction.

At $t=0$ the first coefficient is zero. Taking expectations gives
\[
\EE\!\left[\|w_1-w^\star_\lambda\|_2^2\right]
\le\frac1{\gamma^2}\le\frac4{3\gamma^2}.
\]
This is the base case of \eqref{eq:itercal-source-mse}. Suppose its
assertion holds at some $t\ge1$. By
\eqref{eq:itercal-source-recursion}, we have
\begin{align*}
\EE\!\left[\|w_{t+1}-w^\star_\lambda\|_2^2\right]
&\le\frac{t}{t+2}\frac4{\gamma^2(t+2)}
+\frac4{\gamma^2(t+2)^2}=\frac{4(t+1)}{\gamma^2(t+2)^2}
\le\frac4{\gamma^2(t+3)}.
\end{align*}
The final inequality follows from
$(t+1)(t+3)=(t+2)^2-1\le(t+2)^2$. This proves \eqref{eq:itercal-source-mse}.

For the increment bound, apply the same projection variational
inequality with $r=w_t\in W$ and the same $u,p$ as above.
It implies
\begin{align*}
\|w_{t+1}-w_t\|_2^2
&\le\langle u-w_t,w_{t+1}-w_t\rangle\le\|u-w_t\|_2\,\|w_{t+1}-w_t\|_2=\frac{\|g_t\|_2}{\gamma(t+2)}
\|w_{t+1}-w_t\|_2.
\end{align*}
Thus, we obtain that $\|w_{t+1}-w_t\|_2\le\frac2{\gamma(t+2)}
\ \text{for every }t\ge0$.
\end{proof}
\begin{proof}[Proof of Lemma \ref{lem:itercal-z}]
Recall the common feasible sets and the two softmax policies. The set $\cB(s)$ contains the legal next tokens at a
nonterminal prefix $s$, and $\cA(x)$ contains the feasible complete
answers of length $H$ for prompt $x$. These sets are finite and do
not depend on either parameter. Every nonterminal reachable state
has at least one legal token. For $a\in\cB(s)$,
\[
\begin{aligned}
\pi_w(a\mid s)
=\frac{\exp(w^\top\phi(s,a))}
{\sum_{b\in\cB(s)}\exp(w^\top\phi(s,b))},\ \pi_{\mathrm{stu},\theta}(a\mid s)
=\frac{\exp(\theta^\top\phi_{\mathrm{stu}}(s,a))}
{\sum_{b\in\cB(s)}
\exp(\theta^\top\phi_{\mathrm{stu}}(s,b))}.
\end{aligned}
\]
The feature maps are fixed and satisfy
\[
\|\phi(s,a)\|_2\le1,\ 
\|\phi_{\mathrm{stu}}(s,a)\|_2\le1,\ \|w\|_2\le B,\ \|\theta\|_2\le B.
\]
For a fixed feasible answer, $s_h=(x,a_{1:h-1})$, the 
autoregressive definition gives
\[
\begin{aligned}
\pi_w(a_{1:H}\mid x)=\prod_{h=1}^{H}\pi_w(a_h\mid s_h),\ 
\pi_{\mathrm{stu},\theta}(a_{1:H}\mid x)=\prod_{h=1}^{H}\pi_{\mathrm{stu},\theta}(a_h\mid s_h).
\end{aligned}
\]
Both probabilities are strictly positive on $\cA(x)$.

Starting from the definition of $Z_t$ and the autoregressive product,
\begin{align*}
Z_t(\vartheta;x,a_{1:H})
=&\lambda\log
\frac{\pi_{\mathrm{stu},\vartheta}(a_{1:H}\mid x)}
{\pi_{w_t}(a_{1:H}\mid x)}=\lambda\log
\frac{\prod_{h=1}^{H}\pi_{\mathrm{stu},\vartheta}(a_h\mid s_h)}
{\prod_{h=1}^{H}\pi_{w_t}(a_h\mid s_h)}\\
=&\lambda\sum_{h=1}^{H}
\left[\log\pi_{\mathrm{stu},\vartheta}(a_h\mid s_h)
-\log\pi_{w_t}(a_h\mid s_h)\right]\\
=&\lambda\sum_{h=1}^{H}\Bigg[
\vartheta^\top\phi_{\mathrm{stu}}(s_h,a_h)
-w_t^\top\phi(s_h,a_h)
-\log\!\left(\sum_{b\in\cB(s_h)}
e^{\vartheta^\top\phi_{\mathrm{stu}}(s_h,b)}\right)
+\log\!\left(\sum_{b\in\cB(s_h)}e^{w_t^\top\phi(s_h,b)}\right)
\Bigg].
\end{align*}
We now bound the token log ratios in this exact expression.
The calculation is uniform in the parameters, so we carry it out
for arbitrary $w\in W$ and $\theta\in\Theta$ before substituting
$w=w_t$ and $\theta=\vartheta$.
By the Cauchy-Schwarz inequality, for every legal token $b$, we have that
\[
\begin{aligned}
|w^\top\phi(s,b)|
&\le\|w\|_2\|\phi(s,b)\|_2\le B,\ |\theta^\top\phi_{\mathrm{stu}}(s,b)|\le\|\theta\|_2\|\phi_{\mathrm{stu}}(s,b)\|_2\le B.
\end{aligned}
\]
Exponentiation preserves these inequalities, so we get
\[
\begin{aligned}
e^{-B}\le\exp(w^\top\phi(s,b))\le e^B,\ e^{-B}\le\exp(\theta^\top\phi_{\mathrm{stu}}(s,b))\le e^B.
\end{aligned}
\]
Summing the lower and upper bounds over all legal tokens gives
\[
\begin{aligned}
|\cB(s)|e^{-B}
&\le\sum_{b\in\cB(s)}\exp(w^\top\phi(s,b))
\le|\cB(s)|e^B,\ 
|\cB(s)|e^{-B}\le\sum_{b\in\cB(s)}
\exp(\theta^\top\phi_{\mathrm{stu}}(s,b))
\le|\cB(s)|e^B.
\end{aligned}
\]
The numerators and denominators are positive. Therefore, substituting the preceding bounds yields
\begin{align}
\pi_w(a\mid s)
&\ge\frac{e^{-B}}{|\cB(s)|e^B}
=\frac{e^{-2B}}{|\cB(s)|},\ 
\pi_w(a\mid s)
\le\frac{e^B}{|\cB(s)|e^{-B}}
=\frac{e^{2B}}{|\cB(s)|},\\
\pi_{\mathrm{stu},\theta}(a\mid s)
&\ge\frac{e^{-B}}{|\cB(s)|e^B}
=\frac{e^{-2B}}{|\cB(s)|},\ 
\pi_{\mathrm{stu},\theta}(a\mid s)
\le\frac{e^B}{|\cB(s)|e^{-B}}
=\frac{e^{2B}}{|\cB(s)|}.
\label{eq:itercal-student-envelope}
\end{align}
Hence, for $\pi_{stu,\theta}$ and $\pi_w$, we have
\[
\begin{aligned}
\frac{\pi_{\mathrm{stu},\theta}(a\mid s)}{\pi_w(a\mid s)}
&\ge
\frac{e^{-2B}/|\cB(s)|}{e^{2B}/|\cB(s)|}
=e^{-4B},\ 
\frac{\pi_{\mathrm{stu},\theta}(a\mid s)}{\pi_w(a\mid s)}\le
\frac{e^{2B}/|\cB(s)|}{e^{-2B}/|\cB(s)|}
=e^{4B}.
\end{aligned}
\]
We thus obtain $-4B\le
\log\frac{\pi_{\mathrm{stu},\theta}(a\mid s)}{\pi_w(a\mid s)}
\le4B$. Finally, we have that

\begin{align*}
\left|\log
\frac{\pi_{\mathrm{stu},\theta}(a_{1:H}\mid x)}
{\pi_w(a_{1:H}\mid x)}\right|
&=\left|\log\prod_{h=1}^{H}
\frac{\pi_{\mathrm{stu},\theta}(a_h\mid s_h)}
{\pi_w(a_h\mid s_h)}\right|=\left|\sum_{h=1}^{H}
\log\frac{\pi_{\mathrm{stu},\theta}(a_h\mid s_h)}
{\pi_w(a_h\mid s_h)}\right|\le\sum_{h=1}^{H}
\left|\log\frac{\pi_{\mathrm{stu},\theta}(a_h\mid s_h)}
{\pi_w(a_h\mid s_h)}\right|\le4BH.
\end{align*}

In particular, setting $w=w_t$ and $\theta=\vartheta$ gives
\[
|Z_t(\vartheta;x,a_{1:H})|
=\lambda\left|\log
\frac{\pi_{\mathrm{stu},\vartheta}(a_{1:H}\mid x)}
{\pi_{w_t}(a_{1:H}\mid x)}\right|
\le4\lambda BH.
\]
This proves \eqref{eq:itercal-z-bound}.
\end{proof}

\begin{proof}[Proof of Lemma \ref{lem:itercal-lipschitz}]
Starting from the definition of $C_w(\theta)$, applying 
Lemma~\ref{lem:itercal-z}, we have
\[
\begin{aligned}
C_w(\theta)
&=\frac{\lambda}{m}\sum_{j=1}^{m}
\sum_{a_{1:H}\in\cA(\widetilde x_j)}
\pi_{\mathrm{stu},\theta}(a_{1:H}\mid\widetilde x_j)
\log\frac{\pi_{\mathrm{stu},\theta}(a_{1:H}\mid\widetilde x_j)}
{\pi_w(a_{1:H}\mid\widetilde x_j)}\\
&\le\frac{\lambda}{m}\sum_{j=1}^{m}
\sum_{a_{1:H}\in\cA(\widetilde x_j)}
\pi_{\mathrm{stu},\theta}(a_{1:H}\mid\widetilde x_j)
\left|\log\frac{\pi_{\mathrm{stu},\theta}(a_{1:H}\mid\widetilde x_j)}
{\pi_w(a_{1:H}\mid\widetilde x_j)}\right|\\
&\le\frac{4\lambda BH}{m}\sum_{j=1}^{m}
\underbrace{\sum_{a_{1:H}\in\cA(\widetilde x_j)}
\pi_{\mathrm{stu},\theta}(a_{1:H}\mid\widetilde x_j)}_{=1}=4\lambda BH.
\end{aligned}
\]
For the lower bound, apply the KL non-negativity and we immediately observe that 
\[
\begin{aligned}
C_w(\theta)
&=\frac{\lambda}{m}\sum_{j=1}^{m}
\operatorname{KL}\!\left(
\pi_{\mathrm{stu},\theta}(\cdot\mid\widetilde x_j)
\,\middle\|\,
\pi_w(\cdot\mid\widetilde x_j)\right)\ge0.
\end{aligned}
\]
Here $\lambda/m>0$, so summing the nonnegative divergences
preserves the inequality. Together with the upper bound, this
proves \eqref{eq:itercal-cost-bound}.

Now, we bound the derivative of $C_w$, define the trajectory score
\[
S_{\mathrm{stu},\theta}(x,a_{1:H})
:=\nabla_\theta\log
\pi_{\mathrm{stu},\theta}(a_{1:H}\mid x).
\]
We first bound $S_{\text{stu},\theta}$. Fix a prompt $x$ and parameter $\theta$, by direct differentiation, we have
\[
\begin{aligned}
S_{\mathrm{stu},\theta}(x,a_{1:H})
&=\nabla_\theta\log\pi_{\mathrm{stu},\theta}(a_{1:H}\mid x)=\nabla_\theta\log\prod_{h=1}^{H}
\pi_{\mathrm{stu},\theta}(a_h\mid s_h)=\sum_{h=1}^{H}\nabla_\theta
\log\pi_{\mathrm{stu},\theta}(a_h\mid s_h).
\end{aligned}
\]
We evaluate each derivative in this expression, by the linear softmax parametrization, we have
\[
\begin{aligned}
\log\pi_{\mathrm{stu},\theta}(a_h\mid s_h)
&=\theta^\top\phi_{\mathrm{stu}}(s_h,a_h)
-\log\!\left[\sum_{b\in\cB(s_h)}
e^{\theta^\top\phi_{\mathrm{stu}}(s_h,b)}\right],\\
\nabla_\theta\log\pi_{\mathrm{stu},\theta}(a_h\mid s_h)
&=\phi_{\mathrm{stu}}(s_h,a_h)
-\frac{\sum_{b\in\cB(s_h)}
e^{\theta^\top\phi_{\mathrm{stu}}(s_h,b)}
\phi_{\mathrm{stu}}(s_h,b)}
{\sum_{c\in\cB(s_h)}e^{\theta^\top\phi_{\mathrm{stu}}(s_h,c)}}\\
&=\phi_{\mathrm{stu}}(s_h,a_h)
-\sum_{b\in\cB(s_h)}
\frac{e^{\theta^\top\phi_{\mathrm{stu}}(s_h,b)}}
{\sum_{c\in\cB(s_h)}e^{\theta^\top\phi_{\mathrm{stu}}(s_h,c)}}
\phi_{\mathrm{stu}}(s_h,b)\\
&=\phi_{\mathrm{stu}}(s_h,a_h)
-\sum_{b\in\cB(s_h)}
\pi_{\mathrm{stu},\theta}(b\mid s_h)
\phi_{\mathrm{stu}}(s_h,b)\\
&=:D_{\theta,h}(x,a_{1:h}).
\end{aligned}
\]
Substituting the computed
token derivative into the score expansion at the start of this
step yields
\[
\begin{aligned}
S_{\mathrm{stu},\theta}(x,a_{1:H})
&=\sum_{h=1}^{H}
\nabla_\theta\log\pi_{\mathrm{stu},\theta}(a_h\mid s_h)
=\sum_{h=1}^{H}D_{\theta,h}(x,a_{1:h}).
\end{aligned}
\]
The pointwise norm of each increment is bounded by
\[
\begin{aligned}
\|D_{\theta,h}(x,a_{1:h})\|_2
&=\left\|\phi_{\mathrm{stu}}(s_h,a_h)
-\sum_{b\in\cB(s_h)}\pi_{\mathrm{stu},\theta}(b\mid s_h)
\phi_{\mathrm{stu}}(s_h,b)\right\|_2\\
&\le\|\phi_{\mathrm{stu}}(s_h,a_h)\|_2
+\left\|\sum_{b\in\cB(s_h)}
\pi_{\mathrm{stu},\theta}(b\mid s_h)
\phi_{\mathrm{stu}}(s_h,b)\right\|_2\\
&\le1+\sum_{b\in\cB(s_h)}
\pi_{\mathrm{stu},\theta}(b\mid s_h)
\|\phi_{\mathrm{stu}}(s_h,b)\|_2\\
&\le1+\sum_{b\in\cB(s_h)}
\pi_{\mathrm{stu},\theta}(b\mid s_h)
=2.
\end{aligned}
\]
Consequently, we have
\begin{equation}
\|S_{\mathrm{stu},\theta}(x,a_{1:H})\|_2
=\left\|\sum_{h=1}^{H}D_{\theta,h}(x,a_{1:h})\right\|_2
\le\sum_{h=1}^{H}\|D_{\theta,h}(x,a_{1:h})\|_2
\le2H.\label{eq:itercal-score-pointwise}
\end{equation}
To obtain the sharper second-moment bound, we  use the
conditional centering of the increments.
Given $x,a_{1:h-1}$, the next token has law
$a_h\sim\pi_{\mathrm{stu},\theta}(\cdot\mid s_h)$. Hence, we have
\[
\begin{aligned}
&\EE_{a_h\sim\pi_{\mathrm{stu},\theta}(\cdot\mid s_h)}
\!\left[D_{\theta,h}(x,a_{1:h})\,\middle|\,x,a_{1:h-1}\right]\\
=&\sum_{b\in\cB(s_h)}
\pi_{\mathrm{stu},\theta}(b\mid s_h)
\left[\phi_{\mathrm{stu}}(s_h,b)
-\sum_{c\in\cB(s_h)}
\pi_{\mathrm{stu},\theta}(c\mid s_h)
\phi_{\mathrm{stu}}(s_h,c)\right]\\
=&\sum_{b\in\cB(s_h)}
\pi_{\mathrm{stu},\theta}(b\mid s_h)\phi_{\mathrm{stu}}(s_h,b)-
\left[\sum_{b\in\cB(s_h)}\pi_{\mathrm{stu},\theta}(b\mid s_h)\right]
\left[\sum_{c\in\cB(s_h)}\pi_{\mathrm{stu},\theta}(c\mid s_h)
\phi_{\mathrm{stu}}(s_h,c)\right]\\
=&\sum_{b\in\cB(s_h)}
\pi_{\mathrm{stu},\theta}(b\mid s_h)\phi_{\mathrm{stu}}(s_h,b)
-\sum_{c\in\cB(s_h)}
\pi_{\mathrm{stu},\theta}(c\mid s_h)\phi_{\mathrm{stu}}(s_h,c)=0.
\end{aligned}
\]
The second equality uses that the probabilities sum to one.
For the conditional second moment, expanding the square using
$\|u-v\|_2^2=\|u\|_2^2-2u^\top v+\|v\|_2^2$, we get
\begin{align*}
&\EE_{a_h\sim\pi_{\mathrm{stu},\theta}(\cdot\mid s_h)}
\!\left[\|D_{\theta,h}(x,a_{1:h})\|_2^2
\,\middle|\,x,a_{1:h-1}\right]\\
=&\sum_{b\in\cB(s_h)}\pi_{\mathrm{stu},\theta}(b\mid s_h)
\left\|\phi_{\mathrm{stu}}(s_h,b)
-\sum_{c\in\cB(s_h)}\pi_{\mathrm{stu},\theta}(c\mid s_h)
\phi_{\mathrm{stu}}(s_h,c)\right\|_2^2\\
=&\sum_{b\in\cB(s_h)}\pi_{\mathrm{stu},\theta}(b\mid s_h)
\|\phi_{\mathrm{stu}}(s_h,b)\|_2^2-2
\left[\sum_{b\in\cB(s_h)}\pi_{\mathrm{stu},\theta}(b\mid s_h)
\phi_{\mathrm{stu}}(s_h,b)\right]^\top
\left[\sum_{c\in\cB(s_h)}\pi_{\mathrm{stu},\theta}(c\mid s_h)
\phi_{\mathrm{stu}}(s_h,c)\right]\\
&\ +
\left[\sum_{b\in\cB(s_h)}\pi_{\mathrm{stu},\theta}(b\mid s_h)\right]
\left\|\sum_{c\in\cB(s_h)}\pi_{\mathrm{stu},\theta}(c\mid s_h)
\phi_{\mathrm{stu}}(s_h,c)\right\|_2^2\\
=&\sum_{b\in\cB(s_h)}\pi_{\mathrm{stu},\theta}(b\mid s_h)
\|\phi_{\mathrm{stu}}(s_h,b)\|_2^2
-2\left\|\sum_{c\in\cB(s_h)}\pi_{\mathrm{stu},\theta}(c\mid s_h)
\phi_{\mathrm{stu}}(s_h,c)\right\|_2^2\\
&\ +
\left\|\sum_{c\in\cB(s_h)}\pi_{\mathrm{stu},\theta}(c\mid s_h)
\phi_{\mathrm{stu}}(s_h,c)\right\|_2^2\\
=&\sum_{b\in\cB(s_h)}\pi_{\mathrm{stu},\theta}(b\mid s_h)
\|\phi_{\mathrm{stu}}(s_h,b)\|_2^2
-\left\|\sum_{c\in\cB(s_h)}\pi_{\mathrm{stu},\theta}(c\mid s_h)
\phi_{\mathrm{stu}}(s_h,c)\right\|_2^2\\
\le&\sum_{b\in\cB(s_h)}\pi_{\mathrm{stu},\theta}(b\mid s_h)
\|\phi_{\mathrm{stu}}(s_h,b)\|_2^2\le\sum_{b\in\cB(s_h)}\pi_{\mathrm{stu},\theta}(b\mid s_h)
=1.
\end{align*}

If $h<k$, then $D_{\theta,h}(x,a_{1:h})$ is already determined
by $x,a_{1:k-1}$. Marginalizing the suffix after time $k$ and
then applying iterated expectation, we get
\begin{align*}
&\EE_{a_{1:H}\sim\pi_{\mathrm{stu},\theta}(\cdot\mid x)}
\!\left[D_{\theta,h}(x,a_{1:h})^\top
D_{\theta,k}(x,a_{1:k})\right]\\
=&\EE_{a_{1:k-1}\sim\pi_{\mathrm{stu},\theta}(\cdot\mid x)}
\Bigg[
\EE_{a_k\sim\pi_{\mathrm{stu},\theta}(\cdot\mid s_k)}
\!\left[D_{\theta,h}(x,a_{1:h})^\top
D_{\theta,k}(x,a_{1:k})\,\middle|\,x,a_{1:k-1}\right]
\Bigg]\\
=&\EE_{a_{1:k-1}\sim\pi_{\mathrm{stu},\theta}(\cdot\mid x)}
\Bigg[D_{\theta,h}(x,a_{1:h})^\top
\EE_{a_k\sim\pi_{\mathrm{stu},\theta}(\cdot\mid s_k)}
\!\left[D_{\theta,k}(x,a_{1:k})\,\middle|\,x,a_{1:k-1}\right]
\Bigg]\\
=&\EE_{a_{1:k-1}\sim\pi_{\mathrm{stu},\theta}(\cdot\mid x)}
\!\left[D_{\theta,h}(x,a_{1:h})^\top 0\right]
=0.
\end{align*}
Pulling the earlier increment outside the inner expectation is
valid precisely because the earlier increment is measurable
given that prefix.

By the linearity of conditional means, we further have
\begin{align*}
\EE_{a_{1:H}\sim\pi_{\mathrm{stu},\theta}(\cdot\mid x)}
\!\left[S_{\mathrm{stu},\theta}(x,a_{1:H})\right]&=\sum_{h=1}^{H}
\EE_{a_{1:h-1}\sim\pi_{\mathrm{stu},\theta}(\cdot\mid x)}
\Bigg[\EE_{a_h\sim\pi_{\mathrm{stu},\theta}(\cdot\mid s_h)}
\!\left[D_{\theta,h}(x,a_{1:h})\,\middle|\,x,a_{1:h-1}\right]
\Bigg]\\
&=\sum_{h=1}^{H}
\EE_{a_{1:h-1}\sim\pi_{\mathrm{stu},\theta}(\cdot\mid x)}[0]
=0.
\end{align*}
When $h=1$, the prefix is empty and its expectation has just one
possible value. Expanding the square of the sum of the
increments, including all cross terms, gives
\begin{align*}
&\EE_{a_{1:H}\sim\pi_{\mathrm{stu},\theta}(\cdot\mid x)}
\!\left[\|S_{\mathrm{stu},\theta}(x,a_{1:H})\|_2^2\right]\\
=&\EE_{a_{1:H}\sim\pi_{\mathrm{stu},\theta}(\cdot\mid x)}
\!\left[\left\|\sum_{h=1}^{H}D_{\theta,h}(x,a_{1:h})\right\|_2^2\right]\\
=&\sum_{h=1}^{H}
\EE_{a_{1:H}\sim\pi_{\mathrm{stu},\theta}(\cdot\mid x)}
\!\left[\|D_{\theta,h}(x,a_{1:h})\|_2^2\right]+2\sum_{1\le h<k\le H}
\underbrace{\EE_{a_{1:H}\sim\pi_{\mathrm{stu},\theta}(\cdot\mid x)}
\!\left[D_{\theta,h}(x,a_{1:h})^\top
D_{\theta,k}(x,a_{1:k})\right]}_{=0}\\
=&\sum_{h=1}^{H}
\EE_{a_{1:h-1}\sim\pi_{\mathrm{stu},\theta}(\cdot\mid x)}
\Bigg[\EE_{a_h\sim\pi_{\mathrm{stu},\theta}(\cdot\mid s_h)}
\!\left[\|D_{\theta,h}(x,a_{1:h})\|_2^2
\,\middle|\,x,a_{1:h-1}\right]\Bigg]\\
\le&\sum_{h=1}^{H}
\EE_{a_{1:h-1}\sim\pi_{\mathrm{stu},\theta}(\cdot\mid x)}[1]
=H.
\end{align*}
We have therefore proved the two score identities
\begin{equation}
\begin{aligned}
\EE_{a_{1:H}\sim\pi_{\mathrm{stu},\theta}(\cdot\mid x)}
\!\left[S_{\mathrm{stu},\theta}(x,a_{1:H})\right]&=0,\ 
\EE_{a_{1:H}\sim\pi_{\mathrm{stu},\theta}(\cdot\mid x)}
\!\left[\|S_{\mathrm{stu},\theta}(x,a_{1:H})\|_2^2\right]&\le H.
\end{aligned}
\label{eq:itercal-score-moment}
\end{equation}
Now, we can bound $\nabla_\theta C_w(\theta)$. Fix $w$ while differentiating with respect to $\theta$.
\begin{align*}
\nabla_\theta C_w(\theta)
&=\nabla_\theta\Bigg[
\frac{\lambda}{m}\sum_{j=1}^{m}
\sum_{a_{1:H}\in\cA(\widetilde x_j)}
\pi_{\mathrm{stu},\theta}(a_{1:H}\mid\widetilde x_j)
\log\frac{\pi_{\mathrm{stu},\theta}(a_{1:H}\mid\widetilde x_j)}
{\pi_w(a_{1:H}\mid\widetilde x_j)}\Bigg]\\
&=\frac{\lambda}{m}\sum_{j=1}^{m}
\sum_{a_{1:H}\in\cA(\widetilde x_j)}
\nabla_\theta\Bigg[
\pi_{\mathrm{stu},\theta}(a_{1:H}\mid\widetilde x_j)
\log\frac{\pi_{\mathrm{stu},\theta}(a_{1:H}\mid\widetilde x_j)}
{\pi_w(a_{1:H}\mid\widetilde x_j)}\Bigg].
\end{align*}
We now expand the derivative of each summand. First, by direct algebra, for any $x$, we have
\[
\begin{aligned}
  \nabla_\theta\pi_{\mathrm{stu},\theta}(a_{1:H}\mid x)
&=\nabla_\theta\exp\!\left[
\log\pi_{\mathrm{stu},\theta}(a_{1:H}\mid x)\right]=\pi_{\mathrm{stu},\theta}(a_{1:H}\mid x)
\nabla_\theta\log\pi_{\mathrm{stu},\theta}(a_{1:H}\mid x)\\&=\pi_{\mathrm{stu},\theta}(a_{1:H}\mid x)
S_{\mathrm{stu},\theta}(x,a_{1:H}).
\end{aligned}
\]
Moreover, $\pi_w$ does not depend on $\theta$, and hence we have
\[
\begin{aligned}
\nabla_\theta\log
\frac{\pi_{\mathrm{stu},\theta}(a_{1:H}\mid x)}
{\pi_w(a_{1:H}\mid x)}
&=\nabla_\theta\log\pi_{\mathrm{stu},\theta}(a_{1:H}\mid x)
-\nabla_\theta\log\pi_w(a_{1:H}\mid x)=S_{\mathrm{stu},\theta}(x,a_{1:H})-0.
\end{aligned}
\]
Therefore, by the product rule of differentiation, we have
\begin{align*}
&\nabla_\theta\left[
\pi_{\mathrm{stu},\theta}(a_{1:H}\mid x)
\log\frac{\pi_{\mathrm{stu},\theta}(a_{1:H}\mid x)}
{\pi_w(a_{1:H}\mid x)}\right]\\
=&
\left[\nabla_\theta\pi_{\mathrm{stu},\theta}(a_{1:H}\mid x)\right]
\log\frac{\pi_{\mathrm{stu},\theta}(a_{1:H}\mid x)}
{\pi_w(a_{1:H}\mid x)}+
\pi_{\mathrm{stu},\theta}(a_{1:H}\mid x)
\nabla_\theta\log
\frac{\pi_{\mathrm{stu},\theta}(a_{1:H}\mid x)}
{\pi_w(a_{1:H}\mid x)}\\
=&
\pi_{\mathrm{stu},\theta}(a_{1:H}\mid x)
S_{\mathrm{stu},\theta}(x,a_{1:H})
\log\frac{\pi_{\mathrm{stu},\theta}(a_{1:H}\mid x)}
{\pi_w(a_{1:H}\mid x)}+
\pi_{\mathrm{stu},\theta}(a_{1:H}\mid x)
S_{\mathrm{stu},\theta}(x,a_{1:H}).
\end{align*}
For the second term, its sum is
the score expectation already evaluated in
\eqref{eq:itercal-score-moment}:
\[
\begin{aligned}
&\sum_{a_{1:H}\in\cA(x)}
\pi_{\mathrm{stu},\theta}(a_{1:H}\mid x)
S_{\mathrm{stu},\theta}(x,a_{1:H})=\EE_{a_{1:H}\sim\pi_{\mathrm{stu},\theta}(\cdot\mid x)}
\!\left[S_{\mathrm{stu},\theta}(x,a_{1:H})\right]
=0.
\end{aligned}
\]
The first equality is the definition of expectation and the second is the first identity in
\eqref{eq:itercal-score-moment}, applied to this $x$ and $\theta$.
Substituting both product-rule terms into the definition of
$C_w$, and then using this zero sum, we obtain
\begin{align}
\nabla_\theta C_w(\theta)
&=\frac{\lambda}{m}\sum_{j=1}^{m}
\sum_{a_{1:H}\in\cA(\widetilde x_j)}
\pi_{\mathrm{stu},\theta}(a_{1:H}\mid\widetilde x_j)
S_{\mathrm{stu},\theta}(\widetilde x_j,a_{1:H})\times
\log\frac{\pi_{\mathrm{stu},\theta}(a_{1:H}\mid\widetilde x_j)}
{\pi_w(a_{1:H}\mid\widetilde x_j)}\nonumber\\
&\ +\frac{\lambda}{m}\sum_{j=1}^{m}
\underbrace{\sum_{a_{1:H}\in\cA(\widetilde x_j)}
\pi_{\mathrm{stu},\theta}(a_{1:H}\mid\widetilde x_j)
S_{\mathrm{stu},\theta}(\widetilde x_j,a_{1:H})}_{=0}\nonumber\\
&=\frac{\lambda}{m}\sum_{j=1}^{m}
\EE_{a_{1:H}\sim\pi_{\mathrm{stu},\theta}(\cdot\mid\widetilde x_j)}
\Bigg[S_{\mathrm{stu},\theta}(\widetilde x_j,a_{1:H})\times
\log\frac{\pi_{\mathrm{stu},\theta}(a_{1:H}\mid\widetilde x_j)}
{\pi_w(a_{1:H}\mid\widetilde x_j)}\Bigg].
\label{eq:itercal-cost-gradient}
\end{align}
In the last line, we rewrite the finite weighted
sum as an expectation.

The quantity we want to control has the exact representation

\begin{align}
|C_w(\theta)-C_w(\vartheta)|
&=\left|\int_0^1\frac{d}{du}
C_w\bigl(\vartheta+u(\theta-\vartheta)\bigr)\,du\right|=\left|\int_0^1
\nabla_\theta C_w\bigl(\vartheta+u(\theta-\vartheta)\bigr)^\top
(\theta-\vartheta)\,du\right|.\label{eq:itercal-student-segment}
\end{align}

These identities follow from the fundamental theorem of calculus. We bound the gradient appearing here first
and then return to this expression. By the triangle inequality, the pointwise
log-ratio bound in
Lemma~\ref{lem:itercal-z}, we have
\begin{align*}
\|\nabla_\theta C_w(\theta)\|_2
&=\frac{\lambda}{m}\Bigg\|\sum_{j=1}^{m}
\sum_{a_{1:H}\in\cA(\widetilde x_j)}
\pi_{\mathrm{stu},\theta}(a_{1:H}\mid\widetilde x_j)
S_{\mathrm{stu},\theta}(\widetilde x_j,a_{1:H})\times
\log\frac{\pi_{\mathrm{stu},\theta}(a_{1:H}\mid\widetilde x_j)}
{\pi_w(a_{1:H}\mid\widetilde x_j)}\Bigg\|_2\\
&\le\frac{\lambda}{m}\sum_{j=1}^{m}
\EE_{a_{1:H}\sim\pi_{\mathrm{stu},\theta}(\cdot\mid\widetilde x_j)}
\Bigg[\|S_{\mathrm{stu},\theta}(\widetilde x_j,a_{1:H})\|_2\times
\left|\log\frac{\pi_{\mathrm{stu},\theta}(a_{1:H}\mid\widetilde x_j)}
{\pi_w(a_{1:H}\mid\widetilde x_j)}\right|\Bigg]\\
&\le\frac{4\lambda BH}{m}\sum_{j=1}^{m}
\EE_{a_{1:H}\sim\pi_{\mathrm{stu},\theta}(\cdot\mid\widetilde x_j)}
\!\left[\|S_{\mathrm{stu},\theta}(\widetilde x_j,a_{1:H})\|_2\right]\\
&\le\frac{4\lambda BH}{m}\sum_{j=1}^{m}
\left(\EE_{a_{1:H}\sim\pi_{\mathrm{stu},\theta}(\cdot\mid\widetilde x_j)}
\!\left[\|S_{\mathrm{stu},\theta}(\widetilde x_j,a_{1:H})\|_2^2\right]
\right)^{1/2}\\
&\le\frac{4\lambda BH}{m}\sum_{j=1}^{m}\sqrt H
=4\lambda BH^{3/2}.
\end{align*}

In the second inequality, we apply Lemma \ref{lem:itercal-z} and in the last inequality, we apply \eqref{eq:itercal-score-moment}.

For $u\in[0,1]$, the segment between $\vartheta$ and $\theta$
remains inside the student ball because
\[
\begin{aligned}
\|\vartheta+u(\theta-\vartheta)\|_2
&=\|(1-u)\vartheta+u\theta\|_2\le(1-u)\|\vartheta\|_2+u\|\theta\|_2
\le(1-u)B+uB=B.
\end{aligned}
\]
Thus, the gradient bound applies at every point of the segment in
\eqref{eq:itercal-student-segment}. Return to that representation and we get
\[
\begin{aligned}
|C_w(\theta)-C_w(\vartheta)|
&\le\int_0^1
\left|\nabla_\theta C_w\bigl(\vartheta+u(\theta-\vartheta)\bigr)^\top
(\theta-\vartheta)\right|\,du\\
&\le\|\theta-\vartheta\|_2\int_0^1
\left\|\nabla_\theta C_w\bigl(\vartheta+u(\theta-\vartheta)\bigr)\right\|_2\,du\\
&\le\|\theta-\vartheta\|_2\int_0^1
4\lambda BH^{3/2}\,du=4\lambda BH^{3/2}\|\theta-\vartheta\|_2.
\end{aligned}
\]
This proves \eqref{eq:itercal-stu-lip}.
To  prove \eqref{eq:itercal-w-lip}, we start with the difference that must be bounded and expand two costs, we have
\begin{align}
C_w(\theta)-C_v(\theta)
&=\frac{\lambda}{m}\sum_{j=1}^{m}
\sum_{a_{1:H}\in\cA(\widetilde x_j)}
\pi_{\mathrm{stu},\theta}(a_{1:H}\mid\widetilde x_j)\times[
\log\frac{\pi_{\mathrm{stu},\theta}(a_{1:H}\mid\widetilde x_j)}
{\pi_w(a_{1:H}\mid\widetilde x_j)}
-\log\frac{\pi_{\mathrm{stu},\theta}(a_{1:H}\mid\widetilde x_j)}
{\pi_v(a_{1:H}\mid\widetilde x_j)}]\nonumber\\
&=\frac{\lambda}{m}\sum_{j=1}^{m}
\sum_{a_{1:H}\in\cA(\widetilde x_j)}
\pi_{\mathrm{stu},\theta}(a_{1:H}\mid\widetilde x_j)\times\Big[
-\log\pi_w(a_{1:H}\mid\widetilde x_j)+\log\pi_v(a_{1:H}\mid\widetilde x_j)\Big]\nonumber\\
&=\frac{\lambda}{m}\sum_{j=1}^{m}
\sum_{a_{1:H}\in\cA(\widetilde x_j)}
\pi_{\mathrm{stu},\theta}(a_{1:H}\mid\widetilde x_j)\times\left[
\log\pi_v(a_{1:H}\mid\widetilde x_j)
-\log\pi_w(a_{1:H}\mid\widetilde x_j)\right].
\label{eq:itercal-calibration-cost-difference}
\end{align}
Thus, it suffices to control the difference of calibration log
probabilities. We derive that control from the calibration
softmax itself:
\[
\log\pi_w(a\mid s)=w^\top\phi(s,a)
-\log\!\left[\sum_{b\in\cB(s)}e^{w^\top\phi(s,b)}\right],\]
\begin{align*}
\nabla_w\log\pi_w(a\mid s)=&\phi(s,a)
-\frac{\sum_{b\in\cB(s)}e^{w^\top\phi(s,b)}\phi(s,b)}
{\sum_{c\in\cB(s)}e^{w^\top\phi(s,c)}}=\phi(s,a)
-\sum_{b\in\cB(s)}
\frac{e^{w^\top\phi(s,b)}}{\sum_{c\in\cB(s)}e^{w^\top\phi(s,c)}}
\phi(s,b)\\
=&\phi(s,a)-\sum_{b\in\cB(s)}\pi_w(b\mid s)\phi(s,b).
\end{align*}
We bound this derivative by the following:
\[
\begin{aligned}
\|\nabla_w\log\pi_w(a\mid s)\|_2
&=\left\|\phi(s,a)-\sum_{b\in\cB(s)}
\pi_w(b\mid s)\phi(s,b)\right\|_2\le\|\phi(s,a)\|_2
+\left\|\sum_{b\in\cB(s)}\pi_w(b\mid s)\phi(s,b)\right\|_2\\
&\le1+\sum_{b\in\cB(s)}\pi_w(b\mid s)\|\phi(s,b)\|_2\le1+\sum_{b\in\cB(s)}\pi_w(b\mid s)=2.
\end{aligned}
\]
Since $W$ is convex, $v+u(w-v)\in W$ for $u\in[0,1]$.
The same one-dimensional integration argument, now displayed
for the token log probability, gives
\[
\begin{aligned}
\log\pi_w(a\mid s)-\log\pi_v(a\mid s)
&=\int_0^1\frac{d}{du}\log\pi_{v+u(w-v)}(a\mid s)\,du\\
&=\int_0^1
\left.\nabla_w\log\pi_w(a\mid s)\right|_{w=v+u(w-v)}^{\top}
(w-v)\,du.
\end{aligned}
\]
Consequently, we have
\[
\begin{aligned}
|\log\pi_w(a\mid s)-\log\pi_v(a\mid s)|
&\le\int_0^1
\left\|\left.\nabla_w\log\pi_w(a\mid s)\right|_{w=v+u(w-v)}\right\|_2
\|w-v\|_2\,du\\
&\le\int_0^1 2\|w-v\|_2\,du
=2\|w-v\|_2.
\end{aligned}
\]
Thus, we have
\[
\begin{aligned}
|\log\pi_w(a_{1:H}\mid x)-\log\pi_v(a_{1:H}\mid x)|&=\left|\sum_{h=1}^{H}
\bigl[\log\pi_w(a_h\mid s_h)-\log\pi_v(a_h\mid s_h)\bigr]\right|\\
&\le\sum_{h=1}^{H}
|\log\pi_w(a_h\mid s_h)-\log\pi_v(a_h\mid s_h)|\\
&\le\sum_{h=1}^{H}2\|w-v\|_2
=2H\|w-v\|_2.
\end{aligned}
\]
Now, we return to the cost-difference representation
\eqref{eq:itercal-calibration-cost-difference}. Plugging the bound derived above back, we obtain, 
\begin{align*}
|C_w(\theta)-C_v(\theta)|
&\le\frac{\lambda}{m}\sum_{j=1}^{m}
\sum_{a_{1:H}\in\cA(\widetilde x_j)}
\pi_{\mathrm{stu},\theta}(a_{1:H}\mid\widetilde x_j)\times
|\log\pi_v(a_{1:H}\mid\widetilde x_j)
-\log\pi_w(a_{1:H}\mid\widetilde x_j)|\\
&\le\frac{2\lambda H\|w-v\|_2}{m}\sum_{j=1}^{m}
\sum_{a_{1:H}\in\cA(\widetilde x_j)}
\pi_{\mathrm{stu},\theta}(a_{1:H}\mid\widetilde x_j)\\
&=\frac{2\lambda H\|w-v\|_2}{m}\sum_{j=1}^{m}1
=2\lambda H\|w-v\|_2.
\end{align*}
The right-hand side is independent of $\theta$. Taking the
supremum over $\theta\in\Theta$ proves \eqref{eq:itercal-w-lip}.
\end{proof}

\begin{proof}[Proof of Lemma~\ref{lem:itercal-smoothness}]
We first verify analyticity, which will also be used in the two
\L{}ojasiewicz arguments below. For each fixed legal state $s$ and
token $a\in\cB(s)$, the softmax expressions are
\[
\begin{aligned}
\pi_{\mathrm{stu},\theta}(a\mid s)
&=\frac{e^{\theta^\top\phi_{\mathrm{stu}}(s,a)}}
{\sum_{b\in\cB(s)}e^{\theta^\top\phi_{\mathrm{stu}}(s,b)}},\\
\log\pi_{\mathrm{stu},\theta}(a\mid s)
&=\theta^\top\phi_{\mathrm{stu}}(s,a)
-\log\!\left[\sum_{b\in\cB(s)}
e^{\theta^\top\phi_{\mathrm{stu}}(s,b)}\right].
\end{aligned}
\]
The exponentials of these linear functions are analytic, and their
finite sum is strictly positive on $\RR^d$. Since reciprocals and
logarithms are analytic on $(0,\infty)$, both displayed expressions
are analytic. For each fixed feasible answer, the chain rule for
trajectory probabilities gives
\[
\begin{aligned}
\pi_{\mathrm{stu},\theta}(a_{1:H}\mid x)
&=\prod_{h=1}^{H}\pi_{\mathrm{stu},\theta}(a_h\mid x,a_{1:h-1}),\\
\log\pi_{\mathrm{stu},\theta}(a_{1:H}\mid x)
&=\sum_{h=1}^{H}\log\pi_{\mathrm{stu},\theta}(a_h\mid x,a_{1:h-1}).
\end{aligned}
\]
Finite products and sums preserve analyticity. For any fixed
positive answer laws $Q_j$, expanding the average KL gives
\begin{align*}
&\frac1m\sum_{j=1}^{m}
\KL\!\left(\pi_{\mathrm{stu},\theta}(\cdot\mid\widetilde x_j)
\,\middle\|\,Q_j\right)=\frac1m\sum_{j=1}^{m}
\sum_{a_{1:H}\in\cA(\widetilde x_j)}
\pi_{\mathrm{stu},\theta}(a_{1:H}\mid\widetilde x_j)\times
\left[
\log\pi_{\mathrm{stu},\theta}(a_{1:H}\mid\widetilde x_j)
-\log Q_j(a_{1:H})\right].
\end{align*}
Each $\log Q_j(a_{1:H})$ is a finite constant, and the sum has
finitely many analytic summands. Thus this function is analytic.
Choosing $Q_j=\pi_w(\cdot\mid\widetilde x_j)$ and multiplying
by $\lambda$ proves the assertion for $C_w$; choosing
$Q_j=\pi_{\mathrm{stu},\theta^\dagger_{\lambda,m}}
(\cdot\mid\widetilde x_j)$ proves it for
$\cK_{\lambda,m}$. In particular, $F=C_{w^\star_\lambda}$ and
$\Delta_{\lambda,m}=F-F(\theta^\dagger_{\lambda,m})$ are analytic.

We next bound the Hessian of $C_w$. For this calculation, define
the unscaled trajectory log ratio
\[
\ell_w(\theta;x,a_{1:H})
:=\log\frac{\pi_{\mathrm{stu},\theta}(a_{1:H}\mid x)}
{\pi_w(a_{1:H}\mid x)}.
\]
By Lemma~\ref{lem:itercal-z}, $|\ell_w|\le4BH$ on $\Theta$.
Recall from \eqref{eq:itercal-cost-gradient} that
\[
\nabla_\theta C_w(\theta)
=\frac{\lambda}{m}\sum_{j=1}^{m}
\sum_{a_{1:H}\in\cA(\widetilde x_j)}
\pi_{\mathrm{stu},\theta}(a_{1:H}\mid\widetilde x_j)
\ell_w(\theta;\widetilde x_j,a_{1:H})
S_{\mathrm{stu},\theta}(\widetilde x_j,a_{1:H}).
\]
We take the derivative again and use the product rule to compute the Hessian matrix
\begin{align}
\nabla_\theta^2 C_w(\theta)
&=\frac{\lambda}{m}\sum_{j=1}^{m}
\EE_{a_{1:H}\sim\pi_{\mathrm{stu},\theta}(\cdot\mid\widetilde x_j)}
\Bigl[
\bigl(\ell_w(\theta;\widetilde x_j,a_{1:H})+1\bigr)
S_{\mathrm{stu},\theta}(\widetilde x_j,a_{1:H})
S_{\mathrm{stu},\theta}(\widetilde x_j,a_{1:H})^\top
\nonumber\\[-2pt]
&\hspace{34mm}
+\ell_w(\theta;\widetilde x_j,a_{1:H})
\nabla_\theta S_{\mathrm{stu},\theta}(\widetilde x_j,a_{1:H})
\Bigr].
\label{eq:itercal-hessian-expression}
\end{align}
We already know from \eqref{eq:itercal-score-moment} that the expected
squared score norm is at most $H$. To bound its derivative, we
differentiate the token-score formula in the proof of
Lemma~\ref{lem:itercal-lipschitz}:
\begin{align*}
\nabla_\theta S_{\mathrm{stu},\theta}(x,a_{1:H})
&=-\sum_{h=1}^{H}\Bigg[
\sum_{b\in\cB(s_h)}\pi_{\mathrm{stu},\theta}(b\mid s_h)
\phi_{\mathrm{stu}}(s_h,b)\phi_{\mathrm{stu}}(s_h,b)^\top\\
&
-\left(\sum_{b\in\cB(s_h)}\pi_{\mathrm{stu},\theta}(b\mid s_h)
\phi_{\mathrm{stu}}(s_h,b)\right)
\left(\sum_{b\in\cB(s_h)}\pi_{\mathrm{stu},\theta}(b\mid s_h)
\phi_{\mathrm{stu}}(s_h,b)\right)^\top\Bigg].
\end{align*}
This is a negative covariance matrix and is negative
semidefinite. For any unit vector $v\in\RR^d$, we have
\begin{align*}
0
&\le-v^\top\nabla_\theta S_{\mathrm{stu},\theta}(x,a_{1:H})v\\
&=\sum_{h=1}^{H}\Bigg[
\sum_{b\in\cB(s_h)}\pi_{\mathrm{stu},\theta}(b\mid s_h)
\bigl(v^\top\phi_{\mathrm{stu}}(s_h,b)\bigr)^2
-\left(\sum_{b\in\cB(s_h)}\pi_{\mathrm{stu},\theta}(b\mid s_h)
v^\top\phi_{\mathrm{stu}}(s_h,b)\right)^2\Bigg]\\
&\le\sum_{h=1}^{H}\sum_{b\in\cB(s_h)}
\pi_{\mathrm{stu},\theta}(b\mid s_h)
\|v\|_2^2\|\phi_{\mathrm{stu}}(s_h,b)\|_2^2
\le\sum_{h=1}^{H}1=H.
\end{align*}
The first upper bound discards a nonnegative square and applies
Cauchy-Schwarz; the last uses the feature norm bound and that the
token probabilities sum to one. Thus
$\|\nabla_\theta S_{\mathrm{stu},\theta}(x,a_{1:H})\|_{\mathrm{op}}
\le H$. Taking operator norms in \eqref{eq:itercal-hessian-expression}
and using $\|SS^\top\|_{\mathrm{op}}=\|S\|_2^2$ now yields
\begin{align*}
\|\nabla_\theta^2 C_w(\theta)\|_{\mathrm{op}}
&\le\frac{\lambda}{m}\sum_{j=1}^{m}
\EE_{a_{1:H}\sim\pi_{\mathrm{stu},\theta}(\cdot\mid\widetilde x_j)}
\!\left[
(4BH+1)\|S_{\mathrm{stu},\theta}(\widetilde x_j,a_{1:H})\|_2^2
+4BH^2\right]\\
&\le\frac{\lambda}{m}\sum_{j=1}^{m}
\left[(4BH+1)H+4BH^2\right]
=\lambda H(1+8BH)=L_{\mathrm{st}}.
\end{align*}
This proves \eqref{eq:itercal-hessian-bound}. Since $\Theta$ is
convex, the segment joining $\vartheta$ and $\theta$ lies in
$\Theta$. The fundamental theorem of calculus then gives
\begin{align*}
\|\nabla_\theta C_w(\theta)-\nabla_\theta C_w(\vartheta)\|_2
&=\left\|\int_0^1
\nabla_\theta^2 C_w(\vartheta+u(\theta-\vartheta))
(\theta-\vartheta)\,du\right\|_2\le\int_0^1 L_{\mathrm{st}}\|\theta-\vartheta\|_2\,du
=L_{\mathrm{st}}\|\theta-\vartheta\|_2,
\end{align*}
which proves \eqref{eq:itercal-gradient-smoothness}.

Finally, subtract the two representations in
\eqref{eq:itercal-cost-gradient}. The student law and its score agree
in both terms, so cancellation of their log probabilities gives
\begin{align*}
\nabla_\theta C_w(\theta)-\nabla_\theta C_v(\theta)
&=\frac{\lambda}{m}\sum_{j=1}^{m}
\EE_{a_{1:H}\sim\pi_{\mathrm{stu},\theta}(\cdot\mid\widetilde x_j)}
\!\left[
S_{\mathrm{stu},\theta}(\widetilde x_j,a_{1:H})
\log\frac{\pi_v(a_{1:H}\mid\widetilde x_j)}
{\pi_w(a_{1:H}\mid\widetilde x_j)}\right].
\end{align*}
The proof of \eqref{eq:itercal-w-lip} already established the
pointwise bound
$|\log(\pi_v(a_{1:H}\mid x)/\pi_w(a_{1:H}\mid x))|
\le2H\|w-v\|_2$. Applying that bound and Cauchy-Schwarz to the
score expectation, and then using \eqref{eq:itercal-score-moment},
we obtain
\begin{align*}
\|\nabla_\theta C_w(\theta)-\nabla_\theta C_v(\theta)\|_2
&\le\frac{2\lambda H\|w-v\|_2}{m}
\sum_{j=1}^{m}
\EE_{a_{1:H}\sim\pi_{\mathrm{stu},\theta}(\cdot\mid\widetilde x_j)}
\!\left[\|S_{\mathrm{stu},\theta}(\widetilde x_j,a_{1:H})\|_2\right]\\
&\le\frac{2\lambda H\|w-v\|_2}{m}
\sum_{j=1}^{m}
\left(\EE_{a_{1:H}\sim\pi_{\mathrm{stu},\theta}(\cdot\mid\widetilde x_j)}
\!\left[\|S_{\mathrm{stu},\theta}(\widetilde x_j,a_{1:H})\|_2^2\right]
\right)^{1/2}\\
&\le2\lambda H^{3/2}\|w-v\|_2.
\end{align*}
This proves \eqref{eq:itercal-gradient-calibration}.
\end{proof}

\begin{proof}[Proof of Lemma \ref{lem:itercal-unbiased}]
We first compute the conditional mean of each summand, then bound
the variance of their average. Conditional on $\cG_{t+1}$,
$\theta_t$ and $w_{t+1}$ are fixed. For every $\ell$, Algorithm
\ref{alg:itercal} samples $j_{t+1,\ell}^{\mathrm g}$ uniformly and
then samples an answer from the current student at that question.
Thus, for $a_{1:H}\in\cA(\widetilde x_j)$, we have
\begin{align}
\Pr\!\left(j_{t+1,\ell}^{\mathrm g}=j,\,
a^{\mathrm g}_{t+1,\ell,1:H}=a_{1:H}
\,\middle|\,\cG_{t+1}\right)\nonumber&=
\Pr\!\left(j_{t+1,\ell}^{\mathrm g}=j
\,\middle|\,\cG_{t+1}\right)\nonumber\times
\Pr\!\left(a^{\mathrm g}_{t+1,\ell,1:H}=a_{1:H}
\,\middle|\,\cG_{t+1},j_{t+1,\ell}^{\mathrm g}=j\right)\nonumber\\
&=\frac1m\pi_{\mathrm{stu},\theta_t}
(a_{1:H}\mid\widetilde x_j).
\label{eq:P(j)}
\end{align}
Starting from the definition of
$\widehat g_{t+1,\ell}^{\mathrm{stu}}$, we obtain
\begin{align*}
\EE\!\left[\widehat g_{t+1,\ell}^{\mathrm{stu}}
\,\middle|\,\cG_{t+1}\right]
&=\EE\!\left[
S_{\mathrm{stu},\theta_t}
(\widetilde x_{j_{t+1,\ell}^{\mathrm g}},a^{\mathrm g}_{t+1,\ell,1:H})
Z_{t+1}(\theta_t;\widetilde x_{j_{t+1,\ell}^{\mathrm g}},
a^{\mathrm g}_{t+1,\ell,1:H})
\,\middle|\,\cG_{t+1}\right]\\
&=\frac1m\sum_{j=1}^{m}
\sum_{a_{1:H}\in\cA(\widetilde x_j)}
\pi_{\mathrm{stu},\theta_t}(a_{1:H}\mid\widetilde x_j)
S_{\mathrm{stu},\theta_t}(\widetilde x_j,a_{1:H})
Z_{t+1}(\theta_t;\widetilde x_j,a_{1:H})\\
&=\frac{\lambda}{m}\sum_{j=1}^{m}
\sum_{a_{1:H}\in\cA(\widetilde x_j)}
\pi_{\mathrm{stu},\theta_t}(a_{1:H}\mid\widetilde x_j)
S_{\mathrm{stu},\theta_t}(\widetilde x_j,a_{1:H})
\log\frac{\pi_{\mathrm{stu},\theta_t}(a_{1:H}\mid\widetilde x_j)}
{\pi_{w_{t+1}}(a_{1:H}\mid\widetilde x_j)}\\
&=\nabla_\theta C_{w_{t+1}}(\theta_t)
=\nabla_\theta C_{t+1}(\theta_t).
\end{align*}
The second equality uses \eqref{eq:P(j)}, the third substitutes
the definition of $Z_{t+1}$, and the fourth applies the
cost-gradient identity \eqref{eq:itercal-cost-gradient}.
By linearity of conditional expectation, we have,
\[
\begin{aligned}
\EE\!\left[\widehat g_{t+1}^{\mathrm{stu}}
\,\middle|\,\cG_{t+1}\right]
&=\frac1{b_{t+1}}\sum_{\ell=1}^{b_{t+1}}
\EE\!\left[\widehat g_{t+1,\ell}^{\mathrm{stu}}
\,\middle|\,\cG_{t+1}\right]=\frac{b_{t+1}}{b_{t+1}}\nabla_\theta C_{t+1}(\theta_t)
=\nabla_\theta C_{t+1}(\theta_t).
\end{aligned}
\]

To control the variance, we first bound the second moment. Its definition and \eqref{eq:P(j)} give
\begin{align*}
\EE\!\left[\|\widehat g_{t+1,\ell}^{\mathrm{stu}}\|_2^2
\,\middle|\,\cG_{t+1}\right]&=\frac1m\sum_{j=1}^{m}
\sum_{a_{1:H}\in\cA(\widetilde x_j)}
\pi_{\mathrm{stu},\theta_t}(a_{1:H}\mid\widetilde x_j)
\|S_{\mathrm{stu},\theta_t}(\widetilde x_j,a_{1:H})\|_2^2
[Z_{t+1}(\theta_t;\widetilde x_j,a_{1:H})]^2\\
&\le\frac{16\lambda^2B^2H^2}{m}
\sum_{j=1}^{m}
\sum_{a_{1:H}\in\cA(\widetilde x_j)}
\pi_{\mathrm{stu},\theta_t}(a_{1:H}\mid\widetilde x_j)
\|S_{\mathrm{stu},\theta_t}(\widetilde x_j,a_{1:H})\|_2^2\\
&=\frac{16\lambda^2B^2H^2}{m}
\sum_{j=1}^{m}
\EE_{a_{1:H}\sim\pi_{\mathrm{stu},\theta_t}(\cdot\mid\widetilde x_j)}
\!\left[\|S_{\mathrm{stu},\theta_t}(\widetilde x_j,a_{1:H})\|_2^2\right]\\
&\le\frac{16\lambda^2B^2H^2}{m}\sum_{j=1}^{m}H
=16\lambda^2B^2H^3.
\end{align*}
The first inequality uses $|Z_{t+1}|\le4\lambda BH$ from
Lemma~\ref{lem:itercal-z}. The last inequality applies the score
second-moment bound \eqref{eq:itercal-score-moment} at each target
question.

For this proof, we set
$e_{t+1,\ell}:=\widehat g_{t+1,\ell}^{\mathrm{stu}}
-\nabla_\theta C_{t+1}(\theta_t)$.
The identity above yields
$\EE[e_{t+1,\ell}\mid\cG_{t+1}]=0$. Expanding the square yields
\begin{align*}
\EE[\|e_{t+1,\ell}\|_2^2\mid\cG_{t+1}]
&=\EE[\|\widehat g_{t+1,\ell}^{\mathrm{stu}}\|_2^2
\mid\cG_{t+1}]-2\left\langle
\EE[\widehat g_{t+1,\ell}^{\mathrm{stu}}\mid\cG_{t+1}],
\nabla_\theta C_{t+1}(\theta_t)\right\rangle
+\|\nabla_\theta C_{t+1}(\theta_t)\|_2^2\\
&=\EE[\|\widehat g_{t+1,\ell}^{\mathrm{stu}}\|_2^2
\mid\cG_{t+1}]
-\|\nabla_\theta C_{t+1}(\theta_t)\|_2^2
\le16\lambda^2B^2H^3.
\end{align*}
For $\ell\ne\ell'$, the two prompt-answer pairs are independent
conditional on $\cG_{t+1}$, and therefore
\[
\EE[\langle e_{t+1,\ell},e_{t+1,\ell'}\rangle
\mid\cG_{t+1}]
=
\left\langle
\EE[e_{t+1,\ell}\mid\cG_{t+1}],
\EE[e_{t+1,\ell'}\mid\cG_{t+1}]
\right\rangle=0.
\]
Using the definition of the averaged estimator, we conclude that
\begin{align*}
\EE\!\left[
\left\|\widehat g_{t+1}^{\mathrm{stu}}
-\nabla_\theta C_{t+1}(\theta_t)\right\|_2^2
\,\middle|\,\cG_{t+1}\right]
&=\EE\!\left[
\left\|\frac1{b_{t+1}}\sum_{\ell=1}^{b_{t+1}}
e_{t+1,\ell}\right\|_2^2
\,\middle|\,\cG_{t+1}\right]\\
&=\frac1{b_{t+1}^2}\sum_{\ell=1}^{b_{t+1}}
\EE[\|e_{t+1,\ell}\|_2^2\mid\cG_{t+1}]
+\frac2{b_{t+1}^2}\sum_{\ell<\ell'}
\EE[\langle e_{t+1,\ell},e_{t+1,\ell'}\rangle
\mid\cG_{t+1}]\\
&\le\frac{b_{t+1}(16\lambda^2B^2H^3)}{b_{t+1}^2}
=\frac{16\lambda^2B^2H^3}{b_{t+1}}.
\end{align*}
This proves \eqref{eq:itercal-student-variance}.

Finally, $\theta_t$ is $\cF_t$-measurable, so
$\sigma(w_{t+1},\theta_t)\subseteq\cG_{t+1}$.
The tower property gives the remaining conditional identity:
\begin{align*}
\EE\!\left[\widehat g_{t+1}^{\mathrm{stu}}
\,\middle|\,w_{t+1},\theta_t\right]
&=\EE\!\left[
\EE[\widehat g_{t+1}^{\mathrm{stu}}\mid\cG_{t+1}]
\,\middle|\,w_{t+1},\theta_t\right]=\EE\!\left[\nabla_\theta C_{w_{t+1}}(\theta_t)
\,\middle|\,w_{t+1},\theta_t\right]
=\nabla_\theta C_{t+1}(\theta_t).
\end{align*}
The last equality holds because the displayed gradient is a
function of $w_{t+1}$ and $\theta_t$. This completes the proof.
\end{proof}

\begin{proof}[Proof of Lemma \ref{lem:itercal-selection}]
Conditional on $\cH_t$, the candidate policies and $w_t$ are fixed. Since we sample $j_{t,k,\ell}$ uniformly from $\{1,\ldots,m\}$
and then generate $a^{\mathrm{val}}_{t,k,\ell,1:H}$ from
$\pi_{\mathrm{stu},\vartheta_{t,k}}
(\cdot\mid\widetilde x_{j_{t,k,\ell}})$, we have
\begin{align*}
&\EE\!\left[
Z_t(\vartheta_{t,k};\widetilde x_{j_{t,k,\ell}},
a^{\mathrm{val}}_{t,k,\ell,1:H})\mid\cH_t\right]\\
=&\frac1m\sum_{j=1}^{m}
\sum_{a_{1:H}\in\cA(\widetilde x_j)}
\pi_{\mathrm{stu},\vartheta_{t,k}}(a_{1:H}\mid\widetilde x_j)
\lambda\log
\frac{\pi_{\mathrm{stu},\vartheta_{t,k}}(a_{1:H}\mid\widetilde x_j)}
{\pi_{w_t}(a_{1:H}\mid\widetilde x_j)}=C_t(\vartheta_{t,k}).
\end{align*}
By Lemma~\ref{lem:itercal-z}, the squared
summand of $Z_t(\vartheta_{t,k};\widetilde x_{j_{t,k,\ell}},
a^{\mathrm{val}}_{t,k,\ell,1:H}),\ell=1,\cdots,q_t$ is at most $16\lambda^2B^2H^2$. Since the samples within a
batch are independent and centered after subtracting their mean,
the cross terms vanish when expanding the squared sample-mean error and we have
\begin{align*}
\EE[(\widehat C_{t,k}-C_t(\vartheta_{t,k}))^2\mid\cH_t]
&=\frac1{q_t^2}\sum_{\ell=1}^{q_t}
\Var\!\left(
Z_t(\vartheta_{t,k};\widetilde x_{j_{t,k,\ell}},
a^{\mathrm{val}}_{t,k,\ell,1:H})\mid\cH_t\right)\le\frac{16\lambda^2B^2H^2}{q_t}.
\end{align*}
Then, we apply the Cauchy-Schwarz inequality and get
\[
\EE[|\widehat C_{t,k}-C_t(\vartheta_{t,k})|\mid\cH_t]
\le
\left(\EE[(\widehat C_{t,k}-C_t(\vartheta_{t,k}))^2
\mid\cH_t]\right)^{1/2}
\le\frac{4\lambda BH}{\sqrt{q_t}}.
\]
The maximum of two nonnegative numbers is no larger than their
sum. Consequently, we have
\[
\EE[\xi_t\mid\cH_t]
\le2\sum_{k=1}^{2}
\EE[|\widehat C_{t,k}-C_t(\vartheta_{t,k})|\mid\cH_t]
\le\frac{16\lambda BH}{\sqrt{q_t}}
=\frac{16\lambda BH}{\sqrt{t+1}}.
\]

Choose an index $k_t^\star$ that minimizes the two true costs, only
for this proof. The definition of $\xi_t$ and empirical minimization
implies that
\begin{align}\label{eq:selection}
C_t(\theta_t)
&=C_t(\vartheta_{t,\widehat k_t})
\le\widehat C_{t,\widehat k_t}+\xi_t/2\le\widehat C_{t,k_t^\star}+\xi_t/2
\le C_t(\vartheta_{t,k_t^\star})+\xi_t.
\end{align}
The first and last inequalities bound the evaluation errors; the
middle inequality is the actual two-way selection rule.
This proves \eqref{eq:itercal-selection} and we finish the proof.
\end{proof}

\begin{proof}[Proof of Lemma~\ref{lem:itercal-hybrid}]
Recall that $F=C_{w^\star_\lambda}$ and
$\Delta_{\lambda,m}(\theta)=F(\theta)-F(\theta^\dagger_{\lambda,m})$.
By optimality, the student Lipschitz bound in
\eqref{eq:itercal-stu-lip}, and the diameter $2B$ of $\Theta$,
\begin{equation}
0\le\Delta_{\lambda,m}(\theta)
\le4\lambda BH^{3/2}\|\theta-\theta^\dagger_{\lambda,m}\|_2
\le8\lambda B^2H^{3/2}=M_{\mathrm{opt}}.
\label{eq:itercal-true-gap-range}
\end{equation}
We first show that the exact projected step gives a quadratic
improvement when this gap is small, and then use uniform exploration
for the remaining values of the gap.

To apply Lemma~\ref{lem:itercal-subgradient-lojas}, we define the
extended-real function and the normal cone
\[
\Psi(\theta):=
\begin{cases}
F(\theta),&\theta\in\Theta,\\
+\infty,&\theta\notin\Theta,
\end{cases}
\]
\[
N_\Theta(\theta):=
\{v\in\RR^d:v^\top(\vartheta-\theta)\le0
\text{ for every }\vartheta\in\Theta\},\ \theta\in\Theta.
\]
For this smooth function on a closed convex set, We claim that
\begin{equation}
\partial\Psi(\theta)
=\nabla_\theta F(\theta)+N_\Theta(\theta),
\ \theta\in\Theta.
\label{eq:itercal-constrained-subgradient}
\end{equation}
For $\theta,\vartheta\in\Theta$, we first have $F(\vartheta)-F(\theta)
=\nabla_\theta F(\theta)^\top(\vartheta-\theta)
+o(\|\vartheta-\theta\|_2)$.

Hence, by the definition of the Fr\'echet subdifferential, we have
\begin{align*}
v\in\widehat\partial\Psi(\theta)
&\iff
\liminf_{\substack{\vartheta\to\theta,\ \vartheta\in\Theta\\
                  \vartheta\ne\theta}}
\frac{F(\vartheta)-F(\theta)-v^\top(\vartheta-\theta)}
{\|\vartheta-\theta\|_2}\ge0\iff
\liminf_{\substack{\vartheta\to\theta,\ \vartheta\in\Theta\\
                  \vartheta\ne\theta}}
\frac{(\nabla_\theta F(\theta)-v)^\top(\vartheta-\theta)}
{\|\vartheta-\theta\|_2}\ge0.
\end{align*}
On one hand, by the definition of $N_\Theta(\theta)$, we know that
\begin{align*}
v-\nabla_\theta F(\theta)\in N_\Theta(\theta)
&\implies
(\nabla_\theta F(\theta)-v)^\top(\vartheta-\theta)\ge0
\ \forall \vartheta\in\Theta\implies v\in\widehat\partial\Psi(\theta).
\end{align*}
Conversely, let $v\in\widehat\partial\Psi(\theta)$.
$\forall\zeta\in\Theta\setminus\{\theta\}$, by the convexity of $\Theta$, we have that $$\vartheta_u=\theta+u(\zeta-\theta)\in\Theta,
\ u\in(0,1].$$
Therefore,
\begin{align*}
0
&\le\liminf_{u\downarrow0}
\frac{(\nabla_\theta F(\theta)-v)^\top(\vartheta_u-\theta)}
{\|\vartheta_u-\theta\|_2}=\liminf_{u\downarrow0}
\frac{u(\nabla_\theta F(\theta)-v)^\top(\zeta-\theta)}
{u\|\zeta-\theta\|_2}=\frac{(\nabla_\theta F(\theta)-v)^\top(\zeta-\theta)}
{\|\zeta-\theta\|_2}.
\end{align*}
Consequently, we obtain
$(v-\nabla_\theta F(\theta))^\top(\zeta-\theta)\le0
\ \forall\zeta\in\Theta$. Thus,
\[ 
v-\nabla_\theta F(\theta)\in N_\Theta(\theta).
\]
Combining both inclusions yields
$\widehat\partial\Psi(\theta)
=\nabla_\theta F(\theta)+N_\Theta(\theta)$.
By the definition of the limiting subdifferential, we know that
\[ 
v\in\partial\Psi(\theta)
\iff
\exists\,(\theta_n,v_n)_{n\ge1}:
\begin{cases}
\theta_n\in\Theta,\  \theta_n\to\theta,\\
\Psi(\theta_n)\to\Psi(\theta),\  v_n\to v,\\
v_n\in\widehat\partial\Psi(\theta_n).
\end{cases}
\]
For such a sequence and every $\zeta\in\Theta$, we have
$(v_n-\nabla_\theta F(\theta_n))^\top(\zeta-\theta_n)\le0$.
Continuity of $\nabla_\theta F$ yields that
\begin{align*}
(v-\nabla_\theta F(\theta))^\top(\zeta-\theta)
&=\lim_{n\to\infty}
(v_n-\nabla_\theta F(\theta_n))^\top(\zeta-\theta_n)\le0.
\end{align*}
Thus, we obtain
\[
\partial\Psi(\theta)
\subseteq\nabla_\theta F(\theta)+N_\Theta(\theta).
\]
For the reverse inclusion, the constant sequences
$\theta_n=\theta$ and $v_n=v$ give
\[
v\in\nabla_\theta F(\theta)+N_\Theta(\theta)
=\widehat\partial\Psi(\theta)
\implies v\in\partial\Psi(\theta).
\]
Therefore, we have proved the claim.

Lemma~\ref{lem:itercal-smoothness} shows that $F$ is analytic.
Consequently, the finite graph of $\Psi$ is
\[
\{(\theta,u)\in\RR^d\times\RR:
\|\theta\|_2^2\le B^2,\ u-F(\theta)=0\},
\]
which is semianalytic and hence subanalytic. Its domain is the
closed ball $\Theta$, and $\Psi$ is continuous on this domain.
Moreover, optimality of $\theta^\dagger_{\lambda,m}$ on every
segment in $\Theta$ implies that
\[
\nabla_\theta F(\theta^\dagger_{\lambda,m})^\top
(\vartheta-\theta^\dagger_{\lambda,m})\ge0,
\ \vartheta\in\Theta.
\]
Thus $-\nabla_\theta F(\theta^\dagger_{\lambda,m})
\in N_\Theta(\theta^\dagger_{\lambda,m})$, and
\eqref{eq:itercal-constrained-subgradient} gives
$0\in\partial\Psi(\theta^\dagger_{\lambda,m})$.
All hypotheses of Lemma~\ref{lem:itercal-subgradient-lojas} therefore
hold. There are an open neighborhood $U$ of $\theta^\dagger_{\lambda,m}$,
$c_{\mathrm{loc}}>0$, and $\rho_{\mathrm{loc}}\in[0,1)$ such that
\begin{equation}
\operatorname{dist}(0,\partial\Psi(\theta))
\ge c_{\mathrm{loc}}[\Delta_{\lambda,m}(\theta)]^{\rho_{\mathrm{loc}}},\ 
\theta\in U\cap\Theta,\ \Delta_{\lambda,m}(\theta)>0.
\label{eq:itercal-local-subgradient}
\end{equation}

Uniqueness of the minimizer now lets us choose
$0<\delta_{\mathrm{loc}}\le\min\{1,M_{\mathrm{opt}}\}$ such that
\[
\{\theta\in\Theta:\Delta_{\lambda,m}(\theta)
\le\delta_{\mathrm{loc}}\}\subset U.
\]
To justify this choice, if $\Theta\setminus U$ is nonempty,
its compactness and continuity of $\Delta_{\lambda,m}$ give an
attained minimum there. This minimum is strictly positive because
the only zero is $\theta^\dagger_{\lambda,m}\in U$.
Choose $\delta_{\mathrm{loc}}$ smaller than that minimum.
If the complement is empty, any positive value satisfying the
displayed upper bound suffices. For
$0<\Delta_{\lambda,m}(\theta)\le\delta_{\mathrm{loc}}\le1$,
the fact that $\rho_{\mathrm{loc}}<1$ yields
\[
[\Delta_{\lambda,m}(\theta)]^{\rho_{\mathrm{loc}}}
=\Delta_{\lambda,m}(\theta)
[\Delta_{\lambda,m}(\theta)]^{\rho_{\mathrm{loc}}-1}
\ge\Delta_{\lambda,m}(\theta).
\]
Hence \eqref{eq:itercal-local-subgradient} implies the weaker but
sufficient linear bound
\begin{equation}
\operatorname{dist}(0,\partial\Psi(\theta))
\ge c_{\mathrm{loc}}\Delta_{\lambda,m}(\theta)
,\ 0<\Delta_{\lambda,m}(\theta)\le\delta_{\mathrm{loc}}.
\label{eq:itercal-local-linear-slope}
\end{equation}

We apply this bound to the exact projected step $y(\theta)$
defined in the lemma. The first-order condition for Euclidean
projection is
\[
\bigl(\theta-\alpha_{\mathrm{stu}}\nabla_\theta F(\theta)
-y(\theta)\bigr)^\top(\vartheta-y(\theta))\le0
\ \forall \vartheta\in\Theta.
\]
Setting $\vartheta=\theta$ and rearranging gives
\[
\nabla_\theta F(\theta)^\top(y(\theta)-\theta)
\le-\frac{\|y(\theta)-\theta\|_2^2}{\alpha_{\mathrm{stu}}}.
\]
The gradient Lipschitz bound in
\eqref{eq:itercal-gradient-smoothness}, integrated along the segment
from $\theta$ to $y(\theta)$, gives
\begin{align*}
F(y(\theta))-F(\theta)
&=\nabla_\theta F(\theta)^\top(y(\theta)-\theta) +\int_0^1
\bigl[\nabla_\theta F(\theta+u(y(\theta)-\theta))
-\nabla_\theta F(\theta)\bigr]^\top(y(\theta)-\theta)\,du\\
&\le-\frac{\|y(\theta)-\theta\|_2^2}{\alpha_{\mathrm{stu}}}
+\int_0^1 uL_{\mathrm{st}}\|y(\theta)-\theta\|_2^2\,du\\
&=-\left(\frac1{\alpha_{\mathrm{stu}}}
-\frac{L_{\mathrm{st}}}{2}\right)
\|y(\theta)-\theta\|_2^2
\le-\frac{\|y(\theta)-\theta\|_2^2}{2\alpha_{\mathrm{stu}}}.
\end{align*}
The last inequality uses
$\alpha_{\mathrm{stu}}=1/(2L_{\mathrm{st}})$.
In particular, $F(y(\theta))\le F(\theta)$.
The projection condition also implies
\[
\frac{\theta-y(\theta)}{\alpha_{\mathrm{stu}}}
-\nabla_\theta F(\theta)\in N_\Theta(y(\theta)).
\]
Using \eqref{eq:itercal-constrained-subgradient} at $y(\theta)$,
and then \eqref{eq:itercal-gradient-smoothness}, we obtain
\begin{align*}
\operatorname{dist}(0,\partial\Psi(y(\theta)))
&\le\left\|
\nabla_\theta F(y(\theta))-\nabla_\theta F(\theta)
+\frac{\theta-y(\theta)}{\alpha_{\mathrm{stu}}}
\right\|_2\le\left(L_{\mathrm{st}}+\frac1{\alpha_{\mathrm{stu}}}\right)
\|y(\theta)-\theta\|_2.
\end{align*}

Suppose now that $0<\Delta_{\lambda,m}(\theta)
\le\delta_{\mathrm{loc}}$. If
$\Delta_{\lambda,m}(y(\theta))\le\Delta_{\lambda,m}(\theta)/2$,
then \eqref{eq:itercal-true-gap-range} gives
\[
F(\theta)-F(y(\theta))
=\Delta_{\lambda,m}(\theta)-\Delta_{\lambda,m}(y(\theta))
\ge\frac{\Delta_{\lambda,m}(\theta)}2
\ge\frac{[\Delta_{\lambda,m}(\theta)]^2}{2M_{\mathrm{opt}}}.
\]
Otherwise,
$0<\Delta_{\lambda,m}(\theta)/2
<\Delta_{\lambda,m}(y(\theta))\le\delta_{\mathrm{loc}}$,
so \eqref{eq:itercal-local-linear-slope} applies at $y(\theta)$.
Combining the preceding two estimates with that inequality gives
\begin{align*}
F(\theta)-F(y(\theta))
&\ge\frac{\|y(\theta)-\theta\|_2^2}{2\alpha_{\mathrm{stu}}}\ge\frac{\operatorname{dist}(0,\partial\Psi(y(\theta)))^2}
{2\alpha_{\mathrm{stu}}(L_{\mathrm{st}}+1/\alpha_{\mathrm{stu}})^2}\\
&\ge\frac{c_{\mathrm{loc}}^2
[\Delta_{\lambda,m}(y(\theta))]^2}
{2\alpha_{\mathrm{stu}}(L_{\mathrm{st}}+1/\alpha_{\mathrm{stu}})^2}\ge\frac{c_{\mathrm{loc}}^2}
{8\alpha_{\mathrm{stu}}(L_{\mathrm{st}}+1/\alpha_{\mathrm{stu}})^2}
[\Delta_{\lambda,m}(\theta)]^2.
\end{align*}
The same quadratic lower bound, with the smaller of the two
coefficients, therefore holds throughout the small sublevel set.
When the gap is zero, the required bound is immediate.

It remains to control $\Delta_{\lambda,m}(\theta)
>\delta_{\mathrm{loc}}$. We establish the needed exploration
probability for a calibrated objective. Fix $w\in W$ and
$0<u\le M_{\mathrm{opt}}$.
By continuity and compactness, we can choose
\[
\vartheta_w^\star\in\argmin_{\vartheta\in\Theta}C_w(\vartheta).
\]
Set
$\tau=u/M_{\mathrm{opt}}\in(0,1]$ and consider $\{(1-\tau)\vartheta_w^\star+\tau\zeta:\zeta\in\Theta\}$. By convexity, we know that this set lies in $\Theta$. Each of its points
$\vartheta$ has distance at most $2B\tau$ from
$\vartheta_w^\star$, so \eqref{eq:itercal-stu-lip} gives
\[
C_w(\vartheta)-\min_{\zeta\in\Theta}C_w(\zeta)
\le(4\lambda BH^{3/2})(2B\tau)
=M_{\mathrm{opt}}\tau=u.
\]
The affine map defining the set scales $d$-dimensional volume by
$\tau^d$. Since
$\vartheta^{\mathrm{unif}}$ is uniform on $\Theta$, we conclude that
\begin{equation}
\Pr\!\left(
C_w(\vartheta^{\mathrm{unif}})-\min_{\vartheta\in\Theta}C_w(\vartheta)
\le u\right)
\ge\left(\frac{u}{M_{\mathrm{opt}}}\right)^d,
\  0<u\le M_{\mathrm{opt}}.
\label{eq:itercal-volume-hit}
\end{equation}
Apply this bound with $w=w^\star_\lambda$ and
$u=\delta_{\mathrm{loc}}/2$ to obtain
\[
\Pr\!\left(
\Delta_{\lambda,m}(\vartheta^{\mathrm{unif}})
\le\frac{\delta_{\mathrm{loc}}}{2}\right)
\ge\left(\frac{\delta_{\mathrm{loc}}}{2M_{\mathrm{opt}}}\right)^d.
\]
On this event and when
$\Delta_{\lambda,m}(\theta)>\delta_{\mathrm{loc}}$, we have
$F(\theta)-F(\vartheta^{\mathrm{unif}})
\ge\Delta_{\lambda,m}(\theta)-\delta_{\mathrm{loc}}/2
\ge\Delta_{\lambda,m}(\theta)/2$. Since
$F(y(\theta))\le F(\theta)$, the two-candidate gain is nonnegative
and at least $F(\theta)-F(\vartheta^{\mathrm{unif}})$ for every draw.
It therefore dominates the positive part below:
\begin{align*}
\EE_{\vartheta^{\mathrm{unif}}}\!\left[
F(\theta)-\min\{F(y(\theta)),F(\vartheta^{\mathrm{unif}})\}
\right]&\ge
\EE_{\vartheta^{\mathrm{unif}}}\!\left[
[F(\theta)-F(\vartheta^{\mathrm{unif}})]_+\right]\\
&\ge\frac12
\left(\frac{\delta_{\mathrm{loc}}}{2M_{\mathrm{opt}}}\right)^d
\Delta_{\lambda,m}(\theta)\ge\frac1{2M_{\mathrm{opt}}}
\left(\frac{\delta_{\mathrm{loc}}}{2M_{\mathrm{opt}}}\right)^d
[\Delta_{\lambda,m}(\theta)]^2.
\end{align*}
The last inequality uses \eqref{eq:itercal-true-gap-range}.
For the small sublevel set, the minimum of the two costs is at
most $F(y(\theta))$, so the projected-step bounds already proved
apply directly to the same left-hand side. Taking
\[
\kappa_{\mathrm{opt}}
:=\min\left\{
\frac1{4M_{\mathrm{opt}}},
\frac{c_{\mathrm{loc}}^2}
{8\alpha_{\mathrm{stu}}(L_{\mathrm{st}}+1/\alpha_{\mathrm{stu}})^2},
\frac1{2M_{\mathrm{opt}}}
\left(\frac{\delta_{\mathrm{loc}}}{2M_{\mathrm{opt}}}\right)^d
\right\}>0
\]
therefore proves \eqref{eq:itercal-hybrid-progress} for every
$\theta\in\Theta$. All constants were chosen using the fixed
objective and parameter set, and hence are independent of $t$.
\end{proof}

\begin{proof}[Proof of Lemma \ref{lem:itercal-optimization}]
We study the fixed true cost
$F(\theta):=C_{w^\star_\lambda}(\theta)$, so that
$\Delta_{\lambda,m}(\theta)=F(\theta)-F(\theta^\dagger_{\lambda,m})$.
The algorithm evaluates $C_t$, whereas $F$ is used only in this
proof. Recall from \eqref{eq:itercal-true-gap-range} that
$0\le\Delta_{\lambda,m}(\theta)\le M_{\mathrm{opt}}$.
For target update $t\ge1$, let
$\cG_t:=\cF_{t-1}\vee\sigma(w_t)$.
The current student $\theta_{t-1}$ is fixed conditional on $\cG_t$.
Compare the actual gradient candidate with the population candidate
\[
\begin{aligned}
\vartheta_{t,1}
&=\Proj{\Theta}{\theta_{t-1}-\alpha_{\mathrm{stu}}\widehat g_t^{\mathrm{stu}}},\\
y(\theta_{t-1})
&=\Proj{\Theta}{\theta_{t-1}-\alpha_{\mathrm{stu}}\nabla_\theta F(\theta_{t-1})}.
\end{aligned}
\]
The map $y$ is the proof-only update in Lemma~\ref{lem:itercal-hybrid}.

We first write the one-step error comparison, before bounding
its perturbation terms. Lemma~\ref{lem:itercal-selection} and
\eqref{eq:itercal-w-lip} imply, on every sample path,
\begin{align*}
F(\theta_t)
&\le C_t(\theta_t)
+\sup_{\theta\in\Theta}|F(\theta)-C_t(\theta)|\le\min_{k\in\{1,2\}}C_t(\vartheta_{t,k})
+\xi_t+\sup_{\theta\in\Theta}|F(\theta)-C_t(\theta)|\\
&\le\min_{k\in\{1,2\}}F(\vartheta_{t,k})
+\xi_t+2\sup_{\theta\in\Theta}|F(\theta)-C_t(\theta)|\le\min\{F(\vartheta_{t,1}),F(\vartheta_{t,2})\}
+\xi_t+4\lambda H\|w_t-w^\star_\lambda\|_2.
\end{align*}
Replacing one entry of a minimum changes that minimum by at most
the absolute change in the entry. Since $F$ is
$4\lambda BH^{3/2}$-Lipschitz according to Lemma~\ref{lem:itercal-lipschitz}, we thus have
\begin{align*}
&\min\{F(\vartheta_{t,1}),F(\vartheta_{t,2})\}\le
\min\{F(y(\theta_{t-1})),F(\vartheta_{t,2})\}
+4\lambda BH^{3/2}\|\vartheta_{t,1}-y(\theta_{t-1})\|_2.
\end{align*}
Conditional on $\cG_t$, the exploration candidate
$\vartheta_{t,2}$ is uniform on $\Theta$. By Lemma \ref{lem:itercal-hybrid}, we have
\[
\begin{aligned}
&\EE\!\left[
F(\theta_{t-1})
-\min\{F(y(\theta_{t-1})),F(\vartheta_{t,2})\}
\,\middle|\,\cG_t\right]\ge
\kappa_{\mathrm{opt}}
[\Delta_{\lambda,m}(\theta_{t-1})]^2.
\end{aligned}
\]
Since $F(\theta_{t-1})$ is measurable with respect to $\cG_t$, rearrange and we have
\[
\EE[\min\{F(y(\theta_{t-1})),F(\vartheta_{t,2})\}|\,\cG_t]\le F(\theta_{t-1})-\kappa_{\mathrm{opt}}
[\Delta_{\lambda,m}(\theta_{t-1})]^2
\]
Thus, utilizing the inequalities we have proved above, we have
\begin{align}
&\EE[\Delta_{\lambda,m}(\theta_t)\mid\cG_t]\nonumber\\
=&\EE[F(\theta_t)\mid\cG_t]
-F(\theta^\dagger_{\lambda,m})\nonumber\\
\le&
\EE\!\left[
\min\{F(y(\theta_{t-1})),F(\vartheta_{t,2})\}
\,\middle|\,\cG_t\right]
-F(\theta^\dagger_{\lambda,m})+
4\lambda BH^{3/2}
\EE\!\left[
\|\vartheta_{t,1}-y(\theta_{t-1})\|_2
\,\middle|\,\cG_t\right]+
4\lambda H\|w_t-w^\star_\lambda\|_2
+\EE[\xi_t\mid\cG_t]\nonumber\\
\le&
F(\theta_{t-1})-F(\theta^\dagger_{\lambda,m})
-\kappa_{\mathrm{opt}}
[\Delta_{\lambda,m}(\theta_{t-1})]^2+
4\lambda BH^{3/2}
\EE\!\left[
\|\vartheta_{t,1}-y(\theta_{t-1})\|_2
\,\middle|\,\cG_t\right]+
4\lambda H\|w_t-w^\star_\lambda\|_2
+\EE[\xi_t\mid\cG_t].
\label{eq:itercal-true-one-step}
\end{align}
We next bound the candidate discrepancy in this recursion.
Nonexpansiveness of Euclidean projection gives
\begin{align*}
\|\vartheta_{t,1}-y(\theta_{t-1})\|_2
&\le\alpha_{\mathrm{stu}}
\|\widehat g_t^{\mathrm{stu}}-\nabla_\theta F(\theta_{t-1})\|_2\le\alpha_{\mathrm{stu}}\left[
\|\widehat g_t^{\mathrm{stu}}-\nabla_\theta C_t(\theta_{t-1})\|_2
+\|\nabla_\theta C_t(\theta_{t-1})
-\nabla_\theta F(\theta_{t-1})\|_2
\right].
\end{align*}
Conditional Cauchy-Schwarz and
Lemma~\ref{lem:itercal-unbiased} bound the first term by
\[
\begin{aligned}
\EE\!\left[
\|\widehat g_t^{\mathrm{stu}}-\nabla_\theta C_t(\theta_{t-1})\|_2
\,\middle|\,\cG_t\right]
&\le
\left(\EE\!\left[
\|\widehat g_t^{\mathrm{stu}}-\nabla_\theta C_t(\theta_{t-1})\|_2^2
\,\middle|\,\cG_t\right]\right)^{1/2}\le\frac{4\lambda BH^{3/2}}{\sqrt{b_t}}.
\end{aligned}
\]
Lemma~\ref{lem:itercal-smoothness} bounds the second term by
$2\lambda H^{3/2}\|w_t-w^\star_\lambda\|_2$.
Multiplying by the cost Lipschitz constant therefore gives
\begin{align*}
&4\lambda BH^{3/2}
\EE[\|\vartheta_{t,1}-y(\theta_{t-1})\|_2\mid\cG_t]\le
\frac{16\alpha_{\mathrm{stu}}\lambda^2B^2H^3}{\sqrt{b_t}}
+8\alpha_{\mathrm{stu}}\lambda^2BH^3
\|w_t-w^\star_\lambda\|_2.
\end{align*}
For validation, $\cG_t\subseteq\cH_t$ in
Lemma~\ref{lem:itercal-selection}. The tower property yields
\[
\EE[\xi_t\mid\cG_t]
=\EE[\EE[\xi_t\mid\cH_t]\mid\cG_t]
\le\frac{16\lambda BH}{\sqrt{t+1}}.
\]
To combine these bounds, write
$e_t:=\EE[\Delta_{\lambda,m}(\theta_t)]$.
Taking total expectations in \eqref{eq:itercal-true-one-step}, using
$b_t=t+1$, and substituting the two perturbation bounds gives
\begin{align*}
e_t
&\le e_{t-1}
-\kappa_{\mathrm{opt}}
\EE\!\left[[\Delta_{\lambda,m}(\theta_{t-1})]^2\right]
+\frac{16\alpha_{\mathrm{stu}}\lambda^2B^2H^3}{\sqrt{t+1}}+
(8\alpha_{\mathrm{stu}}\lambda^2BH^3+4\lambda H)
\EE[\|w_t-w^\star_\lambda\|_2]
+\frac{16\lambda BH}{\sqrt{t+1}}.
\end{align*}
Applying the calibration bound \eqref{eq:itercal-source-mse} and
the Cauchy-Schwarz inequality, we have
\[
\EE[\|w_t-w^\star_\lambda\|_2]
\le\left(\EE[\|w_t-w^\star_\lambda\|_2^2]\right)^{1/2}
\le\frac2{\gamma\sqrt{t+2}}
\le\frac2{\gamma\sqrt{t+1}}.
\]
Also,
$\EE[[\Delta_{\lambda,m}(\theta_{t-1})]^2]\ge e_{t-1}^2$
by Jensen's inequality.
Consequently,
\begin{equation}
e_t\le e_{t-1}-\kappa_{\mathrm{opt}}e_{t-1}^2
+\frac{E_{\mathrm{opt}}}{\sqrt{t+1}},
\  t\ge1.
\label{eq:itercal-gradient-recursion}
\end{equation}
We solve \eqref{eq:itercal-gradient-recursion} by induction with the
deterministic bound
$v_t:=K_{\mathrm{opt}}(t+1)^{-1/4}$.
Initially, we have $e_0\le M_{\mathrm{opt}}\le K_{\mathrm{opt}}=v_0$.

Suppose $e_t\le v_t$. If $v_{t+1}\ge M_{\mathrm{opt}}$, the
bound $e_{t+1}\le M_{\mathrm{opt}}$ proves the next step.
Otherwise, we have
\[
v_t=v_{t+1}\left(\frac{t+2}{t+1}\right)^{1/4}
<2^{1/4}M_{\mathrm{opt}}<2M_{\mathrm{opt}}.
\]
Since $\kappa_{\mathrm{opt}}\le1/(4M_{\mathrm{opt}})$,
the derivative of $u-\kappa_{\mathrm{opt}}u^2$ on
$[0,2M_{\mathrm{opt}}]$ is
$1-2\kappa_{\mathrm{opt}}u\ge0$.
Apply this monotonicity to the recursion at time $t+1$ and we have
\begin{align*}
e_{t+1}
&\le v_t-\kappa_{\mathrm{opt}}v_t^2
+\frac{E_{\mathrm{opt}}}{\sqrt{t+2}}\le v_t-
\frac{\kappa_{\mathrm{opt}}K_{\mathrm{opt}}^2-E_{\mathrm{opt}}}
{\sqrt{t+1}}.
\end{align*}
By the definition of $K_{\mathrm{opt}}$, we have $\kappa_{\mathrm{opt}}K_{\mathrm{opt}}^2-E_{\mathrm{opt}}
\ge\frac{\kappa_{\mathrm{opt}}K_{\mathrm{opt}}^2}{2}
\ge\frac{K_{\mathrm{opt}}}{4}$. The first inequality uses
$K_{\mathrm{opt}}^2\ge2E_{\mathrm{opt}}/\kappa_{\mathrm{opt}}$;
the second uses $K_{\mathrm{opt}}\ge1/(2\kappa_{\mathrm{opt}})$.
On the other hand, integration of the derivative of $u^{-1/4}$
gives
\[
v_t-v_{t+1}
=\frac{K_{\mathrm{opt}}}{4}\int_{t+1}^{t+2}u^{-5/4}\,du
\le\frac{K_{\mathrm{opt}}}{4(t+1)^{5/4}}
\le\frac{K_{\mathrm{opt}}}{4\sqrt{t+1}}.
\]
Combining the last three displays yields $e_{t+1}\le v_{t+1}$.
By induction, we prove the first bound in
\eqref{eq:itercal-optimization-bound} for every $T\ge1$.

Finally, we translate the true-cost bound to the calibrated
optimization error. By the elementary inequality
$|\min f-\min g|\le\sup|f-g|$, applied on $\Theta$, we have
\begin{align}
|\varepsilon_T-\Delta_{\lambda,m}(\theta_T)|
&=\left|C_T(\theta_T)-F(\theta_T)
+\min_{\theta\in\Theta}F(\theta)-\min_{\theta\in\Theta}C_T(\theta)\right|
\nonumber\\
&\le2\sup_{\theta\in\Theta}|C_T(\theta)-F(\theta)|
\le4\lambda H\|w_T-w^\star_\lambda\|_2.
\label{eq:itercal-error-transfer}
\end{align}
Taking expectations and applying \eqref{eq:itercal-source-mse}
proves the second bound in \eqref{eq:itercal-optimization-bound}.
 We finish the proof.
\end{proof}
\begin{proof}[Proof of Lemma \ref{lem:itercal-lojas}]
We will apply the nonnegative form of
Lemma~\ref{lem:itercal-compact-lojas} with
\[
K=\Theta,\ 
f(\theta)=\Delta_{\lambda,m}(\theta),\ 
 g(\theta)=\cK_{\lambda,m}(\theta).
\]
Accordingly, we verify that the two functions are nonnegative,
that their graphs are subanalytic and compact, and that
$\Delta_{\lambda,m}(\theta)=0$ implies
$\cK_{\lambda,m}(\theta)=0$. We also establish an upper bound
on $\cK_{\lambda,m}$ to use in the theorem's exponent normalization.
Their definitions are recalled in the lemma statement.

\medskip\noindent
\emph{Nonnegativity.}
By the optimality of $\theta^\dagger_{\lambda,m}$ and the
nonnegativity of KL divergence, respectively, we have
\[
\Delta_{\lambda,m}(\theta)\ge0,
\ \cK_{\lambda,m}(\theta)\ge0,
\ \theta\in\Theta.
\]
\emph{Analyticity.}
Lemma~\ref{lem:itercal-smoothness} proves analyticity for the average
KL against any fixed positive answer laws. Apply that result
with $Q_j=\pi_{w^\star_\lambda}(\cdot\mid\widetilde x_j)$
and with
$Q_j=\pi_{\mathrm{stu},\theta^\dagger_{\lambda,m}}
(\cdot\mid\widetilde x_j)$, respectively.
Thus $C_{w^\star_\lambda}$ and $\cK_{\lambda,m}$ are analytic.
Subtracting the constant
$C_{w^\star_\lambda}(\theta^\dagger_{\lambda,m})$
shows that $\Delta_{\lambda,m}$ is analytic.

\emph{Compact subanalytic graphs.}
For $F\in\{\Delta_{\lambda,m},\cK_{\lambda,m}\}$, the restricted
graph is
\[
\operatorname{graph}(F|_\Theta)
=\{(\theta,u)\in\RR^{d+1}:
B^2-\|\theta\|_2^2\ge0,\ u-F(\theta)=0\}.
\]
The defining inequality is polynomial and the equality is analytic
by Lemma~\ref{lem:itercal-smoothness}. Hence the graph is semianalytic,
and therefore subanalytic.
It is also compact because it is the image of the compact ball $\Theta$
under the continuous map $\theta\mapsto(\theta,F(\theta))$.

\emph{Zero-set inclusion.}
If $\Delta_{\lambda,m}(\theta)=0$, its definition gives
\[
C_{w^\star_\lambda}(\theta)
=C_{w^\star_\lambda}(\theta^\dagger_{\lambda,m})
=\min_{\vartheta\in\Theta}C_{w^\star_\lambda}(\vartheta).
\]
Thus $\theta$ is also a global minimizer. By
Assumption~\ref{ass:itercal-unique}, we have
\[
\pi_{\mathrm{stu},\theta}(\cdot\mid\widetilde x_j)
=\pi_{\mathrm{stu},\theta^\dagger_{\lambda,m}}
(\cdot\mid\widetilde x_j),\ j=1,\ldots,m.
\]
Each KL summand defining $\cK_{\lambda,m}(\theta)$ is consequently
the divergence of a law from itself, hence zero. We obtain the
required inclusion
\[
\{\theta\in\Theta:\Delta_{\lambda,m}(\theta)=0\}
\subseteq\{\theta\in\Theta:\cK_{\lambda,m}(\theta)=0\}.
\]

\medskip\noindent
\emph{A uniform bound for the KL divergence}
Apply the student probability envelope
\eqref{eq:itercal-student-envelope} from Lemma~\ref{lem:itercal-z}
to $\theta$ and $\theta^\dagger_{\lambda,m}$.
Dividing the lower bound for the numerator by the upper bound
for the denominator, and conversely, gives
\[
e^{-4B}
\le\frac{\pi_{\mathrm{stu},\theta}(a\mid s)}
{\pi_{\mathrm{stu},\theta^\dagger_{\lambda,m}}(a\mid s)}
\le e^{4B}
\ (a\in\cB(s)).
\]
Taking logarithms gives an absolute token log-ratio bound of $4B$.
For $a_{1:H}\in\cA(\widetilde x_j)$, put
$s_{j,h}=(\widetilde x_j,a_{1:h-1})$.
By ~\eqref{eq:itercal-chain} and the triangle inequality, we have
\begin{align*}
\left|\log
\frac{\pi_{\mathrm{stu},\theta}(a_{1:H}\mid\widetilde x_j)}
{\pi_{\mathrm{stu},\theta^\dagger_{\lambda,m}}
(a_{1:H}\mid\widetilde x_j)}\right|
&=\left|\sum_{h=1}^{H}\log
\frac{\pi_{\mathrm{stu},\theta}(a_h\mid s_{j,h})}
{\pi_{\mathrm{stu},\theta^\dagger_{\lambda,m}}(a_h\mid s_{j,h})}
\right|\le\sum_{h=1}^{H}\left|\log
\frac{\pi_{\mathrm{stu},\theta}(a_h\mid s_{j,h})}
{\pi_{\mathrm{stu},\theta^\dagger_{\lambda,m}}(a_h\mid s_{j,h})}
\right|\le4BH.
\end{align*}
Using the expectation form of KL in the definition of
$\cK_{\lambda,m}$, this pointwise bound gives
\[
0\le\cK_{\lambda,m}(\theta)
\le\frac1m\sum_{j=1}^{m}
\EE_{a_{1:H}\sim\pi_{\mathrm{stu},\theta}(\cdot\mid\widetilde x_j)}
[4BH]
=4BH,
\]
Via the steps above, we verify the nonnegativity, compact subanalytic graphs,
and zero-set inclusion required by
Lemma~\ref{lem:itercal-compact-lojas}. Moreover, the last step provides the bound
$\cK_{\lambda,m}(\theta)\le4BH\le M_{\lambda,m}$ with
$M_{\lambda,m}:=\max\{1,4BH\}$.

Since the functions and their domain are fixed independently
of the iteration number, applying Lemma \ref{lem:itercal-compact-lojas} and we get that there exist
$a_{\lambda,m}>0$ and $p_{\lambda,m}\ge1$ such that
\[
\Delta_{\lambda,m}(\theta)
\ge a_{\lambda,m}[\cK_{\lambda,m}(\theta)]^{p_{\lambda,m}}
\ \theta\in\Theta,
\]
We finish the proof.
\end{proof}

\section{Proofs in Section \ref{sec:sm-temperature}}
\label{app:sm-details}

This appendix proves Proposition~\ref{prop:sm-temp-oracle} and
Theorem~\ref{thm:sm-temperature}. We first verify the separation construction,
then characterize the population and finite-iteration behavior of direct
teacher matching.

\begin{proof}[Proof of Proposition~\ref{prop:sm-temp-oracle}]
We verify realizability and source identification, derive the student oracle,
and then compare it with the teacher.

\paragraph{Feasibility and shared realizability.}
Equations~\eqref{eq:sm-temp-covariates}--\eqref{eq:sm-temp-features} give
\[
\max_x\|v(x)\|_2\le1,\qquad
\|w^\star_1\|_2=2<B,\qquad
\|w_{\mathrm{tea}}\|_2=3/2<B.
\]
Because $\phi(x,1)=v(x)$ and $\phi(x,0)=0$, a calibration parameter's
binary logit is its inner product with $v(x)$. The true target logits are
\[
(1,1,-1,1)\quad\text{and}\quad(1,1,-1,-1)
\]
in groups 1 and 2, respectively, and every source logit equals $1$.
These are exactly the Gibbs logits $R(x,1)-R(x,0)$ for $\lambda=1$ and
the uniform reference. Lemma~\ref{lem:itercal-gibbs} therefore verifies
Assumption~\ref{ass:itercal-truth}. Both target groups contain both labels,
so a student using one probability per group cannot represent the
unrestricted optimum.

\paragraph{Source identification.}
Let $v_j:=v(x_j)$. For $H=1$ and two legal tokens, the analytical
legal-token average in~\eqref{eq:itercal-gram} contributes
$\frac12v_jv_j^\top$ at source prompt $x_j$, independently of the teacher
draw. Hence
\begin{equation}
G_{\mathrm{joint}}
=\frac18\sum_{j=1}^4v_jv_j^\top
=\frac1{16}
\begin{pmatrix}
3&1&0&1/\sqrt2\\
1&1&0&0\\
0&0&1&0\\
1/\sqrt2&0&0&1/2
\end{pmatrix}.
\label{eq:sm-temp-joint-exact}
\end{equation}
The source vectors span $\RR^4$, so this matrix is positive definite.
More explicitly, its determinant is $2^{-17}$ and its trace is $11/32$;
therefore
$\lambda_{\min}(G_{\mathrm{joint}})
\ge\det(G_{\mathrm{joint}})/\operatorname{tr}(G_{\mathrm{joint}})^3
=1/5324>0$.
This proves Assumption~\ref{ass:itercal-identification}.
Lemma~\ref{lem:itercal-calibration} further gives
$G_{\mathrm{src}}(\theta_t)\succeq e^{-2B}G_{\mathrm{joint}}$ for the
actual current-student alternatives. The source student features have norm
at most one.

\paragraph{Student oracle and the exact KL identity.}
Define
\[
A(u):=\log\frac{1+e^u}{2},\qquad
K_0(u):=u\sigma(u)-A(u).
\]
Direct substitution into Bernoulli KL gives
\begin{equation}
K_0(u)=\KL(\Ber(\sigma(u))\|\Ber(1/2)),\qquad
A'(u)=\sigma(u),\qquad K_0'(u)=u\sigma'(u).
\label{eq:sm-temp-reference-cost}
\end{equation}
For a group with positive-label proportion $r$, a student with logit $u$
has
\[
\overline R_r(u)=1-r+(2r-1)\sigma(u),\qquad
J_r(u)=1-r+(2r-1)\sigma(u)-K_0(u).
\]
Since $J_r'(u)=\sigma'(u)[(2r-1)-u]$, its unique maximizer is $2r-1$.
The two target proportions are $3/4$ and $1/2$, so
\[
J(\theta)=\frac12[J_{3/4}(\theta_1)+J_{1/2}(\theta_2)],\qquad
\theta^\dagger_{1,8}=(1/2,0)^\top.
\]
This parameter lies in $\operatorname{int}(\Theta)$, proving uniqueness
and~\eqref{eq:sm-temp-oracle-law}. Moreover,
$\overline R(\theta)=3/8+\sigma(\theta_1)/4$.

For finite logits $u,v$, let
\begin{equation}
D(u\|v):=\KL(\Ber(\sigma(u))\|\Ber(\sigma(v)))
=(u-v)\sigma(u)-A(u)+A(v).
\label{eq:sm-temp-bernoulli-kl}
\end{equation}
Putting $v=2r-1$ gives $J_r(v)-J_r(u)=D(u\|v)$. Thus
\begin{align}
J(\theta^\dagger_{1,8})-J(\theta)
&=\frac12[D(\theta_1\|1/2)+D(\theta_2\|0)]
=\cK(\theta)=\Delta_{1,8}(\theta).
\label{eq:sm-temp-exact-gap}
\end{align}
This proves~\eqref{eq:sm-temp-return}.

\paragraph{Teacher quality.}
Let $\tau_{i,k}:=w_{\mathrm{tea}}^\top v(\widetilde x_{i,k})$.
Then
\begin{equation}
(\tau_{1,k})_{k=1}^4=(1,1,-1,1/4),\qquad
(\tau_{2,k})_{k=1}^4=(1,1,-1,1/8),
\label{eq:sm-temp-teacher-logits}
\end{equation}
and
\begin{equation}
\log\frac{\pi_{\mathrm{tea}}(a\mid\widetilde x_{i,k})}
{\pi_{\mathrm{pre}}(a\mid\widetilde x_{i,k})}
=\tau_{i,k}\mathbf1\{a=1\}-A(\tau_{i,k}).
\label{eq:sm-temp-teacher-ratio}
\end{equation}
The teacher's mean unregularized reward is
\[
\overline R(\pi_{\mathrm{tea}})
=\frac{6\sigma(1)+\sigma(1/4)+1-\sigma(1/8)}8>\frac58,
\]
because $\sigma(1)>2/3$ and $\sigma(1/4)>\sigma(1/8)$. Every student satisfies
$\overline R(\theta)=3/8+\sigma(\theta_1)/4<5/8$.

For regularized reward, the unrestricted optimal target value is
$J^\star=A(1)$. The teacher matches the unrestricted optimal logit on six
prompts, so
\[
J^\star-J(\pi_{\mathrm{tea}})
=\frac18[D(1/4\|1)+D(1/8\|-1)].
\]
Differentiating~\eqref{eq:sm-temp-bernoulli-kl} and integrating gives
\begin{equation}
D(u\|v)=(u-v)^2\int_0^1t\,\sigma'(v+t(u-v))\,dt.
\label{eq:sm-temp-kl-integral}
\end{equation}
Since $\sigma'\le1/4$,
\[
J^\star-J(\pi_{\mathrm{tea}})
\le\frac18\left(\frac9{128}+\frac{81}{512}\right)
=\frac{117}{4096}.
\]
Let $D_0:=D(0\|1)=D(0\|-1)$. Since $\sigma'$ decreases on
$[0,\infty)$ and $\sigma'(1)>3/16$ (using $e<3$),
\eqref{eq:sm-temp-kl-integral} gives $D_0>3/32$.
On the balanced second group, the oracle student uses logit zero against
optimal logits $1,1,-1,-1$, so
\[
J^\star-J(\theta^\dagger_{1,8})\ge D_0/2
>\frac3{64}=\frac{192}{4096}.
\]
Consequently,
\begin{equation}
J(\pi_{\mathrm{tea}})-J(\theta^\dagger_{1,8})
>\frac{75}{4096}>0,
\qquad
J(\pi_{\mathrm{tea}})>J(\theta),\quad\theta\in\Theta.
\label{eq:sm-temp-teacher-better}
\end{equation}

\paragraph{Robustness.}
Replace $r_+$ and $r_-$ by
$u_3/\sqrt2+c_+u_4$ and $-u_3/\sqrt2+c_-u_4$, and replace the teacher's
$u_4$ coefficient by $\rho$. At
$(c_+,c_-,\rho)=(1/2,1/4,1/2)$, the norm bounds for the perturbed
features and parameter have slack, the source Gram is positive definite,
the teacher advantages above are strict, and the minimum direct-matching
parameter gap is positive by
\eqref{eq:sm-temp-all-beta-gap} below. The true parameter ignores $u_4$, so
realizability and the labels are unchanged under these perturbations.
The remaining quantities are continuous in $(c_+,c_-,\rho)$, and the
minimum over $s=(1+\beta)^{-1}\in[0,1]$ is continuous by compactness. Thus positive
definiteness and all strict comparisons persist, with possibly different
constants, throughout a relative open neighborhood of the displayed
instance.
\end{proof}

\begin{proof}[Proof of Theorem~\ref{thm:sm-temperature}]
We first analyze the population objective for arbitrary shared
$\beta\ge0$, then the stochastic recursion for $\beta=1$.

\paragraph{Population minimizers for all shared regularization strengths.}
For a target prompt with teacher logit $\tau$,
\eqref{eq:sm-temp-bernoulli-kl} gives
\[
\KL(\Ber(\sigma(u))\|\Ber(\sigma(\tau)))
=K_0(u)-\tau\sigma(u)+A(\tau).
\]
The groupwise mean teacher logits are
$L=(L_1,L_2)^\top=(5/16,9/32)^\top$. Therefore
\begin{align}
C_{\mathrm{SM},\beta}(\theta)
&=\frac12\sum_{i=1}^2
[(1+\beta)K_0(\theta_i)-L_i\sigma(\theta_i)]+\text{constant},\\
\partial_i C_{\mathrm{SM},\beta}(\theta)
&=\frac12\sigma'(\theta_i)[(1+\beta)\theta_i-L_i].
\label{eq:sm-temp-population-gradient}
\end{align}
Each coordinate contribution strictly decreases and then strictly increases,
so its unique minimizer is
\[
\theta_{\mathrm{SM}}^\star(\beta)=sL,\qquad
s:=\frac1{1+\beta}\in(0,1].
\]
This point lies in $\operatorname{int}(\Theta)$. Completing the square gives
\begin{equation}
\begin{aligned}
\|\theta_{\mathrm{SM}}^\star(\beta)-\theta^\dagger_{1,8}\|_2^2
&=\left(\frac{5s}{16}-\frac12\right)^2
 +\left(\frac{9s}{32}\right)^2\\
&=\frac{181}{1024}\left(s-\frac{160}{181}\right)^2
 +\frac{81}{724}\ge\frac{81}{724}.
\end{aligned}
\label{eq:sm-temp-all-beta-gap}
\end{equation}
All coordinates on the relevant line segments lie in $[0,1/2]$.
Equations~\eqref{eq:sm-temp-exact-gap} and~\eqref{eq:sm-temp-kl-integral}
therefore imply
\[
\cK(\theta_{\mathrm{SM}}^\star(\beta))
\ge\frac{\sigma'(1/2)}4
\|\theta_{\mathrm{SM}}^\star(\beta)-\theta^\dagger_{1,8}\|_2^2
\ge\frac{81\sigma'(1/2)}{2896}
>\frac{\sigma'(1/2)}{36}.
\]
The same quantity is the true regularized-reward gap. The unregularized gap is
\[
\overline R(\theta^\dagger_{1,8})
-\overline R(\theta_{\mathrm{SM}}^\star(\beta))
=\frac14[\sigma(1/2)-\sigma(5s/16)]
\ge\frac14[\sigma(1/2)-\sigma(5/16)]>0.
\]

\paragraph{Unbiased gradient and source-independent moment bounds.}
For the remainder fix $\beta=1$ and write
$\theta_{\mathrm{SM}}^\star=(5/32,9/64)^\top$. Its distance to the oracle is
\begin{equation}
\|\theta_{\mathrm{SM}}^\star-\theta^\dagger_{1,8}\|_2
=\frac{\sqrt{565}}{64}.
\label{eq:sm-temp-gap}
\end{equation}
Let $\cF_t^{\mathrm{SM}}$ be the history before round $t$. At a target
prompt in group $i$,
\[
S_{\mathrm{stu},\theta}(x,a)
=e_i[\mathbf1\{a=1\}-\sigma(\theta_i)],\qquad
\|S_{\mathrm{stu},\theta}(x,a)\|_2\le1,\qquad
\EE_{a\sim\pi_{\mathrm{stu},\theta}}S_{\mathrm{stu},\theta}(x,a)=0.
\]
Also $\nabla_\theta Z_{\mathrm{SM}}=2S_{\mathrm{stu},\theta}$.
Differentiating the finite expectation defining $C_{\mathrm{SM}}$ gives
\[
\nabla C_{\mathrm{SM}}(\theta)
=\frac18\sum_{j=1}^8
\EE_{a\sim\pi_{\mathrm{stu},\theta}(\cdot\mid\widetilde x_j)}
[S_{\mathrm{stu},\theta}(\widetilde x_j,a)
 Z_{\mathrm{SM}}(\theta;\widetilde x_j,a)].
\]
Thus
\begin{equation}
\EE[\widehat g_t^{\mathrm{SM}}\mid\cF_t^{\mathrm{SM}}]
=\nabla C_{\mathrm{SM}}(\theta_t^{\mathrm{SM}}).
\label{eq:sm-temp-unbiased}
\end{equation}
At $\beta=1$,~\eqref{eq:sm-temp-population-gradient} gives
$\partial_iC_{\mathrm{SM}}(\theta)
=\sigma'(\theta_i)(\theta_i-(\theta_{\mathrm{SM}}^\star)_i)$.
Since $\sigma'(u)\ge\sigma'(B)=\kappa_{\mathrm{SM}}$ on $[-B,B]$,
\begin{equation}
\EE[(\theta_t^{\mathrm{SM}}-\theta_{\mathrm{SM}}^\star)^\top
\widehat g_t^{\mathrm{SM}}\mid\cF_t^{\mathrm{SM}}]
\ge\mu_{\mathrm{SM}}
\|\theta_t^{\mathrm{SM}}-\theta_{\mathrm{SM}}^\star\|_2^2.
\label{eq:sm-temp-cross-term}
\end{equation}

Writing
\[
Z_{\mathrm{SM}}(\theta;x,a)
=2\log\frac{\pi_{\mathrm{stu},\theta}(a\mid x)}
{\pi_{\mathrm{pre}}(a\mid x)}
-\log\frac{\pi_{\mathrm{tea}}(a\mid x)}
{\pi_{\mathrm{pre}}(a\mid x)},
\]
the absolute student/reference log ratio is at most $|\theta_i|$, while
\eqref{eq:sm-temp-teacher-ratio} bounds the teacher/reference log ratio by
$1$. Thus $|Z_{\mathrm{SM}}|\le2B+1\le3B=G_{\mathrm{SM}}$, and
\begin{equation}
\EE[\|\widehat g_t^{\mathrm{SM}}\|_2^2\mid\cF_t^{\mathrm{SM}}]
\le G_{\mathrm{SM}}^2.
\label{eq:sm-temp-second-moment}
\end{equation}

\paragraph{Finite-iteration parameter error.}
Projection nonexpansiveness and the last two bounds give
\begin{align}
\EE[\|\theta_{t+1}^{\mathrm{SM}}-\theta_{\mathrm{SM}}^\star\|_2^2
\mid\cF_t^{\mathrm{SM}}]
&\le(1-2\mu_{\mathrm{SM}}\eta_t^{\mathrm{SM}})
\|\theta_t^{\mathrm{SM}}-\theta_{\mathrm{SM}}^\star\|_2^2
+(\eta_t^{\mathrm{SM}})^2G_{\mathrm{SM}}^2\\
&=\frac{t}{t+2}
\|\theta_t^{\mathrm{SM}}-\theta_{\mathrm{SM}}^\star\|_2^2
+\frac{G_{\mathrm{SM}}^2}{\mu_{\mathrm{SM}}^2(t+2)^2}.
\label{eq:sm-temp-error-recursion}
\end{align}
The base case follows from $\theta_0^{\mathrm{SM}}=0$ and
$\|\theta_{\mathrm{SM}}^\star\|_2<B\le
G_{\mathrm{SM}}/\mu_{\mathrm{SM}}$. Induction yields
\begin{equation}
\EE\|\theta_T^{\mathrm{SM}}-\theta_{\mathrm{SM}}^\star\|_2^2
\le\frac{G_{\mathrm{SM}}^2}{\mu_{\mathrm{SM}}^2(T+1)}.
\label{eq:sm-temp-mse}
\end{equation}

\paragraph{Finite-iteration KL and regularized-reward separation.}
For arbitrary $\theta\in\Theta$,~\eqref{eq:sm-temp-kl-integral} gives
\begin{equation}
\cK(\theta)
\ge\frac{\kappa_{\mathrm{SM}}}{4}
\|\theta-\theta^\dagger_{1,8}\|_2^2.
\label{eq:sm-temp-kl-curvature}
\end{equation}
The triangle inequality and Cauchy--Schwarz imply
\[
\frac{\sqrt{565}}{64}
\le\left(\EE\|\theta_T^{\mathrm{SM}}-\theta^\dagger_{1,8}\|_2^2\right)^{1/2}
+\left(\EE\|\theta_T^{\mathrm{SM}}-\theta_{\mathrm{SM}}^\star\|_2^2\right)^{1/2}.
\]
Combining this with~\eqref{eq:sm-temp-mse} and
\eqref{eq:sm-temp-kl-curvature} proves~\eqref{eq:sm-temp-finite}.
If
$T+1\ge16384G_{\mathrm{SM}}^2/(565\mu_{\mathrm{SM}}^2)$, then
\[
\EE[\cK(\theta_T^{\mathrm{SM}})]
\ge\frac{565\kappa_{\mathrm{SM}}}{65536}>0.
\]
By~\eqref{eq:sm-temp-exact-gap}, this also lower-bounds
$J(\theta^\dagger_{1,8})-\EE J(\theta_T^{\mathrm{SM}})$.

\paragraph{IterCaL and the limiting reward gaps.}
All assumptions of Lemma~\ref{lem:itercal-optimization} hold by
Proposition~\ref{prop:sm-temp-oracle}. Denote IterCaL's iterates here by
$\theta_T^{\mathrm{IterCaL}}$. Then
\begin{equation}
\EE[\cK(\theta_T^{\mathrm{IterCaL}})]
=J(\theta^\dagger_{1,8})-\EE J(\theta_T^{\mathrm{IterCaL}})
\le K_{\mathrm{opt}}(T+1)^{-1/4}.
\label{eq:sm-temp-itercal-kl}
\end{equation}
For Bernoulli probabilities $r,s\in(0,1)$,
$\KL(\Ber(r)\|\Ber(s))\ge2(r-s)^2$. Therefore
\[
|\overline R(\theta)-\overline R(\theta^\dagger_{1,8})|
\le\frac14\sqrt{\cK(\theta)},
\]
and Jensen's inequality gives convergence of IterCaL's expected
unregularized reward to the oracle value.
For matching,~\eqref{eq:sm-temp-mse} gives convergence in probability to
$\theta_{\mathrm{SM}}^\star$. Continuity and compactness then give
\begin{align*}
\lim_{T\to\infty}
[\EE\overline R(\theta_T^{\mathrm{IterCaL}})
-\EE\overline R(\theta_T^{\mathrm{SM}})]
&=\frac14[\sigma(1/2)-\sigma(5/32)]>0,\\
\lim_{T\to\infty}
[\EE J(\theta_T^{\mathrm{IterCaL}})-\EE J(\theta_T^{\mathrm{SM}})]
&=\frac12[D(5/32\|1/2)+D(9/64\|0)]>0.
\end{align*}
Thus both expected reward differences are strictly positive for all
sufficiently large $T$.
\end{proof}

\section{Quantitative Benefits of Iterative Calibration}
\label{app:itercal-iterative-gains}

This section isolates how the evolving student proposal changes the
information in source calibration and verifies the effect in an
autoregressive family at arbitrary finite horizons. We use superscripts
$\mathrm A$ and $\mathrm F$ for the adaptive IterCaL proposal and the
fixed initial-student proposal, respectively. The complete
budget-matched calibrate-then-learn procedure using the fixed proposal,
its pseudocode, and the general terminal-KL comparison appear in
Appendix~\ref{app:itercal-ctl-comparison}.

\subsection{Accepted information and calibration covariance}
\label{app:itercal-gains-information}

Throughout this section let $a=\gamma^{-1}$,
and let $I_D$ denote the identity on $\RR^D$. For a source prefix and
token pair $(s,c_1,c_0)$ define
\[
\begin{aligned}
z&=\phi(s,c_1)-\phi(s,c_0),\\
A(s,c_1,c_0)&=\Pr(I_t=1\mid s,c_1,c_0).
\end{aligned}
\]
The branch, reference completion, and rejection steps determine $A$;
it is independent of the current calibration parameter. Define the
\emph{accepted information per source reward query} by
\begin{equation}
\mathcal I_H(\theta)
:=\EE_\theta\!\left[A\sigma'(z^\top w^\star_\lambda)zz^\top\right].
\label{eq:itercal-gains-information}
\end{equation}
Here the expectation averages a uniform source prompt, a uniform
position $h\in[H]$, the frozen-teacher prefix and token, and the student
alternative at that prefix, as in \eqref{eq:itercal-student-information}.
Thus rejected queries contribute zero information, and all token
positions contribute through the joint average. The definition permits
arbitrary prefix-dependent features.

For $e=w-w^\star_\lambda$, put
\begin{equation}
\begin{aligned}
h_\theta(w)&:=\EE_\theta\!\left[A z
\{\sigma(z^\top w)-\sigma(z^\top w^\star_\lambda)\}\right],\\
M(w,\theta)&:=\EE_\theta\!\left[Azz^\top
\int_0^1\sigma'(z^\top(w^\star_\lambda+u e))\,du\right].
\end{aligned}
\label{eq:itercal-gains-drift}
\end{equation}
Lemma~\ref{lem:itercal-label} gives $h_\theta(w)=M(w,\theta)e$ and,
at truth, gives both the drift derivative and the gradient covariance:
\begin{equation}
\begin{aligned}
\nabla_w h_\theta(w^\star_\lambda)&=\mathcal I_H(\theta),\\
\EE_\theta[gg^\top\mid w=w^\star_\lambda]&=\mathcal I_H(\theta).
\end{aligned}
\label{eq:itercal-gains-information-identity}
\end{equation}
In particular, Lemma~\ref{lem:itercal-calibration} and
\eqref{eq:itercal-gamma} imply
\begin{equation}
M(w,\theta)\succeq\gamma I_D,\qquad
\mathcal I_H(\theta)\succeq\gamma I_D.
\label{eq:itercal-gains-curvature}
\end{equation}

\begin{proposition}[Adaptive calibration covariance]
\label{prop:itercal-gains-covariance}
Fix any finite $H\ge1$. Assume
Assumptions~\ref{ass:itercal-truth}--\ref{ass:itercal-identification},
and additionally suppose that $w^\star_\lambda$ is an interior point of $W$ in
$\RR^D$. Set
\[
\mathcal I_{\mathrm A}=\mathcal I_H(\theta^\dagger_{\lambda,m}),
\qquad
\mathcal I_{\mathrm F}=\mathcal I_H(\theta_0).
\]
For a positive definite matrix $Q$ with $a\lambda_{\min}(Q)>1/2$, define
\begin{equation}
\begin{aligned}
\Sigma(Q)&:=a^2Q(2aQ-I_D)^{-1}\\
&=\frac a2\left[I_D+(2aQ-I_D)^{-1}\right].
\end{aligned}
\label{eq:itercal-gains-covariance}
\end{equation}
Write $e_t^j=w_t^j-w^\star_\lambda$. For $j\in\{\mathrm A,\mathrm F\}$, the
prescribed calibration iterates satisfy
\begin{equation}
\begin{aligned}
\sqrt T e_T^j&\ \Rightarrow\
\mathcal N(0,\Sigma(\mathcal I_j)),\\
T\EE[e_T^j(e_T^j)^\top]
&\longrightarrow\Sigma(\mathcal I_j).
\end{aligned}
\label{eq:itercal-gains-calibration-limit}
\end{equation}
Moreover, $\EE\|w_t^j-w^\star_\lambda\|_2^4=O(t^{-2})$.
When $\mathcal I_{\mathrm A}\succeq\mathcal I_{\mathrm F}$,
\begin{equation}
\Sigma(\mathcal I_{\mathrm A})
\preceq\Sigma(\mathcal I_{\mathrm F}).
\label{eq:itercal-gains-covariance-order}
\end{equation}
If the information increase is nonzero, the trace inequality is strict;
if it is positive definite, every nonzero directional covariance
inequality is strict.
\end{proposition}

\begin{proof}
We first prove the adaptive assertion and write $e_t=w_t^{\mathrm A}-w^\star_\lambda$.
The finite source experiment and Lemma~\ref{lem:itercal-label} give,
with $\delta_z(w)=\sigma(z^\top w)-\sigma(z^\top w^\star_\lambda)$,
\begin{equation}
\EE_\theta[gg^\top\mid w]
=\mathcal I_H(\theta)
+\EE_\theta[A\delta_z(w)^2zz^\top].
\label{eq:itercal-gains-second-moment}
\end{equation}
Indeed, conditional on acceptance and the pair, $Y$ has distribution
$\Ber(\sigma(z^\top w^\star_\lambda))$. This proves
\eqref{eq:itercal-gains-information-identity}. Acceptance is at least
$e^{-1/\lambda}$ and the logistic derivative on the parameter segment
is at least $\sigma'(2B)$, giving
\eqref{eq:itercal-gains-curvature}.

\emph{Convergence of the adaptive design.}
Theorem~\ref{thm:itercal-main} gives
$\EE\Delta_{\lambda,m}(\theta_t)\to0$. By
Assumption~\ref{ass:itercal-unique}, continuity, and compactness, this
gap has a positive minimum outside every neighborhood of
$\theta^\dagger_{\lambda,m}$. Markov's inequality therefore yields
$\theta_t\to\theta^\dagger_{\lambda,m}$ in probability. Boundedness
of $\Theta$ upgrades this to $L^2$ convergence. The finite sum in
\eqref{eq:itercal-gains-information} is continuous and bounded, so
\begin{equation}
\EE\|\mathcal I_H(\theta_t)-\mathcal I_{\mathrm A}\|_{\mathrm{op}}^2
\longrightarrow0.
\label{eq:itercal-gains-design-limit}
\end{equation}

\emph{Projection and moments.}
Nonexpansiveness relative to $w^\star_\lambda$ gives
\begin{equation}
\EE[\|e_{t+1}\|_2^2\mid\cF_t]
\le\left(1-\frac2{t+2}\right)\|e_t\|_2^2
+\frac{4a^2}{(t+2)^2}.
\label{eq:itercal-gains-mse-recursion}
\end{equation}
The nonnegative almost-supermartingale argument implies that
$\|e_t\|_2^2$ converges almost surely and
$\sum_t\|e_t\|_2^2/(t+2)<\infty$. Consequently $e_t\to0$ almost
surely. An open ball about $w^\star_\lambda$ lies in $W$. Since $\|g_t\|_2\le2$
and $\eta_t\to0$, the unprojected step eventually lies in that ball
almost surely, so projection is inactive after a finite random time.

For the fourth moment, expand
$\|e_t-\eta_tg_t\|_2^4$, use nonexpansiveness, and apply
$e_t^\top h_{\theta_t}(w_t)\ge\gamma\|e_t\|_2^2$ and
$\|g_t\|_2\le2$. This gives
\begin{equation}
\begin{aligned}
\EE[\|e_{t+1}\|_2^4\mid\cF_t]
\le{}&(1-4\gamma\eta_t)\|e_t\|_2^4\\
&+24\eta_t^2\|e_t\|_2^2
+32\eta_t^3\|e_t\|_2+16\eta_t^4.
\end{aligned}
\label{eq:itercal-gains-fourth-recursion}
\end{equation}
Lemma~\ref{lem:itercal-calibration} supplies
$\EE\|e_t\|_2^2=O(t^{-1})$. Taking expectations in
\eqref{eq:itercal-gains-fourth-recursion} therefore gives, for all
sufficiently large $t$,
\[
\EE\|e_{t+1}\|_2^4
\le\left(1-\frac4{t+2}\right)\EE\|e_t\|_2^4+O(t^{-3}).
\]
Comparison with $C/(t+2)^2$ proves
$\EE\|e_t\|_2^4=O(t^{-2})$. In particular,
$T\|e_T\|_2^2$ is uniformly integrable.

\emph{Linearization with the adaptive proposal.}
Let $\xi_{t+1}=g_t-h_{\theta_t}(w_t)$. This is a bounded martingale
difference. Uniformly over $\Theta$, the bounded logistic derivatives give
\[
h_\theta(w^\star_\lambda+e)=\mathcal I_H(\theta)e+O(\|e\|_2^2).
\]
Using \eqref{eq:itercal-gains-second-moment}, its conditional covariance is
\begin{equation}
\begin{aligned}
\EE[\xi_{t+1}\xi_{t+1}^\top\mid\cF_t]
&=\mathcal I_H(\theta_t)+O(\|e_t\|_2^2)\\
&\longrightarrow\mathcal I_{\mathrm A}
\end{aligned}
\label{eq:itercal-gains-noise-limit}
\end{equation}
in $L^1$ and in probability. Express the recursion as
\begin{equation}
\begin{aligned}
e_{t+1}={}&\left(I_D-\frac{a\mathcal I_{\mathrm A}}{t+2}\right)e_t
-\frac a{t+2}(\xi_{t+1}+r_t)+p_{t+1},
\end{aligned}
\label{eq:itercal-gains-linear-recursion}
\end{equation}
where $p_{t+1}$ is the projection correction and is eventually zero
almost surely. Set
$d_t=(\EE\|\mathcal I_H(\theta_t)-\mathcal I_{\mathrm A}\|_{\mathrm{op}}^2)^{1/2}$.
By Cauchy--Schwarz,
\begin{equation}
\begin{aligned}
\EE\|r_t\|_2
\le{}&d_t\bigl(\EE\|e_t\|_2^2\bigr)^{1/2}\\
&+O(\EE\|e_t\|_2^2)=o(t^{-1/2}).
\end{aligned}
\label{eq:itercal-gains-remainder}
\end{equation}
Set $\beta_0=a\lambda_{\min}(\mathcal I_{\mathrm A})\ge1$.
For all sufficiently large deterministic $k$, the fundamental matrix
\[
\Phi_{T,k}:=\prod_{j=k+1}^{T-1}
\left(I_D-\frac{a\mathcal I_{\mathrm A}}{j+2}\right)
\]
has norm at most $C((k+2)/(T+1))^{\beta_0}$. Thus the expected norm of
the root-$T$ remainder convolution is at most
\[
C\sqrt T\,T^{-\beta_0}
\sum_{k<T}(k+2)^{\beta_0-1}o(k^{-1/2})=o(1).
\]
The initial condition and each of the finitely many projection corrections
also vanish at this scale, the latter almost surely. This proves that
the root-$T$ limit is determined by the weighted martingale differences.

\emph{Martingale limit and expected covariance.}
The deterministic root-$T$ weights of $\xi_{k+1}$ have maximum norm
tending to zero. Bounded increments imply the Lindeberg condition, and
\eqref{eq:itercal-gains-noise-limit} gives convergence of their conditional
variance by a weighted averaging argument. Its limit is
\begin{equation}
\begin{aligned}
&a^2\int_0^1
x^{a\mathcal I_{\mathrm A}-I_D}\mathcal I_{\mathrm A}
x^{a\mathcal I_{\mathrm A}-I_D}\,dx\\
&\qquad=\Sigma(\mathcal I_{\mathrm A}).
\end{aligned}
\label{eq:itercal-gains-variance-integral}
\end{equation}
Diagonalizing $\mathcal I_{\mathrm A}$ evaluates the integral;
$a\lambda_{\min}(\mathcal I_{\mathrm A})\ge1$ ensures finiteness.
The martingale central limit theorem now proves the distributional
limit. This is the stochastic approximation covariance form in
\citep[Theorem~2.1]{fort2015clt}; the preceding argument verifies the
adaptive-design and projection terms for the present recursion.
The fourth-moment bound gives the expected covariance limit by uniform
integrability of directional squares and polarization.

For the fixed proposal, the same proof has
$\mathcal I_H(\theta_t)=\mathcal I_{\mathrm F}$ at every source update.
Finally, matrix inversion reverses positive definite order in
\eqref{eq:itercal-gains-covariance}. This proves
\eqref{eq:itercal-gains-covariance-order}, including for noncommuting
information matrices. A nonzero information increase produces a nonzero
positive semidefinite covariance decrease, hence a strict trace decrease;
a positive definite increase produces a positive definite decrease.
\end{proof}

\subsection{Information-dependent bounds for the returned student}
\label{app:itercal-gains-student}

The calibration covariance also sharpens the final-student bounds for
the general, potentially misspecified student class. Define
\begin{equation}
\begin{aligned}
A_0&:=16\alpha_{\mathrm{stu}}\lambda^2B^2H^3+16\lambda BH,\\
D_0&:=8\alpha_{\mathrm{stu}}\lambda^2BH^3+4\lambda H,\\
\beta(Q)&:=\sqrt{\operatorname{tr}\Sigma(Q)},\\
C_{\mathrm{gap}}(Q)&:=
\sqrt{\frac{A_0+D_0\beta(Q)}{\kappa_{\mathrm{opt}}}}.
\end{aligned}
\label{eq:itercal-gains-gap-constant}
\end{equation}
Here $\kappa_{\mathrm{opt}}$ is the common true-objective progress constant
in Lemma~\ref{lem:itercal-hybrid}.

\begin{corollary}[Information-dependent student bounds]
\label{cor:itercal-gains-student}
Under the hypotheses of Proposition~\ref{prop:itercal-gains-covariance},
for $j\in\{\mathrm A,\mathrm F\}$,
\begin{equation}
\limsup_{T\to\infty}T^{1/4}
\EE\Delta_{\lambda,m}(\theta_T^j)
\le C_{\mathrm{gap}}(\mathcal I_j).
\label{eq:itercal-gains-gap-certificate}
\end{equation}
Let $a_{\lambda,m}>0$ and $p_{\lambda,m}\ge1$ be the constants in
Lemma~\ref{lem:itercal-lojas}. Then
\begin{equation}
\begin{aligned}
\limsup_{T\to\infty}T^{1/(4p_{\lambda,m})}
\EE\cK_{\lambda,m}(\theta_T^j)
\le{}&a_{\lambda,m}^{-1/p_{\lambda,m}}\\
&\times C_{\mathrm{gap}}(\mathcal I_j)^{1/p_{\lambda,m}}.
\end{aligned}
\label{eq:itercal-gains-kl-certificate}
\end{equation}
When student learning increases the accepted information,
$\mathcal I_{\mathrm A}\succeq\mathcal I_{\mathrm F}$ with a nonzero
difference, both upper-bound constants decrease strictly for IterCaL.
\end{corollary}

\begin{proof}
Write $\kappa=\kappa_{\mathrm{opt}}$ and
$M=M_{\mathrm{opt}}=8\lambda B^2H^{3/2}$. The one-step comparison in
the proof of Lemma~\ref{lem:itercal-optimization}, retaining its calibration
term, gives for $u_k=\EE\Delta_{\lambda,m}(\theta_k^{\mathrm A})$,
\begin{equation}
\begin{aligned}
u_k\le{}&u_{k-1}-\kappa u_{k-1}^2
+\frac{A_0}{\sqrt{k+1}}\\
&+D_0\EE\|w_k^{\mathrm A}-w^\star_\lambda\|_2.
\end{aligned}
\label{eq:itercal-gains-target-recursion}
\end{equation}
In this comparison the gradient perturbation contributes
$16\alpha_{\mathrm{stu}}\lambda^2B^2H^3/\sqrt{k+1}$ and
$8\alpha_{\mathrm{stu}}\lambda^2BH^3\EE\|w_k-w^\star_\lambda\|_2$;
validation contributes $16\lambda BH/\sqrt{k+1}$; and transferring the
candidate costs to the fixed true objective contributes
$4\lambda H\EE\|w_k-w^\star_\lambda\|_2$. These are precisely $A_0$ and $D_0$.
The uniform exploration candidate supplies the quadratic progress term
through Lemma~\ref{lem:itercal-hybrid} and Jensen's inequality.

For each $T$, write $v_k^{(T)}$ for the expected two-stage target gap
after $k$ target updates. It satisfies
\begin{equation}
\begin{aligned}
v_k^{(T)}\le{}&v_{k-1}^{(T)}
-\kappa(v_{k-1}^{(T)})^2
+\frac{A_0}{\sqrt{k+1}}\\
&+D_0\EE\|w_T^{\mathrm F}-w^\star_\lambda\|_2,
\qquad 1\le k\le T.
\end{aligned}
\label{eq:itercal-gains-fixed-target-recursion}
\end{equation}
All candidate, gradient, and validation estimates remain uniform for the
frozen $w_T^{\mathrm F}\in W$. Both sequences have gaps in $[0,M]$, and
$\kappa\le1/(4M)$.

We give an explicit terminal-window comparison. If a sequence
$0\le x_k\le M$ satisfies
$x_k\le x_{k-1}-\kappa x_{k-1}^2+f$ for $L$ consecutive steps, put
$r=\sqrt{f/\kappa}$. If $r\ge M$, the bound below follows from
$x_k\le M$. Otherwise $v\mapsto v-\kappa v^2$ is increasing on
$[0,M]$. For $y_k=(x_k-r)_+$ this implies
\[
y_k\le y_{k-1}-\kappa y_{k-1}^2,
\]
because the exact expression when $x_{k-1}\ge r$ also contains the
nonpositive term $-2\kappa r y_{k-1}$, and when $x_{k-1}<r$ the next
gap is at most $r$. The right side is nonnegative. While $y_k>0$,
\[
\frac1{y_k}
\ge\frac1{y_{k-1}(1-\kappa y_{k-1})}
\ge\frac1{y_{k-1}}+\kappa.
\]
If zero is reached, it remains zero. After $L$ steps we obtain
\begin{equation}
x_{\mathrm{last}}
\le\sqrt{f/\kappa}+\frac1{\kappa L}.
\label{eq:itercal-gains-window}
\end{equation}

Proposition~\ref{prop:itercal-gains-covariance} and Cauchy--Schwarz give
\begin{equation}
\begin{aligned}
\sqrt k\,\EE\|w_k^{\mathrm A}-w^\star_\lambda\|_2
&\le\beta(\mathcal I_{\mathrm A})+o(1),\\
\sqrt T\,\EE\|w_T^{\mathrm F}-w^\star_\lambda\|_2
&\le\beta(\mathcal I_{\mathrm F})+o(1).
\end{aligned}
\label{eq:itercal-gains-calibration-rmse}
\end{equation}
Choose $L=\lfloor T^q\rfloor$ with $1/4<q<1$. Every index in the last
$L$ updates is asymptotic to $T$. The forcing terms in
\eqref{eq:itercal-gains-target-recursion} and
\eqref{eq:itercal-gains-fixed-target-recursion} are therefore bounded
uniformly on this window by
\[
f_T^j=\frac{A_0+D_0\beta(\mathcal I_j)+o(1)}{\sqrt T},
\qquad j\in\{\mathrm A,\mathrm F\}.
\]
Apply \eqref{eq:itercal-gains-window} and multiply by $T^{1/4}$.
Since $T^{1/4}/L\to0$, this proves
\eqref{eq:itercal-gains-gap-certificate} for both methods, including the
two-stage sequence indexed by its source budget $T$.

Lemma~\ref{lem:itercal-lojas} and Jensen's inequality give
\[
\EE\cK_{\lambda,m}(\theta_T^j)
\le a_{\lambda,m}^{-1/p_{\lambda,m}}
\bigl(\EE\Delta_{\lambda,m}(\theta_T^j)\bigr)^{1/p_{\lambda,m}},
\]
which proves \eqref{eq:itercal-gains-kl-certificate}.
Finally, a nonzero positive semidefinite information increase strictly
decreases $\operatorname{tr}\Sigma$ by
Proposition~\ref{prop:itercal-gains-covariance}. Since $D_0>0$, it
strictly decreases $C_{\mathrm{gap}}$ and both stated certificate constants.
\end{proof}

\subsection{An autoregressive family at arbitrary horizons}
\label{app:itercal-gains-family}

We next give a family in which the information comparison yields a strict
improvement in the actual final student KL. Fix any integer $H\ge1$ and set
$B=\lambda=1$, $W=\Theta=[-1,1]$, $w_{\mathrm{tea}}=w_0=\theta_0=1$, and
$r=-0.9$. There is one source prompt and one distinct target prompt, with
the same features and reward formula. Answers have $H$ binary tokens
$a_h\in\{-1,+1\}$, and both tokens are legal at every position. Put $a_0=1$
and define
\begin{align}
v_h&=a_{h-1}a_h,
&\phi(s_h,a_h)&=\phi_{\mathrm{stu}}(s_h,a_h)=v_h,
\label{eq:itercal-gains-family-features}\\
d_H&=-\frac{0.2}{H},
&b_H&=r-d_H,
&R(a_{1:H})&=0.2+d_H\sum_{h=1}^H v_h.
\label{eq:itercal-gains-family-reward}
\end{align}
Take $\pi_{\mathrm{pre}}=\pi_{b_H}$ with these same features. In particular,
$R\in[0,0.4]$ at every horizon. The resulting conditional law is
\[
\pi_u(a_h\mid a_{1:h-1})
=\frac{e^{u a_{h-1}a_h}}{2\cosh u},
\qquad
\pi_u(a_h=1\mid a_{1:h-1})=\sigma(2u a_{h-1}).
\]
For $H\ge2$, this is prefix dependent under the teacher, reference, and
oracle. Every position is informative; there are no padded positions.
The transformation $a_{1:H}\mapsto v_{1:H}$ is bijective, with inverse
$a_h=\prod_{j=1}^h v_j$. Under $\pi_u$, the transformed signs are independent
with probability $\sigma(2u)$ of being positive. Thus
\begin{equation}
\pi_u(a_{1:H})=\frac{\exp(u\sum_h v_h)}{(2\cosh u)^H},
\qquad
\pi_{\mathrm{pre}}(a_{1:H})e^{R(a_{1:H})}
\ \propto\ \pi_r(a_{1:H}).
\label{eq:itercal-gains-family-law}
\end{equation}
The shared Gibbs oracle is therefore $\pi_r$, and the unique oracle student
parameter is $\theta^\dagger=r\in\operatorname{int}(\Theta)$.

\paragraph{Exact accepted information.}
Write $p=\sigma(2)$, $q(\theta)=\sigma(2\theta)$, and
$\mathfrak m(\theta)=p+(1-2p)q(\theta)$ for the teacher--student mismatch
probability at any branching position. Define
\begin{align}
T_H&=pe^{d_H}+(1-p)e^{-d_H},
&K_H&=\sigma(2b_H)e^{d_H}+[1-\sigma(2b_H)]e^{-d_H},\nonumber\\
\bar A_H&=\frac{e^{-0.8}}{H}
\sum_{h=1}^H T_H^{h-1}K_H^{H-h+1},
&\mathcal I_H(\theta)&=4\sigma'(2r)\mathfrak m(\theta)\bar A_H.
\label{eq:itercal-gains-family-information}
\end{align}
To verify these formulas, condition on an informative pair at position $h$.
Its two tokens exhaust the legal token set, so the reference-weighted branch
choice has transformed positive-bit probability $\sigma(2b_H)$, independently
of the pair's orientation. The reference suffix has independent transformed
bits with that same probability. Conditional on the frozen teacher prefix,
the acceptance probability is consequently
\[
e^{-0.8}\exp\!\left(d_H\sum_{j<h}v_j\right)K_H^{H-h+1}.
\]
Averaging the prefix under the frozen teacher gives
$e^{-0.8}T_H^{h-1}K_H^{H-h+1}$. Informative pairs have $z=\pm2$, whereas
coincident tokens have $z=0$. This proves
\eqref{eq:itercal-gains-family-information} for information per source reward
query, including rejected trials. A uniform alternative mismatches with
probability $1/2$ at each position, so $G_h=2$ and $G_{\mathrm{joint}}=2$.
For completeness, the finite geometric sum also gives
\[
\bar A_H=e^{-0.8}K_H\frac{K_H^H-T_H^H}{H(K_H-T_H)}.
\]
The expression is interpreted continuously if $K_H=T_H$.
The unchanged algorithm therefore uses
\begin{equation}
\gamma=2e^{-3}\sigma'(2),\qquad
a=\gamma^{-1},\qquad
\eta_t=\frac{a}{t+2},\qquad
\alpha=\frac{1}{2H(1+8H)},\qquad
b_{t+1}=q_{t+1}=t+2.
\label{eq:itercal-gains-family-schedules}
\end{equation}

The following lemma retains the stochastic gradient, projection, uniform
candidate, and independent candidate validation used by Algorithm~\ref{alg:itercal}.

\begin{lemma}[Tracking the calibrated policy]
\label{lem:itercal-gains-family-tracking}
For every fixed $H$, the family above satisfies
$\EE(\theta_t-w_t)^2=O_H(t^{-3/2})$. The same bound holds uniformly for
$T$ target updates with calibration frozen at any random $\widehat w\in W$,
conditional on that calibration and with fresh target samples.
\end{lemma}
\begin{proof}
Put $\mu=\sigma'(2)$ and $P=1/4$. The trajectory bijection gives
\begin{align}
C_w(\theta)&=H\KL\bigl(\Ber(\sigma(2\theta))\|\Ber(\sigma(2w))\bigr),
&C'_w(\theta)&=4H\sigma'(2\theta)(\theta-w),\nonumber\\
c_0(\theta-w)^2&\le C_w(\theta)\le c_1(\theta-w)^2,
&c_0&=2H\mu,\quad c_1=2HP.
\label{eq:itercal-gains-family-quadratic}
\end{align}
The bounds follow by integrating $C'_w$ from $w$ to $\theta$.
For a target trajectory, the actual log-ratio cost $Z$ and score $S$ satisfy
$|Z|\le2H|\theta-w|$ and $|S|\le2H$. Hence the unbiased gradient average
over $b$ trajectories obeys
\begin{equation}
\operatorname{Var}(\widehat g\mid\theta,w)
\le\frac{16H^4}{b}(\theta-w)^2.
\label{eq:itercal-gains-family-multiplicative-noise}
\end{equation}
In particular, this gradient noise vanishes at $\theta=w$.

The exact step $y=\theta-\alpha C'_w(\theta)$ lies between $\theta$ and $w$,
since $4\alpha HP<1$. On $[-1,1]$,
$|C''_w|\le H(1+4)\le L_{\mathrm{st}}=H(1+8H)$.
Descent and \eqref{eq:itercal-gains-family-quadratic} give
$C_w(y)\le(1-\nu)C_w(\theta)$, where
\[
\nu=\alpha(1-\alpha L_{\mathrm{st}}/2)\frac{8H\mu^2}{P}
=\frac{12\mu^2}{1+8H}\in(0,1).
\]
The unprojected stochastic step lies in $[-B_e,B_e]$, where
$B_e=1+8\alpha H^2$. The analytic extension of $C_w$ to that interval has
absolute Hessian bounded by $L_e=H[1+2(B_e+1)]$.
Clipping into $[-1,1]$ decreases $C_w$, because this cost is monotone on
either side of its minimum $w$. Taylor expansion about $y$, the zero mean
of $\widehat g-C'_w(\theta)$, and
\eqref{eq:itercal-gains-family-multiplicative-noise} therefore yield, for
the projected stochastic gradient candidate $G$,
\begin{equation}
\EE[C_w(G)\mid\theta,w]
\le(1-\nu+D_g/b)C_w(\theta),
\qquad
D_g=\frac{4L_e\alpha^2H^3}{\mu}.
\label{eq:itercal-gains-family-gradient-candidate}
\end{equation}
Here $L_e$ controls only the stochastic Taylor term; the step size remains
the prescribed $1/(2L_{\mathrm{st}})$.

Next condition on $w$ and $G=g$, and let $U\sim\operatorname{Unif}[-1,1]$.
Fix $\epsilon>0$. On $C_w(u)>(1+\epsilon)C_w(g)$, the gap between the two
validation means is at least $\epsilon C_w(u)/(1+\epsilon)$. Hoeffding's
inequality for the two fresh independent $q$-trajectory means gives
\begin{align*}
\Pr(U\text{ selected}\mid u,g,w)
&\le\exp\!\left[-\frac{q[C_w(u)-C_w(g)]^2}
{8H^2[(u-w)^2+(g-w)^2]}\right]\\
&\le e^{-a_\epsilon q C_w(u)},
\qquad
a_\epsilon=\frac{\mu\epsilon^2}{8H(1+\epsilon)^2}.
\end{align*}
The second inequality uses
\eqref{eq:itercal-gains-family-quadratic} and
$C_w(g)<C_w(u)/(1+\epsilon)$. On the complementary set either selected
cost is at most $(1+\epsilon)C_w(g)$. Integrating the remaining cost against
the uniform candidate gives the precise selection bound
\begin{align}
\EE[C_w(\theta_{\mathrm{sel}})\mid w,G=g]
&\le(1+\epsilon)C_w(g)+D_\epsilon q^{-3/2},\nonumber\\
D_\epsilon&=\frac{c_1\sqrt\pi}{4(a_\epsilon c_0)^{3/2}}.
\label{eq:itercal-gains-family-selection}
\end{align}
Indeed, the remainder is bounded by
$\tfrac{c_1}{2}\int_{\RR}x^2e^{-a_\epsilon q c_0x^2}\,dx$.
The fixed relative margin above is essential for the exponential bound.
The exponent $3/2$ is unchanged with $H$, because the uniform candidate
has one parameter dimension.

Finally, the Bernoulli Bregman identity gives
\[
C_v(\theta)=C_w(\theta)+C_v(w)
+2H[\sigma(2w)-\sigma(2\theta)](v-w).
\]
Using the quadratic bounds and Young's inequality,
\begin{equation}
C_v(\theta)\le(1+\epsilon)C_w(\theta)+D_{\mathrm d}(v-w)^2,
\qquad
D_{\mathrm d}=4H\left[\frac P2+\frac{P^2}{2\epsilon\mu}\right].
\label{eq:itercal-gains-family-drift}
\end{equation}
Choose $\epsilon=\nu/16$. After finitely many updates, $D_g/b\le\nu/2$.
Combining \eqref{eq:itercal-gains-family-gradient-candidate}--
\eqref{eq:itercal-gains-family-drift} and
$|w_{t+1}-w_t|\le2a/(t+2)$ gives, with
$h_t=\EE C_{w_t}(\theta_t)$,
\[
h_{t+1}\le\rho h_t+O_H((t+2)^{-2})+D_\epsilon(t+2)^{-3/2},
\qquad
\rho=(1+\epsilon)^2(1-\nu/2)\le1-3\nu/8<1.
\]
The finite initial segment is bounded. Splitting the ensuing geometric
convolution at $t/2$ proves $h_t=O_H(t^{-3/2})$; its older terms decay
geometrically and its newer terms sum to $O_H(t^{-3/2})$.
The lower quadratic bound proves the claimed mean-square rate.
For frozen calibration the drift term is zero. All constants are uniform
in $w\in[-1,1]$, so conditioning on a random $\widehat w$ proves the
second assertion as well.
\end{proof}

\begin{proposition}[Strict terminal KL gain at matched budgets]
\label{prop:itercal-gains-family-terminal}
For each fixed $H\ge1$, compare Algorithm~\ref{alg:itercal} with the following
two-stage procedure: run $T$ identical calibration updates from $w_0=1$,
using the fixed initial student $\pi_1$ for every source alternative; then
freeze its final calibration and run $T$ identical target updates from
$\theta_0=1$. Both methods use \eqref{eq:itercal-gains-family-schedules},
with target batches indexed by target-update number. Each uses exactly
$T$ source reward queries and $3T(T+3)/2$ target rollouts.
Let
\[
\mathcal I_A=\mathcal I_H(r),\qquad
\mathcal I_F=\mathcal I_H(1),\qquad
V(I)=\frac{a^2I}{2aI-1}.
\]
For the final students from alternation and two-stage training, respectively,
\begin{align}
T\EE\cK_{1,1}(\theta_T^{A})&\longrightarrow
2H\sigma'(2r)V(\mathcal I_A),\nonumber\\
T\EE\cK_{1,1}(\theta_T^{\mathrm{TS}})&\longrightarrow
2H\sigma'(2r)V(\mathcal I_F).
\label{eq:itercal-gains-family-terminal-limits}
\end{align}
The ratio of these positive constants is at most $0.9304$ for every $H$.
Consequently, alternation has strictly smaller expected final student KL
for all sufficiently large matched $T$, for this specified baseline.
\end{proposition}
\begin{proof}
The unique oracle, interior calibration truth, and $G_{\mathrm{joint}}=2$
verify the hypotheses of Proposition~\ref{prop:itercal-gains-covariance}.
Alternatively, student proposal convergence follows directly from
Lemma~\ref{lem:itercal-gains-family-tracking} and the calibration MSE bound.
That proposition supplies $T\EE(w_T-r)^2\to V(\mathcal I)$ and
$\EE|w_T-r|^4=O(T^{-2})$, for $\mathcal I=\mathcal I_A$ or $\mathcal I_F$.
Set $K(\theta)=\KL(\pi_\theta\|\pi_r)=\cK_{1,1}(\theta)$.
Since $K'(u)=4H\sigma'(2u)(u-r)$,
\[
|K(\theta)-K(w)|
\le4HP\left[|w-r||\theta-w|+\frac12(\theta-w)^2\right].
\]
Lemma~\ref{lem:itercal-gains-family-tracking} and Cauchy--Schwarz make
its expected right-hand side $O_H(T^{-5/4})=o(T^{-1})$ for both methods.
Taylor expansion at $r$ has coefficient $K''(r)/2=2H\sigma'(2r)$ and
expected remainder $O_H(\EE|w_T-r|^3)=O_H(T^{-3/2})$.
This proves \eqref{eq:itercal-gains-family-terminal-limits} while retaining
the actual stochastic target update and the random frozen calibration.

Because $p>1/2$ and $r<1$,
\[
k:=\frac{\mathcal I_A}{\mathcal I_F}
=\frac{\mathfrak m(r)}{\mathfrak m(1)}
=3.6800540383\ldots>1,
\qquad
\frac{V(\mathcal I_A)}{V(\mathcal I_F)}
=1-\frac{k-1}{2ak\mathcal I_F-1}.
\]
The last expression is increasing in $\mathcal I_F$, hence in $\bar A_H$.
Since acceptance is $e^{R-1}$ and $R\in[0,0.4]$,
$e^{-1}\le\bar A_H\le e^{-0.6}$. Substituting the upper endpoint and
$\mathcal I_F=4\sigma'(2r)\mathfrak m(1)\bar A_H$ yields
\begin{equation}
\frac{V(\mathcal I_A)}{V(\mathcal I_F)}
\le
\frac{k[8a\sigma'(2r)\mathfrak m(1)e^{-0.6}-1]}
{8ak\sigma'(2r)\mathfrak m(1)e^{-0.6}-1}
=0.9303962828\ldots<0.9304.
\label{eq:itercal-gains-family-uniform-ratio}
\end{equation}
All denominators are positive since $\mathcal I_F\ge\gamma$.
Both terminal limits are positive, so their strict ordering gives the
eventual expected-KL inequality.
\end{proof}

For illustration, the prescribed calibration multiplier is $a=95.65125738\ldots$.
The exact constants from \eqref{eq:itercal-gains-family-terminal-limits} are
\[
\begin{array}{c|rrr}
H & \text{two-stage KL constant} & \text{alternating KL constant}
  & \text{ratio}\\ \hline
2 & 26.05330 & 23.97895 & 0.92038\\
8 & 104.87753 & 96.06798 & 0.91600\\
64 & 840.71092 & 768.92875 & 0.91462
\end{array}
\]
Each asymptotic statement fixes $H$ before letting $T$ grow. The tracking
constants and the sufficient crossover time may depend on $H$. The comparison uses the same initial student, source queries, target
updates, and batch schedules; only the source proposal adapts during
IterCaL training.

\section{Iterative Calibration versus Calibrate-Then-Learn}
\label{app:itercal-ctl-comparison}

This section compares IterCaL with a matched two-stage procedure that
first calibrates using alternatives from the initial student and then
learns the student with the frozen calibrated policy. The comparison
retains the projected stochastic student update, the uniform candidate,
and its independent validation batches.

\paragraph{Matched calibrate-then-learn baseline.}
To avoid repeating Algorithm~\ref{alg:itercal}, write
$\operatorname{\textsc{SourceCal}}$ for
Lines~\ref{line:br-design}--\ref{line:br-project}. Its inputs are the
current calibration $w$, alternative-token policy
$\pi_{\mathrm{stu},\vartheta}$, and step size $\eta$.
Thus this subroutine includes the frozen-teacher trajectory, branching,
reference completion, source reward query, rejection step, logistic
gradient, and projection onto $W$.
Write $\operatorname{\textsc{TargetLearn}}$ for
Lines~\ref{line:br-kl}--\ref{line:br-student}; its inputs are the current
student $\theta$, calibration $w$, gradient batch size $b$, and validation
budget $q$ per candidate. It retains the projected stochastic-gradient
candidate, the independent uniform candidate, and their fresh validation
batches.

\begin{algorithm}[t]
\caption{Calibrate-Then-Learn (CTL)}
\label{alg:itercal-ctl}
\small
\begin{algorithmic}[1]
\Require The inputs of Algorithm~\ref{alg:itercal}, including
$w_{\mathrm{tea}},\theta_0,T$ and the schedules
in \eqref{eq:br-budgets}.
\State $w^{\mathrm{CTL}}_0\gets w_{\mathrm{tea}}$.
\Statex \textit{Stage I: calibrate with fixed initial-student alternatives.}
\For{$t=0,\ldots,T-1$}
    \State $w^{\mathrm{CTL}}_{t+1}\gets
    \operatorname{\textsc{SourceCal}}
    (w^{\mathrm{CTL}}_t,\theta_0,\eta_t)$.
    \Comment{One source reward query.}
\EndFor
\State $\overline w_T\gets w^{\mathrm{CTL}}_T$ and
$\theta^{\mathrm{CTL}}_0\gets\theta_0$.
\Statex \textit{Stage II: learn with the frozen calibrated policy.}
\For{$k=0,\ldots,T-1$}
    \State $\theta^{\mathrm{CTL}}_{k+1}\gets
    \operatorname{\textsc{TargetLearn}}
    (\theta^{\mathrm{CTL}}_k,\overline w_T;k+2,k+2)$.
\EndFor
\State \Return $\pi_{\mathrm{stu},\theta^{\mathrm{CTL}}_T}$.
\end{algorithmic}
\end{algorithm}

CTL and IterCaL start from the same $(w_{\mathrm{tea}},\theta_0)$.
Each makes exactly $T$ source reward queries. Target update $k+1$ in
either method uses
\[
b_{k+1}+2q_{k+1}=3(k+2)
\]
full-answer rollouts. Hence each method uses exactly
\[
N_T=3\sum_{k=0}^{T-1}(k+2)=\frac{3T(T+3)}2
\]
target rollouts. Their only difference is chronology and the source
alternative policy: CTL uses $\theta_0$ throughout calibration, whereas
IterCaL uses the evolving $\theta_t$ and interleaves the two updates.

\subsection{Adaptive comparison coverage}

Recall that $I_t=\mathbf 1\{U_t\le\exp((R_t-1)/\lambda)\}$ and denote
$A(s,c_1,c_0):=\Pr(I_t=1\mid s_t=s,c_{t,1}=c_1,c_{t,0}=c_0)$.
We define the accepted information per source query by
\[
\mathcal I_H(\theta)
=
\EE_\theta\!\left[
A(s,c_1,c_0)
\sigma'\!\left(z^\top w^\star_\lambda\right)zz^\top
\right],
\ 
z=\phi(s,c_1)-\phi(s,c_0).
\]
The expectation follows the source branching experiment, with
$c_0\sim\pi_{\mathrm{stu},\theta}(\cdot\mid s)$.
Define
\begin{equation}
\begin{aligned}
\mathcal I_{\mathrm{Iter}}
&:=\mathcal I_H(\theta^\dagger_{\lambda,m}),&
\mathcal I_{\mathrm{CTL}}
&:=\mathcal I_H(\theta_0),\\
\Delta\mathcal I
&:=\mathcal I_{\mathrm{Iter}}-\mathcal I_{\mathrm{CTL}}.
\end{aligned}
\label{eq:itercal-ctl-information}
\end{equation}
These are the limiting accepted-comparison information of IterCaL and
the information used throughout CTL, respectively.

We introduce a coverage assumption using the covariance-based approach
common in preference-learning theory. Intuitively, comparisons should
provide information about the parameter directions needed for accurate
prediction; covariance matrices quantify how well each direction is
covered \citep{du2024exploration,zhu2023principled}.
\citet{kim2026coverage} further show how on-policy updates can improve
coverage in subsequent preference comparisons.

In our setting, we require the oracle student's alternatives to provide
at least as much calibration information as the initial student's,
with a strict gain in a direction relevant to target predictions.
To express the latter requirement, we describe how calibration
errors affect the target student. At an interior oracle with positive
definite local curvature, define
\begin{equation}
Q:=\nabla_{\theta\theta}^2C_{w^\star_\lambda}(\theta^\dagger_{\lambda,m}),
\  B_{\theta w}:=\nabla_{\theta w}^2C_{w^\star_\lambda}(\theta^\dagger_{\lambda,m}),
\  L:=-Q^{-1}B_{\theta w}.
\label{eq:itercal-ctl-sensitivity}
\end{equation}
Here $L$ maps a small calibration perturbation to the first-order change
in the minimizing student parameter. The corresponding change in target
answer probabilities is measured by the Fisher matrix
\begin{equation}
\begin{aligned}
F_{\mathrm{tar}}
:={}&
\frac1m\sum_{j=1}^m
\EE_{\pi_{\mathrm{stu},\theta^\dagger_{\lambda,m}}
(\cdot\mid\widetilde x_j)}
\bigl[
S_{\mathrm{stu},\theta^\dagger_{\lambda,m}}
S_{\mathrm{stu},\theta^\dagger_{\lambda,m}}^\top
\bigr]
\end{aligned}
\label{eq:itercal-ctl-target-fisher}
\end{equation}
where both score vectors inside the expectation are evaluated at
$(\widetilde x_j,a_{1:H})$.
With $a=\gamma^{-1}$, define the calibration weighting matrix
\[
M_u:=
\left(
2a\{\mathcal I_{\mathrm{CTL}}
+u\Delta\mathcal I\}-I_D
\right)^{-1},
\  0\le u\le1.
\]
Here $M_0$ weights calibration directions according to CTL's source
information and step-size schedule.
\begin{assumption}[Target-relevant adaptive comparison coverage]
\label{ass:itercal-adaptive-coverage}
Compared with the initial student, the oracle student's alternatives
provide no less calibration information and a strict gain relevant to
target predictions:
\begin{equation}
\Delta\mathcal I\succeq0,
\ 
(\Delta\mathcal I)^{1/2}M_0L^\top F_{\mathrm{tar}}^{1/2}\ne0.
\label{eq:itercal-adaptive-coverage}
\end{equation}
\end{assumption}
The condition allows zero gain in directions irrelevant to target
predictions. It holds, for example, when $\Delta\mathcal I\succ0$ and
$F_{\mathrm{tar}}^{1/2}L\ne0$.

\subsection{Matched-budget terminal-KL comparison}

To compare the final students, we separate two sources of error:
source calibration error shifts the policy that the student learns,
and sampled target gradients introduce additional training noise.
We quantify both contributions to the leading expected oracle-policy KL.

For source calibration, the bounds $A(s,c_1,c_0)\ge e^{-1/\lambda}$
and $\sigma'(z^\top w^\star_\lambda)\ge\sigma'(2B)$, together with
Lemma~\ref{lem:itercal-calibration}, give
\[
\mathcal I_H(\theta)
\succeq e^{-1/\lambda}\sigma'(2B)G_{\mathrm{src}}(\theta)
\succeq\gamma I_D,
\  \theta\in\Theta.
\]
Thus both $\mathcal I_{\mathrm{Iter}}$ and $\mathcal I_{\mathrm{CTL}}$
satisfy this lower bound. With $a=\gamma^{-1}$, either limiting
information matrix satisfies $2a\mathcal I-I_D\succeq I_D$, so we
can define
\[
\Sigma(\mathcal I)
:=a^2\mathcal I(2a\mathcal I-I_D)^{-1}
=\frac a2\left[I_D+(2a\mathcal I-I_D)^{-1}\right].
\]
For calibration step size $a/(t+2)$, the proof below identifies
$\Sigma(\mathcal I)$ as the limiting covariance of
$\sqrt T(w_T-w^\star_\lambda)$. Through the sensitivity matrix $L$,
this calibration error contributes $L\Sigma(\mathcal I)L^\top$
to the limiting student covariance.

Both methods use the same target-learning updates and rollout budgets.
The proof shows that their target-gradient sampling noise contributes
the same leading term to the final expected KL. The following theorem
therefore quantifies the improvement from source calibration under the
coverage condition: IterCaL achieves a smaller expected final
oracle-policy KL at the same budget.

\begin{theorem}[Strict gain over calibrate-then-learn]
\label{thm:itercal-ctl-gain}
Suppose the assumptions of Theorem~\ref{thm:itercal-main} and
Assumption~\ref{ass:itercal-adaptive-coverage} hold. Additionally, suppose
\[
w^\star_\lambda\in\operatorname{int}(W),\ 
\theta^\dagger_{\lambda,m}\in\operatorname{int}(\Theta),\ 
Q\succ0,\ d\ge3.
\]
Let $\theta_T^{\mathrm{Iter}}$ and $\theta_T^{\mathrm{CTL}}$ be the
outputs of IterCaL and Algorithm~\ref{alg:itercal-ctl} under the
matched budgets above. As $T\to\infty$,
\begin{equation}
\EE[\cK_{\lambda,m}(\theta_T^{\mathrm{CTL}})]
-\EE[\cK_{\lambda,m}(\theta_T^{\mathrm{Iter}})]
=\frac{\Gamma_{\mathrm{cov}}}{T}+o(T^{-1}),
\label{eq:itercal-ctl-strict-gain}
\end{equation}
where
\begin{equation}
\Gamma_{\mathrm{cov}}
:=\frac{a^2}{2}\int_0^1
\left\|(\Delta\mathcal I)^{1/2}
M_uL^\top F_{\mathrm{tar}}^{1/2}\right\|_{\mathrm F}^2\,du>0.
\label{eq:itercal-ctl-gain-constant}
\end{equation}
Hence IterCaL has strictly smaller expected final oracle-policy KL
for every sufficiently large matched $T$.
\end{theorem}

\begin{proof}[Proof of Theorem~\ref{thm:itercal-ctl-gain}]
We now prove Theorem~\ref{thm:itercal-ctl-gain} using the following
lemmas, whose proofs we collect in
Appendix~\ref{app:proofs_itercal-ctl}.
We first control calibration and student-tracking errors, then
determine their covariance contributions and compare the resulting
expected KL errors. We work under the assumptions and matched
schedules of the theorem.
Write $w_t^{\mathrm{Iter}}:=w_t$ and
$\theta_t^{\mathrm{Iter}}:=\theta_t$, and retain
$w_t^{\mathrm{CTL}}$ for CTL's calibration iterates, so that
$w_T^{\mathrm{CTL}}=\overline w_T$.

We begin with moment bounds that hold for both fixed and adaptive
source comparisons. We will use these bounds to control departures
from a neighborhood of $w^\star_\lambda$ and the remainders in the
calibration covariance analysis.

\begin{lemma}[Calibration moments]
\label{lem:itercal-ctl-source-moments}
For either calibration procedure and every fixed integer $p\ge1$,
\begin{equation}
\EE\|w_t-w^\star_\lambda\|_2^{2p}=O((t+1)^{-p}).
\label{eq:itercal-ctl-source-moments}
\end{equation}
At $w^\star_\lambda$, the conditional calibration drift Jacobian and
gradient covariance both equal $\mathcal I_H(\theta)$, uniformly bounded
below by $\gamma I_D$.
\end{lemma}

To translate calibration accuracy into student accuracy, we next
identify the student that minimizes each nearby calibrated objective.
We describe its dependence on $w$ through $\psi(w)$ and relate
excess cost to distance from this minimizer.

\begin{lemma}[The calibrated minimizer]
\label{lem:itercal-ctl-minimizer}
There are neighborhoods $V,V^+$ of $w^\star_\lambda$ with
$\overline V\subset V^+$ and a smooth map
$\psi:V^+\to\operatorname{int}(\Theta)$ such that, for
$w\in\overline V$, $\psi(w)$ is the unique global minimizer of $C_w$ and
\[
\psi(w^\star_\lambda)=\theta^\dagger_{\lambda,m},\
D\psi(w^\star_\lambda)=L.
\]
There are constants $c_0,c_1>0$, uniform over $w\in\overline V$, such that
\begin{equation}
c_0\|\theta-\psi(w)\|_2^2
\le C_w(\theta)-C_w(\psi(w))
\le c_1\|\theta-\psi(w)\|_2^2,\ \theta\in\Theta.
\label{eq:itercal-ctl-quadratic-geometry}
\end{equation}
\end{lemma}

We use the uniform bounds on $\overline V$ throughout the argument.
For unconditional moments, we extend $\psi$ smoothly and boundedly
outside $V$, using a cutoff supported in $V^+$; we invoke its
minimizing property only inside $V$.

We now account for the uniform candidate and noisy validation.
Using the quadratic bounds above, we relate candidate cost
differences to parameter distances and bound the displacement
caused by selecting the uniform candidate.

\begin{lemma}[Effect of noisy candidate selection]
\label{lem:itercal-ctl-selection}
Fix $w\in V$ and a projected gradient candidate $g\in\Theta$. Let
$U\sim\operatorname{Unif}(\Theta)$ be the other candidate, evaluated
using the independent validation batches of size $q$. Put
\[
e_g:=\|g-\psi(w)\|_2,\
r_{\mathrm{sel}}:=(U-g)\mathbf1\{U\text{ is selected}\}.
\]
For $p\in\{2,4\}$,
\begin{equation}
\EE[\|r_{\mathrm{sel}}\|_2^p\mid w,g]
\le C_p\{e_g^{d+p}+q^{-(d+p)/4}+e_g^p q^{-d/4}\}.
\label{eq:itercal-ctl-selection-moments}
\end{equation}
\end{lemma}

At update $t$, we use $g=\vartheta_{t,1}$ and
$U=\vartheta_{t,2}$, and denote the resulting displacement by
$r_{\mathrm{sel},t}$.
To apply these local bounds along the training run, we also need
the student to enter and remain near $\psi(w)$.
We use the uniform candidate to reach this neighborhood and the
gradient update with fresh validation to stay there. Together with
the calibration moments, we obtain the following localization bounds.

\begin{lemma}[Localization of the student]
\label{lem:itercal-ctl-localization}
There is $\tau_2>0$ such that, with
\[
N_2(w):=\{\theta\in\Theta:
C_w(\theta)-C_w(\psi(w))\le\tau_2\},\ w\in V,
\]
and $N_2(w):=\varnothing$ outside $V$, the following bounds hold.
For every fixed integer $r_0\ge2$,
\begin{equation}
\Pr\{\theta_t^{\mathrm{Iter}}\notin N_2(w_t)\}
\le C_{r_0}t^{1-r_0}+Ce^{-ct}.
\label{eq:itercal-ctl-localization}
\end{equation}
For CTL, uniformly over $\overline w_T\in V$,
\[
\Pr\{\theta_k^{\mathrm{CTL}}\notin N_2(\overline w_T)
\mid\overline w_T\}\le Ce^{-ck},\
\Pr\{\overline w_T\notin V\}=O(T^{-r_0}).
\]
\end{lemma}

We can now quantify how closely the student follows the calibrated
minimizer. For IterCaL, we distinguish the tracking error from the
movement of the minimizer by writing
$\zeta_t^{\mathrm{Iter}}=\zeta_t:=\theta_t-\psi(w_t)$ and
$\delta_t:=\psi(w_t)-\psi(w_{t-1})$.
For CTL's learning stage, we write
$\zeta_k^{\mathrm{CTL}}:=\theta_k^{\mathrm{CTL}}-\psi(\overline w_T)$
and $\delta_k=0$.
In either procedure, let $\mathcal G_{t-1}^{\mathrm{grad}}$ contain the information
available just before the target gradient batch, including its
calibration. We write the centered gradient error as
\[
\widehat g_t^{\mathrm{stu}}
=\nabla C_{w_t}(\theta_{t-1})+\varepsilon_t,\
\EE[\varepsilon_t\mid\mathcal G_{t-1}^{\mathrm{grad}}]=0.
\]
We write $\varepsilon_t^j$ for this error in method $j$.
Here and below, $w_t$ in a CTL target update means the fixed
$\overline w_T$. Using bounded gradient summands, we have conditional
second and fourth moments of $\varepsilon_t$ of order $b_t^{-1}$
and $b_t^{-2}$.
We combine these estimates with
Lemmas~\ref{lem:itercal-ctl-source-moments},
\ref{lem:itercal-ctl-selection}, and
\ref{lem:itercal-ctl-localization} to obtain the tracking moments
below. We also bound the selection displacement more sharply so that
we can remove it from the subsequent linear representation.

\begin{lemma}[Student tracking moments]
\label{lem:itercal-ctl-tracking-moments}
Let $\beta_4:=\min\{2,(d+4)/4\}>1$. For IterCaL,
\begin{equation}
\EE\|\zeta_t\|_2^2=O(t^{-1}),
\label{eq:itercal-ctl-target-second}
\end{equation}
\begin{equation}
\EE\|\zeta_t\|_2^4=O(t^{-\beta_4}).
\label{eq:itercal-ctl-target-fourth-proof}
\end{equation}
The same bounds hold for CTL at learning time $t$, conditional on
$\overline w_T=w$, uniformly over $w\in V$.
For both methods at their terminal time,
\begin{equation}
\EE\|\theta_T^j-\theta^\dagger_{\lambda,m}\|_2^2=O(T^{-1}),\
\EE\|\theta_T^j-\theta^\dagger_{\lambda,m}\|_2^4=O(T^{-\beta_4}),
\ j\in\{\mathrm{Iter},\mathrm{CTL}\}.
\label{eq:itercal-ctl-target-fourth}
\end{equation}
The candidate-selection displacement satisfies, for IterCaL and
for CTL conditional on $\overline w_T=w$, uniformly over $w\in V$,
\begin{equation}
\EE\|r_{\mathrm{sel},t}\|_2^2=o(t^{-1}).
\label{eq:itercal-ctl-selection-negligible}
\end{equation}
\end{lemma}

We now have the student moments needed to analyze the adaptive source
design. Together with the calibration moments, we use them to control
the drift remainder arising from
$\mathcal I_H(\theta_t)-\mathcal I_{\mathrm{Iter}}$.
For CTL, we work with the fixed information
$\mathcal I_{\mathrm{CTL}}$. We thereby obtain a linear representation
and the limiting calibration covariance for each method.

\begin{lemma}[Limiting calibration covariance]
\label{lem:itercal-ctl-source-limit}
For $j\in\{\mathrm{Iter},\mathrm{CTL}\}$, there is a sum
$\mathcal M_T^j$ of the centered source gradient increments with
deterministic matrix weights such that
\begin{equation}
\begin{aligned}
\sqrt T(w_T^j-w^\star_\lambda)&=\mathcal M_T^j+o_{L^2}(1),\\
\mathcal M_T^j&\Rightarrow\mathcal N(0,\Sigma(\mathcal I_j)),\\
\EE[\mathcal M_T^j(\mathcal M_T^j)^\top]
&\longrightarrow\Sigma(\mathcal I_j).
\end{aligned}
\label{eq:itercal-ctl-source-representation}
\end{equation}
\end{lemma}

Having identified the calibration contribution, we turn to the
student's tracking error. We use the tracking moments to linearize
the actual target update and express its leading error through
the centered gradient increments.
Set $A_\theta:=I_d-\alpha_{\mathrm{stu}}Q$.
Since $0\prec Q\preceq L_{\mathrm{st}}I_d$ and
$\alpha_{\mathrm{stu}}=1/(2L_{\mathrm{st}})$, we have
$\|A_\theta\|_{\mathrm{op}}<1$.

\begin{lemma}[Linear representation of student tracking]
\label{lem:itercal-ctl-tracking-linear}
For both methods,
\begin{equation}
\sqrt T\,\zeta_T^j
=-\alpha_{\mathrm{stu}}\sqrt T
\sum_{k=1}^T A_\theta^{T-k}\varepsilon_k^j+o_{L^2}(1),\
j\in\{\mathrm{Iter},\mathrm{CTL}\}.
\label{eq:itercal-ctl-target-representation}
\end{equation}
For CTL at a fixed calibration $w\in V$, the same representation holds
with $A_\theta(w):=I_d-\alpha_{\mathrm{stu}}
\nabla_{\theta\theta}^2C_w(\psi(w))$, uniformly in $w$.
\end{lemma}

We next compute the covariance of this gradient sum.
We draw $j\sim\operatorname{Unif}\{1,\ldots,m\}$ and
$a_{1:H}\sim\pi_{\mathrm{stu},\theta^\dagger_{\lambda,m}}
(\cdot\mid\widetilde x_j)$, and define the single-rollout gradient
at the oracle by
\[
G^\dagger(j,a_{1:H})
:=\lambda S_{\mathrm{stu},\theta^\dagger_{\lambda,m}}
(\widetilde x_j,a_{1:H})
\log\frac{\pi_{\mathrm{stu},\theta^\dagger_{\lambda,m}}
(a_{1:H}\mid\widetilde x_j)}
{\pi_{w^\star_\lambda}(a_{1:H}\mid\widetilde x_j)},\
\Omega_g:=\operatorname{Cov}(G^\dagger).
\]
At the interior oracle, we have $\EE[G^\dagger]=0$.
To sum the contributions of earlier gradient noise after propagation
through the later updates, we define
\begin{equation}
P:=\alpha_{\mathrm{stu}}^2\sum_{\ell=0}^\infty
A_\theta^\ell\Omega_g(A_\theta^\ell)^\top,\
P=A_\theta PA_\theta^\top+\alpha_{\mathrm{stu}}^2\Omega_g.
\label{eq:itercal-ctl-target-noise}
\end{equation}
Using $\|A_\theta\|_{\mathrm{op}}<1$, we verify that this series
converges. Under the matched batch schedule, we obtain the same
limiting tracking covariance for both methods.

\begin{lemma}[Common target-learning covariance]
\label{lem:itercal-ctl-target-limit}
For both methods,
\[
T\EE[\zeta_T^j(\zeta_T^j)^\top]\longrightarrow P,\
j\in\{\mathrm{Iter},\mathrm{CTL}\}.
\]
\end{lemma}

We are now ready to combine the two error contributions.
Using $D\psi(w^\star_\lambda)=L$ and the calibration moments, we have
\[
\theta_T^j-\theta^\dagger_{\lambda,m}
=L(w_T^j-w^\star_\lambda)+\zeta_T^j
+o_{L^2}(T^{-1/2}),\
j\in\{\mathrm{Iter},\mathrm{CTL}\}.
\]
We combine the source and target linear representations and use
martingale orthogonality under the algorithms' sampling order to add
their limiting covariances. We then apply the Fisher expansion of KL
and the student moment bounds to obtain the expected KL limits.

\begin{lemma}[Terminal student covariance and KL]
\label{lem:itercal-ctl-terminal-kl}
For $j\in\{\mathrm{Iter},\mathrm{CTL}\}$,
\begin{equation}
T\EE[(\theta_T^j-\theta^\dagger_{\lambda,m})
(\theta_T^j-\theta^\dagger_{\lambda,m})^\top]
\longrightarrow L\Sigma(\mathcal I_j)L^\top+P,
\label{eq:itercal-ctl-terminal-covariance}
\end{equation}
and
\begin{equation}
T\EE[\cK_{\lambda,m}(\theta_T^j)]
\longrightarrow
\frac12\operatorname{tr}\!\left[
F_{\mathrm{tar}}\{L\Sigma(\mathcal I_j)L^\top+P\}\right].
\label{eq:itercal-ctl-individual-limits}
\end{equation}
\end{lemma}

We finish by comparing these two limits.
Applying Lemma~\ref{lem:itercal-ctl-terminal-kl}, we subtract the
common target-learning contribution $P$ and obtain
\[
\begin{aligned}
&\lim_{T\to\infty}T\{\EE[\cK_{\lambda,m}(\theta_T^{\mathrm{CTL}})]
-\EE[\cK_{\lambda,m}(\theta_T^{\mathrm{Iter}})]\}\\
&\ =\tfrac12\operatorname{tr}\!\left[
F_{\mathrm{tar}}L\{\Sigma(\mathcal I_{\mathrm{CTL}})
-\Sigma(\mathcal I_{\mathrm{Iter}})\}L^\top\right].
\end{aligned}
\]
We interpolate between the information matrices along
$\mathcal I_u=\mathcal I_{\mathrm{CTL}}+u\Delta\mathcal I$ and
differentiate the inverse to obtain
\[
\frac{d}{du}(2a\mathcal I_u-I_D)^{-1}
=-2aM_u\Delta\mathcal I M_u.
\]
Integrating this identity and using the definition of $\Sigma$, we have
\begin{equation}
\Sigma(\mathcal I_{\mathrm{CTL}})
-\Sigma(\mathcal I_{\mathrm{Iter}})
=a^2\int_0^1M_u\Delta\mathcal I M_u\,du.
\label{eq:itercal-ctl-covariance-difference}
\end{equation}
Substituting into the trace, we obtain
$\Gamma_{\mathrm{cov}}$ in \eqref{eq:itercal-ctl-gain-constant}.
By Assumption~\ref{ass:itercal-adaptive-coverage}, we have
$\Delta\mathcal I\succeq0$ and
$(\Delta\mathcal I)^{1/2}M_0L^\top F_{\mathrm{tar}}^{1/2}\ne0$.
Thus the squared-norm integrand is nonnegative everywhere and
strictly positive at $u=0$. By continuity, we obtain
$\Gamma_{\mathrm{cov}}>0$ and conclude
\eqref{eq:itercal-ctl-strict-gain}, including the strict ordering
for all sufficiently large matched $T$.
\end{proof}

\subsection{A six-step example: planning speaker roles in a dialogue}
\label{app:itercal-six-step-example}

We instantiate Theorem~\ref{thm:itercal-ctl-gain} in a concrete
structured-generation task: assigning speaker roles to six successive
utterance slots. The source prompt describes a tutor--learner practice
dialogue, while the target prompt describes an interviewer--candidate
practice dialogue. At step $h$, the model outputs a role tag
$a_h\in\{-1,+1\}$: $+1$ denotes the tutor or interviewer, and $-1$
denotes the learner or candidate. We take $a_0=+1$ as the preceding
speaker and write
\[
v_h:=a_{h-1}a_h,\ h=1,\ldots,6.
\]
Thus $v_h=+1$ means that the same role continues, whereas $v_h=-1$
means that the speaker role switches. For this fixed-length task, we
set $\cB(s_h)=\{-1,+1\}$ for $h\le6$ and automatically append
$\mathtt{EOS}$ after the sixth role token, outside the modeled
sequence $a_{1:6}$.

\paragraph{Source and target policies.}
We take one source prompt $x_{\mathrm{src}}$ and one target prompt
$x_{\mathrm{tar}}$, so $n=m=1$. Each utterance slot carries a known
type label. We order these labels differently in the two prompts:
\[
k_{\mathrm{src}}=(1,2,3,4,5,6),\
k_{\mathrm{tar}}=(1,4,2,5,3,6).
\]
We use six calibration parameters and three student parameters,
with
\[
H=6,\ D=6,\ d=3,\ \lambda=1,\ B=\frac34,\
W=\{w\in\mathbb R^6:\|w\|_2\le B\},\
\Theta=\{\theta\in\mathbb R^3:\|\theta\|_2\le B\}.
\]
Let $\iota(k):=1+((k-1)\bmod3)$. For a role token $a$ at prefix
$s_h$, define
\[
\phi(s_h,a):=a_{h-1}a\,e_{k_x(h)},\
\phi_{\mathrm{stu}}(s_h,a):=
a_{h-1}a\,e_{\iota(k_x(h))},
\]
where each $e_k$ is a coordinate vector in the corresponding space.
The calibrated policy is therefore
\begin{equation}
\pi_w(a_h\mid s_h)
=\frac{\exp(w_{k_x(h)}a_{h-1}a_h)}
{2\cosh(w_{k_x(h)})}.
\label{eq:itercal-six-step-policy}
\end{equation}
The student uses the same expression with $w_{k_x(h)}$ replaced by
$\theta_{\iota(k_x(h))}$. It thus shares one parameter across each
pair of slot types $(1,4)$, $(2,5)$, and $(3,6)$.
The token law depends on the preceding role, and all six positions
remain stochastic.

\paragraph{Reward, reference, and biased teacher.}
We score role alternation by
\begin{equation}
R(x,a_{1:6})
:=\frac14+\frac1{12}\sum_{h=1}^6\mathbf1\{a_h\ne a_{h-1}\}
=\frac12-\frac1{24}\sum_{h=1}^6v_h
\in\left[\frac14,\frac34\right].
\label{eq:itercal-six-step-reward}
\end{equation}
A source verifier computes this score from the completed role
sequence. We follow the source-only feedback protocol: the learner
queries the verifier on $x_{\mathrm{src}}$, while the target reward
defines the evaluation objective.
For example, the role sequence $(+1,+1,-1,-1,+1,-1)$ switches
roles at steps $3,5,6$ and receives score $1/2$.

Set $c:=1/24$ and
\[
r:=\left(-\frac25,-\frac3{10},-\frac15,
-\frac1{10},-\frac15,-\frac3{10}\right)^\top,\
\pi_{\mathrm{pre}}:=\pi_{r+c\mathbf1_6},\
w_{\mathrm{tea}}:=\frac14\mathbf1_6,\
\theta_0:=\frac3{10}\mathbf1_3.
\]
The frozen teacher favors continuing the previous role, and the initial
student has an even stronger tendency to do so. We will show that
learning the alternation preference increases the information supplied
by the student's source comparisons.

\begin{proposition}[Strict gain in role-sequence planning]
\label{prop:itercal-six-step-gain}
The construction satisfies the assumptions of
Theorem~\ref{thm:itercal-ctl-gain}.
The unique oracle student is
$\theta^\dagger_{1,1}=-(1/4)\mathbf1_3$, and the student class
cannot represent the unrestricted optimal policy.
The frozen teacher has strictly higher target expected reward and
true regularized return than the initial student.
Moreover,
\begin{equation}
\mathcal I_{\mathrm{Iter}}
=\kappa\mathcal I_{\mathrm{CTL}},\
\kappa>1,\ \kappa\approx1.1414233,
\label{eq:itercal-six-step-information-gain}
\end{equation}
and the oracle target-gradient covariance satisfies $\Omega_g\succ0$.
Under the schedules in \eqref{eq:br-budgets},
\begin{equation}
\EE[\cK_{1,1}(\theta_T^{\mathrm{CTL}})]
-\EE[\cK_{1,1}(\theta_T^{\mathrm{Iter}})]
=\frac{\Gamma_{\mathrm{cov}}}{T}+o(T^{-1}),\
\Gamma_{\mathrm{cov}}\approx1.0855084>0.
\label{eq:itercal-six-step-kl-gain}
\end{equation}
Both methods use $T$ source reward queries and
$3T(T+3)/2$ target rollouts.
\end{proposition}

\begin{proof}
We first identify the true calibration and oracle student.
The map from $a_{1:6}$ to $v_{1:6}$ is a bijection, and
\[
\pi_{\mathrm{pre}}(a_{1:6}\mid x)e^{R(x,a_{1:6})}
\propto
\exp\left\{\sum_{h=1}^6r_{k_x(h)}v_h\right\}
\propto\pi_r(a_{1:6}\mid x).
\]
Thus we have $w^\star_1=r$ on both prompts.
For binary natural parameters, write
\[
\mathsf d(u,v):=
\KL\!\left(\operatorname{Bern}(\sigma(2u))
\,\middle\|\,\operatorname{Bern}(\sigma(2v))\right)
=(u-v)\tanh u-\log\cosh u+\log\cosh v.
\]
Using the student parameter sharing, we obtain
\[
C_w(\theta)=\sum_{j=1}^3
\{\mathsf d(\theta_j,w_j)+\mathsf d(\theta_j,w_{j+3})\},\
\partial_j C_w(\theta)
=4\sigma'(2\theta_j)(2\theta_j-w_j-w_{j+3}).
\]
For $w=r$, we have $r_j+r_{j+3}=-1/2$ for every $j$.
The derivative changes sign only at $\theta_j=-1/4$, so we obtain
the unique oracle $\theta^\dagger_{1,1}=-(1/4)\mathbf1_3$.
Since $r_j\ne r_{j+3}$ in every pair, the student class cannot
represent $\pi_r$. In particular,
\begin{equation}
c_{\mathrm{app}}:=
C_r(\theta^\dagger_{1,1})
=\sum_{h=1}^6\log\cosh r_h-6\log\cosh(1/4)
\approx0.025784>0.
\label{eq:itercal-six-step-misspecification}
\end{equation}
All feature norms equal one, and
$r,w_{\mathrm{tea}},\theta_0,\theta^\dagger_{1,1}$ lie strictly
inside their respective parameter balls.
Putting $\chi:=\sigma'(1/2)$, we compute
\begin{equation}
Q=F_{\mathrm{tar}}=8\chi I_3,\
B_{\theta w}=-4\chi[I_3\ I_3],\
L=\frac12[I_3\ I_3].
\label{eq:itercal-six-step-geometry}
\end{equation}
At source position $h$, a uniform legal token matches the teacher
token with probability $1/2$ and otherwise has feature difference
$\pm2e_h$. We therefore have
$G_h=2e_he_h^\top$ and $G_{\mathrm{joint}}=I_6/3$, which verifies
source identification.

We next compute the information from the actual student alternatives.
Let $p_{\mathrm{tea}}:=\sigma(1/2)$. At a source position with
student parameter $u$, the probability of different candidate tokens is
\[
p_{\mathrm{diff}}(u):=
p_{\mathrm{tea}}\{1-\sigma(2u)\}
+(1-p_{\mathrm{tea}})\sigma(2u)
=p_{\mathrm{tea}}+(1-2p_{\mathrm{tea}})\sigma(2u).
\]
Since $p_{\mathrm{tea}}>1/2$, we have
$p'_{\mathrm{diff}}(u)<0$.
Conditional on different candidates, the pair contains both legal
tokens. The reference-weighted branch choice is therefore exactly a
reference-policy draw at that prefix.
Since the disagreement probability does not depend on the preceding
roles, we retain the teacher-prefix law under this conditioning.
Given the teacher prefix $v_{1:h-1}$, we obtain conditional acceptance
probability
\[
e^{-1/2-c\sum_{\ell<h}v_\ell}
\prod_{\ell=h}^6
\frac{\cosh r_\ell}{\cosh(r_\ell+c)}.
\]
We then average over the frozen-teacher prefix to obtain
\begin{equation}
\overline A_h:=
e^{-1/2}
\left\{\frac{\cosh(1/4-c)}{\cosh(1/4)}\right\}^{h-1}
\prod_{\ell=h}^6
\frac{\cosh r_\ell}{\cosh(r_\ell+c)}>0.
\label{eq:itercal-six-step-acceptance}
\end{equation}
For a transition with natural parameter $u$, we use
$\EE[e^{-cv}]=\cosh(u-c)/\cosh u$.
Different candidates have $zz^\top=4e_he_h^\top$, and the branch
position is uniform on six steps. Hence the accepted information
per source query is
\begin{equation}
\mathcal I_H(\theta)
=\operatorname{diag}_{1\le h\le6}
\left\{\frac46\overline A_h\,\sigma'(2r_h)
p_{\mathrm{diff}}(\theta_{\iota(h)})\right\}.
\label{eq:itercal-six-step-information}
\end{equation}
Since every initial coordinate is $3/10$ and every oracle coordinate
is $-1/4$, we obtain
\[
\kappa=
\frac{p_{\mathrm{diff}}(-1/4)}{p_{\mathrm{diff}}(3/10)}
\approx1.1414233>1.
\]
Thus $\Delta\mathcal I\succ0$.
Together with \eqref{eq:itercal-six-step-geometry}, we verify both
parts of Assumption~\ref{ass:itercal-adaptive-coverage} directly.

We can now evaluate the terminal-KL improvement.
The prescribed schedules have
\[
\gamma=\frac13e^{-5/2}\sigma'(3/2),\
a=\gamma^{-1},\
\alpha_{\mathrm{stu}}=\frac1{444}.
\]
Let $i_h^{\mathrm{CTL}}$ and $i_h^{\mathrm{Iter}}$ denote the
diagonal entries in \eqref{eq:itercal-six-step-information} at
$\theta_0$ and $\theta^\dagger_{1,1}$, respectively.
Using \eqref{eq:itercal-six-step-geometry} in the covariance
comparison, we obtain the exact constant
\begin{equation}
\Gamma_{\mathrm{cov}}
=\chi a^2\sum_{h=1}^6
\left\{
\frac{i_h^{\mathrm{CTL}}}{2ai_h^{\mathrm{CTL}}-1}
-\frac{i_h^{\mathrm{Iter}}}{2ai_h^{\mathrm{Iter}}-1}
\right\}
\approx1.0855084.
\label{eq:itercal-six-step-constant}
\end{equation}
Each summand is positive because
$i_h^{\mathrm{Iter}}>i_h^{\mathrm{CTL}}$ and
$\frac{d}{di}\{a^2i/(2ai-1)\}=-a^2/(2ai-1)^2<0$.
Applying Theorem~\ref{thm:itercal-ctl-gain}, we conclude
\eqref{eq:itercal-six-step-kl-gain} under its original matched budgets.

We also verify that the target-gradient noise is nondegenerate.
Under the oracle, the six transition variables have independent,
full-support laws. In target order, the student score has coordinates
$S_j=v_{2j-1}+v_{2j}-2\tanh(-1/4)$, and the answer log-ratio is
\[
\ell(v):=
\log\frac{\pi_{\mathrm{stu},\theta^\dagger_{1,1}}(a_{1:6}\mid x_{\mathrm{tar}})}
{\pi_r(a_{1:6}\mid x_{\mathrm{tar}})}
=c_{\mathrm{app}}+\frac3{20}(v_1-v_2)
+\frac1{20}(v_3-v_4-v_5+v_6).
\]
By strict convexity of $\log\cosh$ and its second-derivative bound
$(\log\cosh)''\le1$, we have
\[
0<c_{\mathrm{app}}
\le\frac12\sum_{h=1}^6(r_h+1/4)^2
=\frac{11}{400}<\frac1{10}.
\]
The nonconstant part of $\ell(v)$ is an integer multiple of $1/10$,
so $\ell(v)\ne0$ on every trajectory.
By varying each token pair while fixing the others, we see that
the score support spans $\mathbb R^3$.
Using $\EE[G^\dagger]=0$ and $G^\dagger=S\ell(v)$, we therefore have
$\upsilon^\top\Omega_g\upsilon
=\EE[\ell(v)^2(\upsilon^\top S)^2]>0$
for every $\upsilon\in\mathbb R^3\setminus\{0\}$.

Finally, we compare the teacher and initial student.
For the common scalar policy parameter $u$, the expected target
reward is $1/2-(1/4)\tanh u$, and the cost relative to $\pi_r$ is
$\sum_h\mathsf d(u,r_h)$.
On $u\in[1/4,3/10]$, the reward strictly decreases, and we have
\[
\frac{d}{du}\sum_{h=1}^6\mathsf d(u,r_h)
=4\sigma'(2u)\left(6u-\sum_{h=1}^6r_h\right)>0.
\]
Using the Gibbs identity, we conclude that the
teacher at $u=1/4$ has higher expected reward and regularized return
than the initial student at $u=3/10$.
\end{proof}

In this instance, the oracle student switches speaker roles more often
than the initial student and therefore proposes a different source
token from the frozen teacher more often. We obtain a strict increase
in accepted comparison information in every calibration direction.
Through the target sensitivity matrix $L$, this increase produces the
strictly smaller asymptotic oracle-policy KL established above.

\section{Proofs of the Calibrate-Then-Learn Comparison Lemmas}
\label{app:proofs_itercal-ctl}

We prove the lemmas used in the proof of
Theorem~\ref{thm:itercal-ctl-gain}, retaining its assumptions,
schedules, and notation.

\begin{proof}[Proof of Lemma~\ref{lem:itercal-ctl-source-moments}]
Write $g=Iz\{\sigma(z^\top w)-Y\}$. Conditional on acceptance and the
candidate pair, the label is Bernoulli with mean
$\sigma(z^\top w^\star_\lambda)$. Thus the drift vanishes at
$w^\star_\lambda$, and differentiation and the Bernoulli variance give
\[
\left.\nabla_w\EE[g\mid\theta]\right|_{w=w^\star_\lambda}
=\EE[gg^\top\mid\theta,w=w^\star_\lambda]
=\mathcal I_H(\theta).
\]
The acceptance, logistic-curvature, and source-coverage bounds used in
Lemma~\ref{lem:itercal-calibration} give
$\mathcal I_H(\theta)\succeq\gamma I_D$ and, for
$e_t=w_t-w^\star_\lambda$,
$\langle e_t,\EE[g_t\mid\mathcal F_t]\rangle
\ge\gamma\|e_t\|_2^2$.

Projection is nonexpansive and the stochastic gradients are bounded.
Expanding the $2p$th norm power therefore gives, for all sufficiently
large $t$,
\[
\begin{aligned}
\EE[\|e_{t+1}\|_2^{2p}\mid\mathcal F_t]
\le{}&
\left(1-\frac{2p}{t+2}\right)\|e_t\|_2^{2p}\\
&+\frac{C_p}{(t+2)^2}\|e_t\|_2^{2p-2}
+\frac{C_p}{(t+2)^{2p}},
\end{aligned}
\]
where $\eta_t=a/(t+2)$ and $a\gamma=1$.
For $p=1$, this recursion gives $O((t+1)^{-1})$.
If the result holds for $p-1$, the last two terms have expectation
$O((t+2)^{-p-1})$; the same recursion then gives
$O((t+1)^{-p})$. Induction proves the claim for both the fixed and
adaptive source designs.
\end{proof}

\begin{proof}[Proof of Lemma~\ref{lem:itercal-ctl-minimizer}]
Apply the implicit function theorem to
$\nabla_\theta C_w(\theta)=0$ at
$(\theta^\dagger_{\lambda,m},w^\star_\lambda)$. Since $Q\succ0$,
it gives a smooth local minimizer on a neighborhood $V^+$, with
derivative $-Q^{-1}B_{\theta w}=L$. Choose $V$ with compact closure
inside $V^+$. Positive curvature gives uniform quadratic bounds
near this minimizer.

At $w^\star_\lambda$, uniqueness and compactness give a positive
objective gap outside any fixed neighborhood of the oracle.
Continuity preserves this gap for $w$ sufficiently close to
$w^\star_\lambda$, so every global minimizer lies in the neighborhood
where the objective is strictly convex. It is therefore $\psi(w)$.
Combining the local quadratic lower bound with the gap on the compact
complement gives the global lower bound in
\eqref{eq:itercal-ctl-quadratic-geometry}. The upper bound follows from
smoothness and $\nabla C_w(\psi(w))=0$.
\end{proof}

\begin{proof}[Proof of Lemma~\ref{lem:itercal-ctl-selection}]
Split the uniform candidate according to
$\|U-\psi(w)\|_2\le Ke_g$, choosing $K$ so that $c_0-c_1/K^2>0$.
This set has probability $O(e_g^d)$, and the displacement is
$O(e_g)$ there, giving the first term.

On the complement, \eqref{eq:itercal-ctl-quadratic-geometry} gives
$C_w(U)-C_w(g)\ge c\|U-\psi(w)\|_2^2$.
The validation costs are bounded, so Hoeffding's inequality bounds the
probability of selecting $U$ by
$2\exp[-cq\|U-\psi(w)\|_2^4]$.
Using $\|U-g\|_2^p\le C_p(\|U-\psi(w)\|_2^p+e_g^p)$ and integrating
in polar coordinates gives
\[
\int_0^\infty(r^p+e_g^p)e^{-cqr^4}r^{d-1}\,dr
\le C_p\{q^{-(d+p)/4}+e_g^p q^{-d/4}\}.
\]
Combining the two regions proves the bound.
\end{proof}

\begin{proof}[Proof of Lemma~\ref{lem:itercal-ctl-localization}]
Put $D_w(\theta):=C_w(\theta)-C_w(\psi(w))$ and
$\mathsf G_w(\theta):=
\Proj{\Theta}{\theta-\alpha_{\mathrm{stu}}\nabla C_w(\theta)}$.
Local positive curvature and smoothness give a uniform
$\rho_{\mathrm g}<1$ and an arbitrarily small $\tau_2>0$ such that
$D_w(\mathsf G_w(\theta))\le\rho_{\mathrm g}D_w(\theta)$ whenever
$D_w(\theta)\le\tau_2$. Choose
\[
\rho_{\mathrm g}\tau_2<\tau_0<\tau_1<\tau_2,\
N_i(w):=\{\theta\in\Theta:D_w(\theta)\le\tau_i\},\ i=0,1,2.
\]
Then $\mathsf G_w(N_2(w))$ lies strictly inside $N_0(w)$, and
$p_0:=\inf_{w\in V}\Pr\{U\in N_0(w)\}>0$.

At any current student, the event $U\in N_0(w)$ gives a chance to enter
$N_1(w)$. If the gradient candidate is also in $N_1(w)$, both
candidates are in that set. Otherwise, $U$ is better by at least
$\tau_1-\tau_0$, so validation selects it with probability at least
$1-2e^{-cq_t}$. Thus the conditional probability of entering $N_1(w)$
is at least $p_0(1-2e^{-cq_t})$.

If the current student is already in $N_2(w)$, bounded gradient
summands and the strict margin around the population update give
$\Pr\{\vartheta_{t,1}\notin N_0(w)\mid\mathcal G_{t-1}^{\mathrm{grad}}\}
\le Ce^{-cb_t}$.
When $\vartheta_{t,1}\in N_0(w)$, any uniform candidate outside
$N_1(w)$ is worse by at least $\tau_1-\tau_0$. Hence the next selected
student remains in $N_1(w)$ except with probability
$C(e^{-cb_t}+e^{-cq_t})$.
Smoothness gives
$\sup_{\theta\in\Theta}|D_{w'}(\theta)-D_w(\theta)|
\le C\|w'-w\|_2$.
Since $\|w_{t+1}-w_t\|_2=O(t^{-1})$, the strict margin ensures
$N_1(w_t)\subset N_2(w_{t+1})$ at all sufficiently large times.

Lemma~\ref{lem:itercal-ctl-source-moments} and a union bound give
\[
\Pr\{\text{some }w_k\notin V,\ t/2\le k\le t\}
\le C_{r_0}t^{1-r_0}.
\]
Over these last $t/2$ rounds, repeated chances of entering $N_1$
leave probability $Ce^{-ct}$ of no entry. The sum of the subsequent
exit probabilities is also $Ce^{-ct}$, since $b_k=q_k=k+1$.
This proves \eqref{eq:itercal-ctl-localization}.
With fixed $\overline w_T\in V$, the same argument gives the CTL
bound uniformly in $\overline w_T$; its exceptional probability follows
from Lemma~\ref{lem:itercal-ctl-source-moments}.
\end{proof}

\begin{proof}[Proof of Lemma~\ref{lem:itercal-ctl-tracking-moments}]
The smooth extension of $\psi$ and the bounded source update give
$\|\delta_t\|_2\le C/t$.
On the localized neighborhood, the population gradient map is
uniformly contractive. Projection nonexpansiveness and
$\theta_{t-1}-\psi(w_t)=\zeta_{t-1}-\delta_t$ therefore give, with
$e_{g,t}:=\|\vartheta_{t,1}-\psi(w_t)\|_2$,
\begin{align}
\EE[e_{g,t}^2\mid\mathcal G_{t-1}^{\mathrm{grad}}]
&\le\bar\rho_2\|\zeta_{t-1}\|_2^2+C(b_t^{-1}+t^{-2}),
\label{eq:itercal-ctl-gradient-second}\\
\EE[e_{g,t}^4\mid\mathcal G_{t-1}^{\mathrm{grad}}]
&\le\bar\rho_4\|\zeta_{t-1}\|_2^4
+Cb_t^{-1}\|\zeta_{t-1}\|_2^2+C(b_t^{-2}+t^{-4}),
\label{eq:itercal-ctl-gradient-fourth}
\end{align}
where $\bar\rho_2,\bar\rho_4<1$.
For CTL, the terms from $\delta_t$ vanish.

Apply Lemma~\ref{lem:itercal-ctl-selection} and
$\|x+y\|_2^p\le(1+\epsilon)\|x\|_2^p+C_{\epsilon,p}\|y\|_2^p$
for $p=2,4$.
Choose $\epsilon$ small and then shrink the localized neighborhood
so that the $e_{g,t}^{d+p}$ terms are absorbed by contraction.
Large gradient deviations and failure of localization contribute only
the tails in Lemma~\ref{lem:itercal-ctl-localization}.
For some $\rho_2,\rho_4<1$, this gives
\begin{align}
\EE\|\zeta_t\|_2^2
&\le\rho_2\EE\|\zeta_{t-1}\|_2^2
+Cb_t^{-1}+Cq_t^{-(d+2)/4}+C_{r_0}t^{1-r_0},
\label{eq:itercal-ctl-target-second-recursion}\\
\EE\|\zeta_t\|_2^4
&\le\rho_4\EE\|\zeta_{t-1}\|_2^4
+Cb_t^{-1}\EE\|\zeta_{t-1}\|_2^2
+Cb_t^{-2}+Cq_t^{-(d+4)/4}+C_{r_0}t^{1-r_0}.
\label{eq:itercal-ctl-target-moment-recursions}
\end{align}
Take $r_0$ sufficiently large. Since $b_t=q_t=t+1$ and $d\ge3$,
geometric convolution gives \eqref{eq:itercal-ctl-target-second};
substitution into the fourth-moment recursion gives
\eqref{eq:itercal-ctl-target-fourth-proof}.
For CTL, the localization tail is exponential and the same argument
is uniform over $\overline w_T\in V$.
The smoothness of $\psi$, Lemma~\ref{lem:itercal-ctl-source-moments},
and the negligible probability $\overline w_T\notin V$ then give
\eqref{eq:itercal-ctl-target-fourth}.

Finally, boundedness and $d+2\ge5$ imply
$\EE e_{g,t}^{d+2}\le C\EE e_{g,t}^4=O(t^{-\beta_4})$.
The other second-moment terms in
\eqref{eq:itercal-ctl-selection-moments} are
$O(t^{-(d+2)/4})$ and $O(t^{-1-d/4})$.
All are $o(t^{-1})$, proving
\eqref{eq:itercal-ctl-selection-negligible}.
\end{proof}

\begin{proof}[Proof of Lemma~\ref{lem:itercal-ctl-source-limit}]
Write $e_t:=w_t-w^\star_\lambda$ and let
$\xi_{t+1}:=g_t-\EE[g_t\mid\mathcal F_t]$.
Taylor expansion of the drift in
Lemma~\ref{lem:itercal-ctl-source-moments} gives its linear part
$\mathcal I_j e_t$. For IterCaL, the remaining drift is
\[
r_t^{\mathrm{src}}
=[\mathcal I_H(\theta_t)-\mathcal I_{\mathrm{Iter}}]e_t
+O(\|e_t\|_2^2).
\]
The finite-design information map is smooth. The source moments and
Lemma~\ref{lem:itercal-ctl-tracking-moments} therefore give
\[
\EE\|r_t^{\mathrm{src}}\|_2^2
\le C\sqrt{\EE\|\theta_t-\theta^\dagger_{\lambda,m}\|_2^4
\EE\|e_t\|_2^4}+C\EE\|e_t\|_2^4
=O(t^{-2\nu}),
\]
where $\nu:=\min\{1,(\beta_4+2)/4\}>1/2$.
For CTL, the information is fixed and the remainder is
$O_{L^2}(t^{-1})$.

Since $w^\star_\lambda$ is interior, a late projection correction
requires $\|e_t\|_2$ to exceed a fixed radius. Its $L^2$ norm is at
most $C_p(t+2)^{-1-p/2}$ by the source moments.
For all sufficiently large $k$, the transition matrices
\[
B_{T,k}^j:=\prod_{\ell=k+1}^{T-1}
\left(I_D-\frac a{\ell+2}\mathcal I_j\right)
\]
satisfy $\|B_{T,k}^j\|_{\mathrm{op}}\le C(k+2)/(T+1)$, because
$a\mathcal I_j\succeq I_D$.
Consequently the root-$T$ weighted drift remainder is bounded by
$CT^{-1/2}\sum_{k=1}^T k^{-\nu}=o(1)$; the initial condition and
projection corrections also vanish in $L^2$.
Solving the linear recursion thus gives the first assertion with
\[
\mathcal M_T^j:=-\sqrt T\sum_{k=0}^{T-1}
\frac a{k+2}B_{T,k}^j\xi_{k+1}^j.
\]

Lemma~\ref{lem:itercal-ctl-source-moments}, continuity, and the student
moment bounds show that the conditional covariance of
$\xi_{t+1}^j$ converges in $L^1$ to $\mathcal I_j$.
The increments are bounded and the largest weighted increment
vanishes, giving the martingale Lindeberg condition.
The stochastic-approximation central limit theorem
\citep[Theorem~2.1]{fort2015clt} yields the Gaussian limit, whose
covariance solves
\[
(a\mathcal I_j-\tfrac12I_D)\Sigma
+\Sigma(a\mathcal I_j-\tfrac12I_D)=a^2\mathcal I_j.
\]
Its solution is $\Sigma(\mathcal I_j)$.
Summing the covariances of the deterministic weighted martingale
increments, using the same transition-matrix bound and $L^1$
convergence, gives convergence of their expected covariance as well.
\end{proof}

\begin{proof}[Proof of Lemma~\ref{lem:itercal-ctl-tracking-linear}]
For IterCaL, put
$\widetilde\zeta_{t-1}:=\theta_{t-1}-\psi(w_t)$.
Taylor expansion on the localized neighborhood gives
\[
\nabla C_{w_t}(\theta_{t-1})=Q\widetilde\zeta_{t-1}+R_t,\
\|R_t\|_2\le C\{\|\widetilde\zeta_{t-1}\|_2^2
+\|w_t-w^\star_\lambda\|_2\|\widetilde\zeta_{t-1}\|_2\}.
\]
Lemmas~\ref{lem:itercal-ctl-source-moments} and
\ref{lem:itercal-ctl-tracking-moments} imply
$\|R_t\|_{L^2}=o(t^{-1/2})$.
The selected-candidate displacement has the same $L^2$ order, and
$\|\delta_t\|_2=O(t^{-1})$.
A target projection correction requires failure of localization or
a batch-gradient deviation exceeding the fixed interior margin.
Lemma~\ref{lem:itercal-ctl-localization}, Hoeffding's inequality,
and boundedness make its $L^2$ norm $o(t^{-1/2})$.
Thus the actual update
satisfies
\begin{equation}
\zeta_t=A_\theta\zeta_{t-1}
-\alpha_{\mathrm{stu}}\varepsilon_t+\rho_t,\
\|\rho_t\|_{L^2}=o(t^{-1/2}).
\label{eq:itercal-ctl-target-linearization}
\end{equation}
Writing $r_A=\|A_\theta\|_{\mathrm{op}}<1$, split the remainder
convolution at $T/2$. Bounded iterates and gradients give
$\sup_k\|\rho_k\|_{L^2}<\infty$, so the two parts satisfy
\[
\begin{aligned}
\sqrt T\sum_{k\le T/2}r_A^{T-k}\|\rho_k\|_{L^2}
&\le CT^{3/2}r_A^{T/2}\longrightarrow0,\\
\sqrt T\sum_{k>T/2}r_A^{T-k}\|\rho_k\|_{L^2}
&\le\frac{\sqrt2}{1-r_A}
\sup_{k>T/2}\sqrt{k}\|\rho_k\|_{L^2}\longrightarrow0.
\end{aligned}
\]
The initial error also decays geometrically, proving
\eqref{eq:itercal-ctl-target-representation} for IterCaL.

Conditional on CTL's fixed $w\in V$, expand instead at
$(\psi(w),w)$. The same proof is uniform in $w$, with $\delta_t=0$
and $A_\theta(w)$ in place of $A_\theta$.
When $w=\overline w_T$, smoothness gives
$\|A_\theta(w)-A_\theta\|_{\mathrm{op}}
\le C\|w-w^\star_\lambda\|_2$.
Both matrices are uniformly contractive, so
\[
\|A_\theta(w)^\ell-A_\theta^\ell\|_{\mathrm{op}}
\le C\ell r^{\ell-1}\|w-w^\star_\lambda\|_2
\]
for some $r<1$.
Conditional martingale orthogonality and
$\EE\|\overline w_T-w^\star_\lambda\|_2^2=O(T^{-1})$
show that replacing these powers in the root-$T$ gradient sum
has squared $L^2$ error $O(T^{-1})$.
The event $\overline w_T\notin V$ is negligible by the source
moments. This proves the stated CTL representation.
\end{proof}

\begin{proof}[Proof of Lemma~\ref{lem:itercal-ctl-target-limit}]
For IterCaL, let
$\Omega_k:=b_k\EE[\varepsilon_k\varepsilon_k^\top
\mid\mathcal G_{k-1}^{\mathrm{grad}}]$.
Continuity, boundedness, and
Lemma~\ref{lem:itercal-ctl-tracking-moments} give
$\Omega_k\to\Omega_g$ in $L^1$.
Martingale orthogonality and
Lemma~\ref{lem:itercal-ctl-tracking-linear} give
\[
\begin{aligned}
T\EE[\zeta_T\zeta_T^\top]
&=\alpha_{\mathrm{stu}}^2\sum_{\ell=0}^{T-1}
\frac{T}{b_{T-\ell}}A_\theta^\ell
\EE[\Omega_{T-\ell}](A_\theta^\ell)^\top+o(1)\\
&\longrightarrow P.
\end{aligned}
\]
Indeed, $T/b_{T-\ell}\to1$ for fixed $\ell$; the tail is controlled
by geometric decay, splitting the sum at $\ell=T/2$.

For CTL with fixed $w\in V$, let $\Omega_g(w)$ be the covariance of a
single gradient summand at $(\psi(w),w)$ and put
\[
P(w):=\alpha_{\mathrm{stu}}^2\sum_{\ell=0}^\infty
A_\theta(w)^\ell\Omega_g(w)(A_\theta(w)^\ell)^\top.
\]
The conditional tracking bounds and the fixed-$w$ representation in
Lemma~\ref{lem:itercal-ctl-tracking-linear} are uniform on $V$.
The preceding covariance argument therefore yields
\[
T\EE[\zeta_T^{\mathrm{CTL}}(\zeta_T^{\mathrm{CTL}})^\top
\mid\overline w_T=w]=P(w)+o(1)
\]
uniformly there. Uniform contraction makes $P(w)$ continuous and
bounded, with $P(w^\star_\lambda)=P$.
Since $\overline w_T\to w^\star_\lambda$ in $L^2$, averaging the
conditional limit gives $P$. For $r_0>1$, boundedness gives
\[
T\EE[\|\zeta_T^{\mathrm{CTL}}\|_2^2
\mathbf1\{\overline w_T\notin V\}]
\le CT\Pr\{\overline w_T\notin V\}=o(1),
\]
so the exceptional event does not affect this limit.
\end{proof}

\begin{proof}[Proof of Lemma~\ref{lem:itercal-ctl-terminal-kl}]
Source and target centered increments are orthogonal at every pair of
times. Within an IterCaL round, condition after the source update to
center the target noise; across rounds, condition before the later
increment. In CTL, every source increment precedes all target batches.
The deterministic weights in
Lemmas~\ref{lem:itercal-ctl-source-limit} and
\ref{lem:itercal-ctl-tracking-linear} preserve this orthogonality.

Uniform Taylor expansion of the smooth extension of $\psi$ on $W$
and the source fourth-moment bound give
\[
\psi(w_T^j)-\theta^\dagger_{\lambda,m}
=L(w_T^j-w^\star_\lambda)+O_{L^2}(T^{-1}).
\]
Adding $\zeta_T^j$, and applying the source and target covariance
limits, proves \eqref{eq:itercal-ctl-terminal-covariance}.

The oracle-policy KL has the local expansion
\[
\cK_{\lambda,m}(\theta)
=\tfrac12(\theta-\theta^\dagger_{\lambda,m})^\top
F_{\mathrm{tar}}(\theta-\theta^\dagger_{\lambda,m})
+O(\|\theta-\theta^\dagger_{\lambda,m}\|_2^3).
\]
Compactness extends the remainder bound to all of $\Theta$ after
increasing its constant. Lemma~\ref{lem:itercal-ctl-tracking-moments}
and Cauchy--Schwarz give
\[
\EE\|\theta_T^j-\theta^\dagger_{\lambda,m}\|_2^3
\le
\sqrt{\EE\|\theta_T^j-\theta^\dagger_{\lambda,m}\|_2^2
\EE\|\theta_T^j-\theta^\dagger_{\lambda,m}\|_2^4}
=o(T^{-1}).
\]
Taking expectations of the KL expansion now proves
\eqref{eq:itercal-ctl-individual-limits}.
\end{proof}

%% file: refs.bib
@article{qian2024offline,
  title={Offline Oracle-Efficient Learning for Contextual MDPs via Layerwise Exploration-Exploitation Tradeoff},
  author={Qian, Jian and Hu, Haichen and Simchi-Levi, David},
  journal={arXiv preprint arXiv:2405.17796},
  year={2024}
}

@inproceedings{jin2020provably,
  title={Provably efficient reinforcement learning with linear function approximation},
  author={Jin, Chi and Yang, Zhuoran and Wang, Zhaoran and Jordan, Michael I},
  booktitle={Conference on learning theory},
  pages={2137--2143},
  year={2020},
  organization={PMLR}
}

@article{wainwright2025wild,
  title={Wild refitting for black box prediction},
  author={Wainwright, Martin J},
  journal={arXiv preprint arXiv:2506.21460},
  year={2025}
}

@misc{hu2025perturbingderivativewildrefitting,
      title={Perturbing the Derivative: Wild Refitting for Model-Free Evaluation of Machine Learning Models under Bregman Losses}, 
      author={Haichen Hu and David Simchi-Levi},
      year={2025},
      eprint={2509.02476},
      archivePrefix={arXiv},
      primaryClass={stat.ML}, 
}

@misc{hu2025doublywildrefittingmodelfree,
      title={Perturbing the Derivative: Doubly Wild Refitting for Model-Free Evaluation of Opaque Machine Learning Predictors}, 
      author={Haichen Hu and David Simchi-Levi},
      year={2025},
      eprint={2511.18789},
      archivePrefix={arXiv},
      primaryClass={cs.LG}, 
}

@article{angelopoulos2023prediction,
  author  = {Angelopoulos, Anastasios N. and Bates, Stephen
             and Fannjiang, Clara and Jordan, Michael I. and Zrnic, Tijana},
  title   = {Prediction-powered inference},
  journal = {Science},
  volume  = {382},
  number  = {6671},
  pages   = {669--674},
  year    = {2023}
}

@article{hu2026interleaved,
  title={Interleaved Resampling and Refitting: Data and Compute-Efficient Evaluation of Black-Box Predictors},
  author={Hu, Haichen and Simchi-Levi, David},
  journal={arXiv preprint arXiv:2603.14218},
  year={2026}
}

@article{bierstone-milman-1988,
  author  = {Bierstone, Edward and Milman, Pierre D.},
  title   = {Semianalytic and subanalytic sets},
  journal = {Publications Math{\'e}matiques de l'IH{\'E}S},
  volume  = {67},
  pages   = {5--42},
  year    = {1988}
}

@inproceedings{zhu2025judgelm,
  title = {{JudgeLM}: Fine-tuned Large Language Models are Scalable Judges},
  author = {Zhu, Lianghui and Wang, Xinggang and Wang, Xinlong},
  booktitle = {International Conference on Learning Representations},
  year = {2025},
}

@inproceedings{kim2024prometheus,
  title = {Prometheus: Inducing Fine-grained Evaluation Capability in Language Models},
  author = {Kim, Seungone and Shin, Jamin and Cho, Yejin and Jang, Joel and
            Longpre, Shayne and Lee, Hwaran and Yun, Sangdoo and Shin, Seongjin and
            Kim, Sungdong and Thorne, James and Seo, Minjoon},
  booktitle = {International Conference on Learning Representations},
  year = {2024},
}

@inproceedings{jiang2017contextual,
  title = {Contextual Decision Processes with low {Bellman} rank are {PAC}-Learnable},
  author = {Jiang, Nan and Krishnamurthy, Akshay and Agarwal, Alekh and Langford, John and Schapire, Robert E.},
  booktitle = {Proceedings of the 34th International Conference on Machine Learning},
  series = {Proceedings of Machine Learning Research},
  volume = {70},
  pages = {1704--1713},
  year = {2017},
}

@article{agarwal2021policy,
  title = {On the Theory of Policy Gradient Methods: Optimality, Approximation, and Distribution Shift},
  author = {Agarwal, Alekh and Kakade, Sham M. and Lee, Jason D. and Mahajan, Gaurav},
  journal = {Journal of Machine Learning Research},
  volume = {22},
  number = {98},
  pages = {1--76},
  year = {2021},
}

@inproceedings{xie2021bellman,
  title = {{Bellman}-consistent Pessimism for Offline Reinforcement Learning},
  author = {Xie, Tengyang and Cheng, Ching-An and Jiang, Nan and Mineiro, Paul and Agarwal, Alekh},
  booktitle = {Advances in Neural Information Processing Systems},
  volume = {34},
  pages = {6683--6694},
  year = {2021},
}

@inproceedings{xie2023coverage,
  title = {The Role of Coverage in Online Reinforcement Learning},
  author = {Xie, Tengyang and Foster, Dylan J. and Bai, Yu and Jiang, Nan and Kakade, Sham M.},
  booktitle = {The Eleventh International Conference on Learning Representations},
  year = {2023},
}

@inproceedings{zhang2023offline,
  title = {Offline Learning in {Markov} Games with General Function Approximation},
  author = {Zhang, Yuheng and Bai, Yu and Jiang, Nan},
  booktitle = {Proceedings of the 40th International Conference on Machine Learning},
  series = {Proceedings of Machine Learning Research},
  volume = {202},
  pages = {40804--40829},
  year = {2023},
}

@misc{hu2026double,
  title = {Model-Based Reinforcement Learning with Double Oracle Efficiency in Policy Optimization and Offline Estimation},
  author = {Hu, Haichen and Qian, Jian and Simchi-Levi, David},
  year = {2026},
  eprint = {2605.00393},
  archivePrefix = {arXiv},
}

@inproceedings{chen2026coverage,
  title = {The Coverage Principle: How Pre-Training Enables Post-Training},
  author = {Chen, Fan and Huang, Audrey and Golowich, Noah and Malladi, Sadhika and Block, Adam and Ash, Jordan T. and Krishnamurthy, Akshay and Foster, Dylan J.},
  booktitle = {The Fourteenth International Conference on Learning Representations},
  year = {2026},
}

@inproceedings{foster2025foundation,
  title = {Is a Good Foundation Necessary for Efficient Reinforcement Learning? The Computational Role of the Base Model in Exploration},
  author = {Foster, Dylan J. and Mhammedi, Zakaria and Rohatgi, Dhruv},
  booktitle = {Proceedings of Thirty Eighth Conference on Learning Theory},
  series = {Proceedings of Machine Learning Research},
  volume = {291},
  pages = {2026--2142},
  year = {2025},
}

@inproceedings{huang2025sharpening,
  title = {Self-Improvement in Language Models: The Sharpening Mechanism},
  author = {Huang, Audrey and Block, Adam and Foster, Dylan J. and Rohatgi, Dhruv and Zhang, Cyril and Simchowitz, Max and Ash, Jordan T. and Krishnamurthy, Akshay},
  booktitle = {The Thirteenth International Conference on Learning Representations},
  year = {2025},
}

@inproceedings{jia2025verify,
  title = {Do We Need to Verify Step by Step? Rethinking Process Supervision from a Theoretical Perspective},
  author = {Jia, Zeyu and Rakhlin, Alexander and Xie, Tengyang},
  booktitle = {Proceedings of the 42nd International Conference on Machine Learning},
  series = {Proceedings of Machine Learning Research},
  volume = {267},
  pages = {27373--27398},
  year = {2025},
}

@inproceedings{yuan2025trajectory,
  title = {Trajectory {Bellman} Residual Minimization: A Simple Value-Based Method for {LLM} Reasoning},
  author = {Yuan, Yurun and Chen, Fan and Jia, Zeyu and Rakhlin, Alexander and Xie, Tengyang},
  booktitle = {Advances in Neural Information Processing Systems},
  volume = {38},
  year = {2025},
}

@inproceedings{chen2025outcome,
  title = {Outcome-Based Online Reinforcement Learning: Algorithms and Fundamental Limits},
  author = {Chen, Fan and Jia, Zeyu and Rakhlin, Alexander and Xie, Tengyang},
  booktitle = {Advances in Neural Information Processing Systems},
  volume = {38},
  year = {2025},
}

@inproceedings{kim2026coverage,
  title = {Coverage Improvement and Fast Convergence of On-policy Preference Learning},
  author = {Kim, Juno and Yun, Jihun and Lee, Jason D. and Jun, Kwang-Sung},
  booktitle = {Proceedings of the 43rd International Conference on Machine Learning},
  series = {Proceedings of Machine Learning Research},
  volume = {306},
  pages = {59211--59246},
  year = {2026},
}

@inproceedings{du2024exploration,
  title = {Exploration-Driven Policy Optimization in {RLHF}: Theoretical Insights on Efficient Data Utilization},
  author = {Du, Yihan and Winnicki, Anna and Dalal, Gal and Mannor, Shie and Srikant, R.},
  booktitle = {Proceedings of the 41st International Conference on Machine Learning},
  series = {Proceedings of Machine Learning Research},
  volume = {235},
  pages = {11830--11887},
  year = {2024},
}

@inproceedings{sriraman2026bestofn,
  title = {Revisiting the (Sub)Optimality of Best-of-{N} for Inference-Time Alignment},
  author = {Sriraman, Ved and Block, Adam},
  year = {2026},
  booktitle = {Proceedings of Thirty Ninth Conference on Learning Theory},
  series = {Proceedings of Machine Learning Research},
  volume = {336},
  pages = {5980--6028},
}

@inproceedings{yu2026caution,
  title = {From Curiosity to Caution: Mitigating Reward Hacking for Best-of-{N} with Pessimism},
  author = {Yu, Zhuohao and Wu, Zhiwei Steven and Block, Adam},
  year = {2026},
  booktitle = {The Fourteenth International Conference on Learning Representations},
}

@misc{sriraman2026behavior,
  title = {Behavior Cloning is Not All You Need: The Optimality of On-Policy Distillation for Noisy Expert Feedback},
  author = {Sriraman, Ved and Liu, Peihan and Hsu, Daniel and Block, Adam},
  year = {2026},
  eprint = {2606.30923},
  archivePrefix = {arXiv},
}

@inproceedings{menon2021statistical,
  title = {A Statistical Perspective on Distillation},
  author = {Menon, Aditya Krishna and Rawat, Ankit Singh and Reddi, Sashank J. and Kim, Seungyeon and Kumar, Sanjiv},
  booktitle = {Proceedings of the 38th International Conference on Machine Learning},
  series = {Proceedings of Machine Learning Research},
  volume = {139},
  pages = {7632--7642},
  year = {2021},
}

@inproceedings{ildiz2025highdimensional,
  title = {High-dimensional Analysis of Knowledge Distillation: Weak-to-Strong Generalization and Scaling Laws},
  author = {Ildiz, Muhammed Emrullah and Gozeten, Halil Alperen and Taga, Ege Onur and Mondelli, Marco and Oymak, Samet},
  booktitle = {The Thirteenth International Conference on Learning Representations},
  year = {2025},
}

@inproceedings{xie2026mathematical,
  title = {Two Mathematical Models of Knowledge Distillation},
  author = {Xie, Audrey and Schmidt, Ludwig and Duchi, John},
  booktitle = {Proceedings of The 29th International Conference on Artificial Intelligence and Statistics},
  series = {Proceedings of Machine Learning Research},
  volume = {300},
  pages = {4726--4734},
  year = {2026},
}

@inproceedings{stiennon2020summarize,
  title = {Learning to Summarize with Human Feedback},
  author = {Stiennon, Nisan and Ouyang, Long and Wu, Jeffrey and Ziegler, Daniel M. and Lowe, Ryan and Voss, Chelsea and Radford, Alec and Amodei, Dario and Christiano, Paul F.},
  booktitle = {Advances in Neural Information Processing Systems},
  volume = {33},
  pages = {3008--3021},
  year = {2020},
}

@inproceedings{dao2021semiparametric,
  title = {Knowledge Distillation as Semiparametric Inference},
  author = {Dao, Tri and Kamath, Govinda M. and Syrgkanis, Vasilis and Mackey, Lester},
  booktitle = {International Conference on Learning Representations},
  year = {2021},
}

@inproceedings{iliopoulos2022weighted,
  title = {Weighted Distillation with Unlabeled Examples},
  author = {Iliopoulos, Fotis and Kontonis, Vasilis and Baykal, Cenk and Menghani, Gaurav and Trinh, Khoa and Vee, Erik},
  booktitle = {Advances in Neural Information Processing Systems},
  volume = {35},
  year = {2022},
}

@misc{yamamoto2026residual,
  title = {Residual-as-Teacher: Mitigating Bias Propagation in Student--Teacher Estimation},
  author = {Yamamoto, Kakei and Wainwright, Martin J.},
  year = {2026},
  eprint = {2603.25466},
  archivePrefix = {arXiv},
}

@inproceedings{lei2021nearoptimal,
  title = {Near-Optimal Linear Regression under Distribution Shift},
  author = {Lei, Qi and Hu, Wei and Lee, Jason D.},
  booktitle = {Proceedings of the 38th International Conference on Machine Learning},
  series = {Proceedings of Machine Learning Research},
  volume = {139},
  pages = {6164--6174},
  year = {2021},
}

@article{ma2023optimally,
  title = {Optimally tackling covariate shift in {RKHS}-based nonparametric regression},
  author = {Ma, Cong and Pathak, Reese and Wainwright, Martin J.},
  journal = {The Annals of Statistics},
  volume = {51},
  number = {2},
  pages = {738--761},
  year = {2023},
}

@inproceedings{ge2024mle,
  title = {Maximum Likelihood Estimation is All You Need for Well-Specified Covariate Shift},
  author = {Ge, Jiawei and Tang, Shange and Fan, Jianqing and Ma, Cong and Jin, Chi},
  booktitle = {The Twelfth International Conference on Learning Representations},
  year = {2024},
}

@article{wang2026pseudolabeling,
  title = {Pseudo-Labeling for Kernel Ridge Regression under Covariate Shift},
  author = {Wang, Kaizheng},
  journal = {The Annals of Statistics},
  volume = {54},
  number = {1},
  pages = {252--276},
  year = {2026},
}

@misc{weill2026pseudolabeling,
  title = {Pseudo-Labeling for Unsupervised Domain Adaptation with Kernel {GLMs}},
  author = {Weill, Nathan and Wang, Kaizheng},
  year = {2026},
  eprint = {2603.19422},
  archivePrefix = {arXiv},
}

@misc{xia2026classification,
  title = {Classification Imbalance as Transfer Learning},
  author = {Xia, Eric and Klusowski, Jason M.},
  year = {2026},
  eprint = {2601.10630},
  archivePrefix = {arXiv},
}

@misc{xia2024prediction,
  title = {Prediction Aided by Surrogate Training},
  author = {Xia, Eric and Wainwright, Martin J.},
  year = {2024},
  eprint = {2412.09364},
  archivePrefix = {arXiv},
}

@inproceedings{saunshi2021mathematical,
  title = {A Mathematical Exploration of Why Language Models Help Solve Downstream Tasks},
  author = {Saunshi, Nikunj and Malladi, Sadhika and Arora, Sanjeev},
  booktitle = {International Conference on Learning Representations},
  year = {2021}
}

@inproceedings{bai2023statisticians,
  title = {Transformers as Statisticians: Provable In-Context Learning with In-Context Algorithm Selection},
  author = {Bai, Yu and Chen, Fan and Wang, Huan and Xiong, Caiming and Mei, Song},
  booktitle = {Advances in Neural Information Processing Systems},
  volume = {36},
  pages = {57125--57211},
  year = {2023}
}

@inproceedings{kim2024minimax,
  title = {Transformers are Minimax Optimal Nonparametric In-Context Learners},
  author = {Kim, Juno and Nakamaki, Tai and Suzuki, Taiji},
  booktitle = {Advances in Neural Information Processing Systems},
  volume = {37},
  pages = {106667--106713},
  year = {2024}
}

@inproceedings{ye2024online,
  title = {Online Iterative Reinforcement Learning from Human Feedback with General Preference Model},
  author = {Ye, Chenlu and Xiong, Wei and Zhang, Yuheng and Dong, Hanze and Jiang, Nan and Zhang, Tong},
  booktitle = {Advances in Neural Information Processing Systems},
  volume = {37},
  year = {2024}
}

@misc{openai2023gpt4,
  title = {{GPT-4} Technical Report},
  author = {{OpenAI}},
  year = {2023},
  eprint = {2303.08774},
  note = {arXiv preprint arXiv:2303.08774},
  archivePrefix = {arXiv},
}

@misc{deepseek2025r1,
  title = {{DeepSeek-R1}: Incentivizing Reasoning Capability in {LLMs} via Reinforcement Learning},
  author = {{DeepSeek-AI}},
  year = {2025},
  eprint = {2501.12948},
  note = {arXiv preprint arXiv:2501.12948},
  archivePrefix = {arXiv},
}

@misc{hinton2015distilling,
  title = {Distilling the Knowledge in a Neural Network},
  author = {Hinton, Geoffrey and Vinyals, Oriol and Dean, Jeff},
  year = {2015},
  eprint = {1503.02531},
  note = {arXiv preprint arXiv:1503.02531},
  archivePrefix = {arXiv},
}

@inproceedings{gu2024minillm,
  title = {{MiniLLM}: Knowledge Distillation of Large Language Models},
  author = {Gu, Yuxian and Dong, Li and Wei, Furu and Huang, Minlie},
  booktitle = {The Twelfth International Conference on Learning Representations},
  year = {2024},
}

@article{lukasik2022teacherspet,
  title = {Teacher's pet: understanding and mitigating biases in distillation},
  author = {Lukasik, Michal and Bhojanapalli, Srinadh and Menon, Aditya Krishna and Kumar, Sanjiv},
  journal = {Transactions on Machine Learning Research},
  year = {2022},
}

@inproceedings{roussinov2025controlling,
  title = {Controlling Out-of-Domain Gaps in {LLMs} for Genre Classification and Generated Text Detection},
  author = {Roussinov, Dmitri and Sharoff, Serge and Puchnina, Nadezhda},
  booktitle = {Proceedings of the 31st International Conference on Computational Linguistics},
  pages = {3329--3344},
  year = {2025},
}

@inproceedings{gureja2025mrewardbench,
  title = {{M-RewardBench}: Evaluating Reward Models in Multilingual Settings},
  author = {Gureja, Srishti and Miranda, Lester James V. and Islam, Shayekh Bin and Maheshwary, Rishabh and Sharma, Drishti and Winata, Gusti and Lambert, Nathan and Ruder, Sebastian and Hooker, Sara and Fadaee, Marzieh},
  booktitle = {Proceedings of the 63rd Annual Meeting of the Association for Computational Linguistics (Volume 1: Long Papers)},
  pages = {43--58},
  year = {2025},
}

@inproceedings{lightman2024verify,
  title = {Let's Verify Step by Step},
  author = {Lightman, Hunter and Kosaraju, Vineet and Burda, Yuri and Edwards, Harrison and Baker, Bowen and Lee, Teddy and Leike, Jan and Schulman, John and Sutskever, Ilya and Cobbe, Karl},
  booktitle = {The Twelfth International Conference on Learning Representations},
  year = {2024},
}

@misc{chen2021code,
  title = {Evaluating Large Language Models Trained on Code},
  author = {Chen, Mark and Tworek, Jerry and Jun, Heewoo and Yuan, Qiming and Pinto, Henrique Ponde de Oliveira and others},
  year = {2021},
  eprint = {2107.03374},
  note = {arXiv preprint arXiv:2107.03374},
  archivePrefix = {arXiv},
  }

@article{bartlett2020benign,
  title={Benign overfitting in linear regression},
  author={Bartlett, Peter L and Long, Philip M and Lugosi, G{\'a}bor and Tsigler, Alexander},
  journal={Proceedings of the National Academy of Sciences},
  volume={117},
  number={48},
  pages={30063--30070},
  year={2020},
  publisher={National Academy of Sciences}
}

@article{bolte2007lojasiewicz,
  author  = {Bolte, J{\'e}r{\^o}me and Daniilidis, Aris and Lewis, Adrian},
  title   = {The {\L{}ojasiewicz} Inequality for Nonsmooth Subanalytic
             Functions with Applications to Subgradient Dynamical Systems},
  journal = {SIAM Journal on Optimization},
  volume  = {17},
  number  = {4},
  pages   = {1205--1223},
  year    = {2007}
}

@article{gemini2023gemini,
  author  = {{Gemini Team}},
  title   = {{Gemini}: A Family of Highly Capable Multimodal Models},
  journal = {arXiv preprint arXiv:2312.11805},
  year    = {2023}
}

@article{grattafiori2024llama3,
  author  = {Grattafiori, Aaron and Dubey, Abhimanyu
             and Jauhri, Abhinav and others},
  title   = {The {Llama} 3 Herd of Models},
  journal = {arXiv preprint arXiv:2407.21783},
  year    = {2024}
}

@article{qwen2024qwen25,
  author  = {{Qwen Team}},
  title   = {{Qwen2.5} Technical Report},
  journal = {arXiv preprint arXiv:2412.15115},
  year    = {2024}
}

@inproceedings{ross2011reduction,
  title = {A Reduction of Imitation Learning and Structured Prediction to No-Regret Online Learning},
  author = {Ross, Stephane and Gordon, Geoffrey and Bagnell, Drew},
  booktitle = {Proceedings of the Fourteenth International Conference on Artificial Intelligence and Statistics},
  series = {Proceedings of Machine Learning Research},
  volume = {15},
  pages = {627--635},
  year = {2011},
}

@inproceedings{czarnecki2019distilling,
  title = {Distilling Policy Distillation},
  author = {Czarnecki, Wojciech M. and Pascanu, Razvan and Osindero, Simon and Jayakumar, Siddhant and Swirszcz, Grzegorz and Jaderberg, Max},
  booktitle = {Proceedings of the Twenty-Second International Conference on Artificial Intelligence and Statistics},
  series = {Proceedings of Machine Learning Research},
  volume = {89},
  pages = {1331--1340},
  year = {2019},
}

@inproceedings{foster2024behavior,
  title = {Is Behavior Cloning All You Need? Understanding Horizon in Imitation Learning},
  author = {Foster, Dylan J. and Block, Adam and Misra, Dipendra},
  booktitle = {Advances in Neural Information Processing Systems},
  volume = {37},
  pages = {120602--120666},
  year = {2024},
}

@misc{zhang2026online,
  title = {When Does Online Imitation Learning Help in {LLM} Post-Training? The Role of (Non-)Realizability Beyond Horizon},
  author = {Zhang, Huaqing and Gai, Jingchu and Kim, Juno and Liu, Bingbin and Risteski, Andrej},
  year = {2026},
  eprint = {2606.30445},
  archivePrefix = {arXiv},
}

@inproceedings{zhu2023principled,
  title = {Principled Reinforcement Learning with Human Feedback from Pairwise or {K}-wise Comparisons},
  author = {Zhu, Banghua and Jordan, Michael I. and Jiao, Jiantao},
  booktitle = {Proceedings of the 40th International Conference on Machine Learning},
  series = {Proceedings of Machine Learning Research},
  volume = {202},
  pages = {43037--43067},
  year = {2023},
}

@inproceedings{xiong2024iterative,
  title = {Iterative Preference Learning from Human Feedback: Bridging Theory and Practice for {RLHF} under {KL}-constraint},
  author = {Xiong, Wei and Dong, Hanze and Ye, Chenlu and Wang, Ziqi and Zhong, Han and Ji, Heng and Jiang, Nan and Zhang, Tong},
  booktitle = {Proceedings of the 41st International Conference on Machine Learning},
  series = {Proceedings of Machine Learning Research},
  volume = {235},
  pages = {54715--54754},
  year = {2024},
}

@inproceedings{xie2025exploratory,
  title = {Exploratory Preference Optimization: Harnessing Implicit {Q*}-Approximation for Sample-Efficient {RLHF}},
  author = {Xie, Tengyang and Foster, Dylan J. and Krishnamurthy, Akshay and Rosset, Corby and Awadallah, Ahmed H. and Rakhlin, Alexander},
  booktitle = {The Thirteenth International Conference on Learning Representations},
  year = {2025},
}

@inproceedings{zhao2025sharp,
  title = {Sharp Analysis for {KL}-Regularized Contextual Bandits and {RLHF}},
  author = {Zhao, Heyang and Ye, Chenlu and Gu, Quanquan and Zhang, Tong},
  booktitle = {Advances in Neural Information Processing Systems},
  volume = {38},
  year = {2025},
}

@inproceedings{agarwal2024onpolicy,
  title = {On-Policy Distillation of Language Models: Learning from Self-Generated Mistakes},
  author = {Agarwal, Rishabh and Vieillard, Nino and Zhou, Yongchao and Stanczyk, Piotr and Ramos, Sabela and Geist, Matthieu and Bachem, Olivier},
  booktitle = {The Twelfth International Conference on Learning Representations},
  year = {2024},
}

@article{fort2015clt,
  author = {Fort, Gersende},
  title = {Central limit theorems for stochastic approximation with controlled {Markov} chain dynamics},
  journal = {ESAIM: Probability and Statistics},
  volume = {19},
  pages = {60--80},
  year = {2015},
  doi = {10.1051/ps/2014013}
}
